\documentclass[12pt]{graphicsclass}
\usepackage[utf8]{inputenc}
\usepackage{palatino}
\usepackage{amsmath}
\usepackage{amssymb} 
\usepackage{graphicx}
\graphicspath{ {images/} }
\usepackage[labelfont=bf]{caption}
\usepackage{subcaption}
\usepackage{setspace}
\usepackage[margin=1in]{geometry}
\usepackage[hidelinks]{hyperref}
\usepackage{cite}
\usepackage{xcolor,colortbl}
\usepackage{overpic}
\usepackage[all]{nowidow}

\usepackage{amsmath}
\usepackage{mathtools}
\usepackage{amssymb,bm,amsthm}
\usepackage{microtype}
\usepackage{booktabs}
\usepackage{multirow} 
\usepackage{algorithmic}
\usepackage{algorithm}

\newtheorem{definition}{Definition}
\newtheorem{lemma}{Lemma}
\newtheorem{theorem}{Theorem}
\newtheorem*{theorem*}{Theorem}

\newcommand{\tensor}[1]{\boldsymbol{\mathcal{#1}}}
\newcommand{\mat}[1]{\mathbf{{#1}}}
\newcommand{\CP}[1]{[\![#1]\!]}
\newcommand{\x}{\mathbf{x}}
\newcommand{\T}{\tensor{T}}

\usepackage{todonotes}

\title{
{Learning From Graph-Structured Data: \\ Addressing Design Issues and Exploring Practical Applications in Graph Representation Learning}\\~\\~\\
{\large School of Computer Science \\ 
	McGill University, Montreal \\ 
	Sep, 2024 \\~\\~\\
	A thesis submitted to McGill University in partial fulfillment of the requirements of the degree of \\~\\ Master of Science }\\~\\
}

\author{\textcopyright Chenqing (William) Hua, 2024}
\date{}

\begin{document}
\maketitle

Heartfelt gratitude extends to my beloved family, friends, and mentors, especially to my supervisors Doina Precup and Guy Wolf, who have steadfastly supported me throughout years of research and study abroad. Their supervision over the years of research contributed to my growth as a complete and well-rounded human being.

\chapter*{Abstract}
\label{sec:engAbstract}
\addcontentsline{toc}{section}{\nameref{sec:engAbstract}}

Graphs serve as fundamental descriptors for systems composed of interacting elements, capturing a wide array of data types, from molecular interactions to social networks and knowledge graphs. In this paper, we present an exhaustive review of the latest advancements in graph representation learning and Graph Neural Networks (GNNs). GNNs, tailored to handle graph-structured data, excel in deriving insights and predictions from intricate relational information, making them invaluable for tasks involving such data. Graph representation learning, a pivotal approach in analyzing graph-structured data, facilitates numerous downstream tasks and applications across machine learning, data mining, biomedicine, and healthcare.

Our work delves into the capabilities of GNNs, examining their foundational designs and their application in addressing real-world challenges. We introduce a GNN equipped with an advanced high-order pooling function, adept at capturing complex node interactions within graph-structured data. This pooling function significantly enhances the GNN's efficacy in both node- and graph-level tasks. Additionally, we propose a molecular graph generative model with a GNN as its core framework. This GNN backbone is proficient in learning invariant and equivariant molecular characteristics. Employing these features, the molecular graph generative model is capable of simultaneously learning and generating molecular graphs with atom-bond structures and precise atom positions. Our models undergo thorough experimental evaluations and comparisons with established methods, showcasing their superior performance in addressing diverse real-world challenges with various datasets.

\chapter*{Abrégé}
\label{sec:frAbstract}
\addcontentsline{toc}{section}{\nameref{sec:frAbstract}}
Les graphiques servent de descripteurs fondamentaux pour les systèmes composés d'éléments en interaction, capturant un large éventail de types de données, des interactions moléculaires aux réseaux sociaux et aux graphiques de connaissances. Dans cet article, nous présentons un examen exhaustif des dernières avancées en matière d'apprentissage de la représentation graphique et des réseaux de neurones graphiques (GNN). Les GNN, conçus pour gérer des données structurées sous forme de graphiques, excellent dans l'obtention d'informations et de prédictions à partir d'informations relationnelles complexes, ce qui les rend inestimables pour les tâches impliquant de telles données. L'apprentissage de la représentation graphique, une approche cruciale dans l'analyse des données structurées sous forme de graphiques, facilite de nombreuses tâches et applications en aval dans les domaines de l'apprentissage automatique, de l'exploration de données, de la biomédecine et des soins de santé.

Nos travaux examinent les capacités des GNN, en examinant leurs conceptions fondamentales et leur application pour relever les défis du monde réel. Nous introduisons un GNN équipé d'une fonction avancée de pooling d'ordre élevé, capable de capturer des interactions de nœuds complexes au sein de données structurées sous forme de graphes. Cette fonction de pooling améliore considérablement l'efficacité du GNN dans les tâches au niveau des nœuds et des graphiques. De plus, nous proposons un modèle génératif de graphes moléculaires avec un GNN comme cadre principal. Ce squelette GNN est compétent dans l’apprentissage des caractéristiques moléculaires invariantes et équivariantes. Grâce à ces fonctionnalités, le modèle génératif de graphes moléculaires est capable d'apprendre et de générer simultanément des graphes moléculaires avec des structures de liaisons atomiques et des positions précises des atomes. Nos modèles sont soumis à des évaluations expérimentales approfondies et à des comparaisons avec des méthodes établies, démontrant leurs performances supérieures pour relever divers défis du monde réel avec divers ensembles de données.

\chapter*{Contribution Statement}
For Graph Neural Networks with High-Order Pooling, Chenqing Hua contributed to model proposal, theoretical proof, experimental validation, and manuscript preparation. Guillaume Rabusseau contributed to theoretical proof, manuscript writing, securing computational resources, funding, and supervision. Jian Tang contributed by securing computational resources, providing funding, and offering supervision.

For Graph Neural Networks for Molecule Generation, Chenqing Hua contributed to model proposal, theoretical proof, experimental validation, and manuscript preparation. Sitao Luan and Minkai Xu assisted with manuscript writing. Rex Ying provided model-related insights and discussions. Jie Fu helped secure computational resources, while Stefano Ermon contributed to model discussion. Doina Precup supported the project by securing computational resources, providing funding, and offering supervision.

\chapter*{Acknowledgements}
\label{sec:ded}
\addcontentsline{toc}{section}{\nameref{sec:ded}}
This work is supported by Natural Sciences and Engineering Research Council of Canada (NSERC) Grant, Canadian Institute for Advanced Research (CIFAR) AI Chairs program, and Fonds d’accélération des collaborations en santé (FACS-Acuity) supported by Ministre de l'Économie et de l'Innovation Canada.

Heartfelt gratitude extends to my beloved family, friends, and mentors, especially to Doctor Sitao Luan, who have steadfastly supported me throughout years of research and study abroad. It is through your unwavering encouragement that I have evolved into a fulfilled individual and a nurtured researcher. Thank you, each one of you, for gracing my life and contributing to my growth as a complete and well-rounded human being.

\tableofcontents
\listoffigures %
\addcontentsline{toc}{section}{\listfigurename}
\listoftables
\addcontentsline{toc}{section}{\listtablename}

\clearpage 
\pagenumbering{arabic} 

\chapter{Introduction}
Graph neural networks (GNNs) generalize traditional neural network architectures for data in the Euclidean domain to data in non-Euclidean domains \cite{kipf2016classification, velivckovic2017attention, luan2021heterophily, luan2022we, hua2022high, luan2023graph}. 
As graphs are very general and flexible data structures and are ubiquitous in the real world, GNNs are now widely used in a variety of domains and applications such as social network analysis \cite{hamilton2017representation, luan2020complete, luan2023we}, recommender systems \cite{ying2018graph}, graph reasoning \cite{zhu2021neural}, and drug discovery \cite{satorras2021n2, hua2022multi, hua2023mudiff, hua2024effective}.  
Indeed, many GNN architectures (e.g., GCN \cite{kipf2016classification}, GAT \cite{velivckovic2017attention}, MPNN \cite{gilmer2017neural}, ACM-GNN \cite{luan2022revisiting}) have been proposed. The essential idea of all these architectures is to iteratively update  node representations by aggregating the information from their neighbors through multiple rounds of neural message passing. The final node representations can be used for downstream tasks such as node classification or link prediction. For graph classification, an additional readout layer is used to combine all the node representations to calculate the entire graph representation.

A recent work, principled neighborhood aggregation (PNA)~\cite{corso2020principal}, aims to design a more flexible aggregation function by combining multiple simple aggregation functions, each of which is associated with a learnable weight. However, the practical capacity of PNA is still limited by simply combining multiple \emph{simple} aggregation functions. A more expressive solution would be to model high-order non-linear interactions when aggregating node features. However, explicitly modeling high-order non-linear interactions among nodes is very expensive, with both the time and memory complexity being exponential in the size of the neighborhood. This raises the question of whether there exists an aggregation function which can model high-order non-linear interactions among nodes while remaining computationally efficient. 

Here, we introduce such an approach based on symmetric tensor decomposition. We design an aggregation function over a set of node representations for GNNs, which is permutation-invariant and is capable of modeling non-linear high-order multiplicative interactions among nodes \cite{hua2022high}. The methods leverages the symmetric CANDECOMP/PARAFAC decomposition~(CP)~\cite{hitchcock1927expression,kolda2009tensor} to design an efficient parameterization of permutation-invariant multilinear maps over a set of node representations. Theoretically, we show that the CP layer can compute any permutation-invariant multilinear polynomial, including the classical sum and mean aggregation functions. We also show that the CP layer is universally strictly more expressive than sum and mean pooling: with probability one, any function computed by a random CP layer cannot be computed using sum and mean pooling. We propose the CP-layer as an expressive mean of performing the aggregation and update functions in GNN. We call the resulting model a tensorized GNN~(tGNN).

Beyond the fundamental designs of GNNs for various classification and regression tasks on graphs, GNNs serve as the foundational model for molecular graph generation tasks \cite{satorras2021n, satorras2021n2, xu2022geodiff, hua2023mudiff, xu2023geometric}.
These generative models based on GNNs are trained on a dataset containing known molecular graph structures and can be subsequently employed to generate unseen molecular structures resembling those present in the training dataset. The key to their success lies in GNNs' ability to learn generalized atom and bond representations, as well as atom-atom interactions, within the message-passing framework. These models can produce either 2D atom-bond molecular graphs or 3D molecular conformations.

While 2D graph structures capture the topology and connectivity of molecules \cite{gilmer2017neural}, 3D geometric structures provide an insight into the spatial arrangements of atoms \cite{schutt2017schnet, xu2022geodiff}. Both types of structural information are crucial for a comprehensive representation of a molecule. 
So, learning 2D and 3D structures together leads to an accurate and complete molecule representation.
Learning 2D and 3D structures together facilitates accurate and complete molecule representation. However, existing generative models for molecules often focus solely on either 2D or 3D molecular data generation \cite{satorras2021n, hoogeboom2022equivariant, vignac2022digress}, limiting their ability to provide a comprehensive representation of molecules.
Recognizing this limitation, we are motivated to propose a novel generative model that jointly generates 2D and 3D molecular data. This approach captures both the topological information from 2D graphs and the spatial atom arrangements from 3D geometry, enabling a more holistic understanding of molecular structures.

To achieve this goal, we introduce a diffusion generative model named MUDiff and a GNN backbone model named MUformer \cite{hua2023mudiff}. MUDiff co-generates the 2D graph structure and 3D geometric structure of a molecule, while MUformer co-learns both molecular representation. The diffusion model introduces continuous and discrete noises to features, including atom features, coordinates, and graph structure, followed by a denoising process that predicts the clean graph structure and estimates original atom features and coordinates.
Through these novel designs, our model can generate and learn a comprehensive molecular representation that captures both 2D and 3D structures, effectively addressing the aforementioned limitations in existing generative models for molecules.

In summary, we begin by laying the groundwork with essential concepts about graph representation learning and GNNs in Chapter~\ref{sec:related.work} and Section~\ref{sec:related.work.gnn}, exploring their designs and inherent limitations. A key drawback identified is the challenge of modeling intricate interactions between nodes in a graph using existing designs (as discussed in Section~\ref{sec:related.work.pooling}). To tackle this limitation, we propose a new GNN design called tensorized-GNN (tGNN) in Section~\ref{sec:gnn.pooling}. We show theoretical results and tGNN architecture in Section~\ref{sec:tgnn.theory}, and further validate tGNN on real-world datasets in Section~\ref{sec:tgnn.experiments}. 
Shifting from the theoretical framework, we delve into the practical realm, focusing on the application of GNNs in real-world scenarios in Chapter~\ref{sec:mol.gen}, particularly in the generation of molecules represented as graph-structured data. 
Here, molecules are represented as graphs, with atoms serving as nodes and bonds as edges. In Section~\ref{sec:graph.transformer}, we introduce a novel GNN design, MUformer, that encodes invariant and equivariant features of molecular systems. In addition, we introduce a novel molecule generative model named MUDiff in Section~\ref{sec:diffusion.process}, which applies MUformer as the backbone network. This comprehensive model not only learns the structures of molecules but also captures their geometry. In Section~\ref{sec:experiments}, we perform rigorous evaluation on real-world datasets and systematically compared MUformer and MUDiff to existing models, providing a practical and tangible contribution to the field. In the end, we conclude our major contributions in Section~\ref{sec:conclusion}, provide deep understandings, and discuss future work in Section~\ref{sec:future.work}.

\chapter{Related Work}
\label{sec:related.work}
In this chapter, we give a comprehensive review of the relevant literature on graph representation learning and graph neural networks (GNNs). 

\section{Graph Representation Learning}
\label{sec:related.work.gnn}
Graph representation learning refers to the process of encoding graph-structured data or relational data into low-dimensional vectors or embeddings. It is a fundamental task that has gained significant attention in recent years, particularly in the field of machine learning and data mining \cite{hamilton2020graph}.

Graphs are universal descriptors of systems with interacting elements, and they can represent various types of data, such as molecular interactions, social networks, and knowledge graphs. Graph representation learning aims to capture the structural and feature information of graphs in a compact and meaningful way, enabling downstream tasks like node classification, link prediction, and anomaly detection \cite{kipf2016classification, velivckovic2017attention, hamilton2017inductive, hamilton2017representation, luan2020complete, hua2022high, hua2023mudiff, hua2022multi}.

The goal of graph representation learning is to construct a set of features or embeddings that represent the structure of the graph and the data associated with it. These embeddings can be categorized into three types: node-wise embeddings, which represent each node in the graph \cite{hamilton2017representation, hamilton2020graph, gilmer2017neural, luan2022revisiting, li2020deepergcn, luan2024heterophilic, luan2024heterophily}; edge-wise embeddings, which represent each edge in the graph \cite{corso2020principal, le2021parameterized, beani2021directional}; and graph-wise embeddings, which represent the graph as a whole \cite{hua2022high, ying2018hierarchical, xu2018powerful, bresson2017residual}.
Overall, graph representation learning plays a crucial role in effectively encoding high-dimensional sparse graph-structured data into low-dimensional dense vectors, enabling various downstream tasks and applications in fields like machine learning, data mining, biomedicine, and healthcare \cite{hamilton2017representation, xu2018powerful, xu2022geodiff, xu2023geometric, hua2022multi, hoogeboom2022equivariant, hua2023mudiff, yu2024fraggen, hua2024effective, hua2024enzymeflow, luan2024heterophily}.

\section{Graph Neural Networks}
Graph Neural Networks (GNNs) are a type of neural network specifically designed to process data in the form of graphs \cite{kipf2016classification, hamilton2017inductive, hamilton2020graph}. They are used for tasks such as graph classification, node classification, and edge prediction \cite{abu2019mixhop, velivckovic2017attention, austin2021structured, anderson2019cormorant, luan2019break, luan2023we, hua2022high, hua2023mudiff}. GNNs have the ability to capture both the structural and relational information in a graph, making them highly effective for tasks involving graph data. GNNs differ from traditional neural networks in that they can handle data in the form of graphs, whereas traditional neural networks are designed to process data in the form of vectors or sequences \cite{he2016deep}.

GNNs apply the predictive power of deep learning to rich data structures that depict objects and their relationships as points connected by lines in a graph. In GNNs, data points are called nodes, which are linked by lines called edges, with elements expressed mathematically so machine learning algorithms can make useful predictions at the level of nodes, edges, or entire graphs \cite{li2020deepergcn, hamilton2017representation, hamilton2020graph, luan2023graph}.

GNNs have been adapted to leverage the structure and properties of graphs. They explore the components needed for building a graph neural network and motivate the design choices behind them. In summary, GNNs are a powerful tool for processing and analyzing graph-structured data, allowing for the extraction of valuable insights and predictions from complex relational information \cite{abu2019mixhop, beani2021directional, corso2020principal, ying2021transformers, ying2018graph}.

\section{Pooling Functions for Graph Neural Networks}
\label{sec:related.work.pooling}
In the design of Graph Neural Networks (GNNs), a key component is an effective pooling function for aggregating features from local neighborhoods to update node representations and for combining these node representations to derive a graph-level representation.

Kipf et al.~\cite{kipf2016classification} successfully define convolutions on graph-structured data by averaging node information in a neighborhood. Xu et al.~\cite{xu2018powerful} prove the incomplete expressivity of mean aggregation to distinguish nodes, and further propose to use sum aggregation to differentiate nodes with similar properties. Corso et al.~\cite{corso2020principal} further generalize this idea and show that mean aggregation can be a particular case of sum aggregation with a linear multiplier, and further propose an architecture with multiple aggregation channels to adaptively learn low-order information. Luan et al.~\cite{luan2020complete} show that the use of aggregation is not sufficient and further propose to utilize GNNs with aggregation and diversification operations simultaneously to learn. 
Most GNNs use low-order aggregation schemes for learning node representations. However, Battiston et al.~\cite{battiston2020networks} show that aggregation should go beyond low-order interactions because high-order terms can more accurately model  real-world complex systems. 

Most GNNs employ low-order aggregation schemes for learning node representations. However,
Morris et al.~\cite{morris2019weisfeiler} consider high-order graph structures into account to build a high-order expressive model that is more powerful than regular message-passing GNNs.
A brand new architecture, Graph Neural Diffusion \cite{chamberlain2021grand}, take graph neural networks as approximations of an underlying partial differential equation, thus it uses more information than just simple low-order pooling information and further addresses depth and over-smoothing issues in graph machine learning. 
Baek et al.~\cite{baek2021accurate} formulate the pooling problem as a multiset encoding problem with auxiliary information about the graph structure, and propose an attention-based pooling layer that captures the interaction between nodes according to their structural dependencies. Wang et al.~\cite{wang2020second} apply second-order statistic methods because the use of second-order statistics takes advantage of the Riemannian geometry of the space of symmetric positive definite matrices.

In GNNs, these pooling functions operate in tandem with convolutional layers designed for graph data. Their primary function is to distill higher-level features from the graph structure and generalize these features. This process enhances the robustness of the network and its ability to recognize and interpret complex patterns and structures within the graph \cite{hua2022high, luan2020complete, xu2018powerful, wang2020second, abu2019mixhop, kipf2016classification, velivckovic2017attention, beani2021directional}.

\section{Graph Neural Networks for Molecule Generation}

Graph Neural Networks (GNNs) have also been increasingly utilized for molecular generation tasks, encompassing property prediction, docking, optimization, and generation \cite{xu2022geodiff, xu2023geometric, hua2022high, vignac2022digress}. Key challenges in this domain include ensuring the validity, diversity, and property adherence of the generated molecules.

Typically, a variational autoencoder (VAE) framework is employed, wherein a GNN encoder maps a molecule's graph to a latent space, and a GNN decoder reconstructs the graph. This approach allows for sampling and manipulation of the latent space to create novel molecules \cite{bongini2021molecular, lim2018molecular, jin2018junction, liu2018constrained, zhang2024deep, zhang2024ecloudgen}. Jin et al. \cite{jin2018junction} enhance molecule generation by decomposing molecules into trees of substructures and using separate GNNs for encoding and decoding these trees and graphs, improving validity, diversity, and accuracy. Liu et al.~\cite{liu2018constrained} advance this technique by introducing a conditional framework where the latent space is informed by property vectors, facilitating the generation of molecules with specific characteristics.

Other than using GNNs as the backbone model for VAE framework, they can be adopted for generative diffusion models for molecule generation \cite{ho2020denoising, kingma2021variational, austin2021structured}. Xu et al.~\cite{xu2022geodiff} leverage diffusion models to produce molecules with minimal conformation energy, while Vignac et al.~\cite{vignac2022digress} use a graph transformer in their diffusion model to refine both atom features and molecular structures. Moreover, recent studies have explored the integration of equivariant GNNs within diffusion models for molecule generation. For example, Hoogeboom et al.~\cite{hoogeboom2022equivariant} introduce a diffusion model with an equivariant GNN, enabling the model to work jointly on atom features and coordinates. Zhang et al.~\cite{zhang2023equivariant} introduce an autoregressive flow model for generating atom and bond types, as well as 3D coordinates, using a local spherical coordinate system for relative positioning. Peng et al.~\cite{peng2023moldiff} address atom-bond inconsistency in 3D molecule generation using a diffusion model that ensures simultaneous and consistent generation of atoms and bonds. Zhang et al.~\cite{zhang2022molecule} focus on structure-based drug design, generating both 2D and 3D molecular graphs to enhance molecular representation.
\chapter{Graph Neural Networks with \\ High-Order Pooling}
\label{sec:gnn.pooling}
In this chapter, we examine the limitations of the pooling functions employed by Graph Neural Networks (GNNs) in graph-level tasks, where all node features are consolidated into a single comprehensive feature to represent the entire graph. Conventional GNNs commonly utilize basic pooling functions such as sum, average, or max to aggregate messages within local neighborhoods for updating node representation or pooling node representations across the entire graph to compute the graph representation. While these linear operations are straightforward and efficient, they fall short in capturing high-order non-linear interactions among nodes. Our proposed solution introduces a highly expressive GNN pooling function that leverages tensor decomposition to model intricate high-order non-linear node interactions.

\section{Preliminaries}
We introduce notations that are particularly used for this chapter. We use bold font letters for vectors (e.g., $\mathbf{v}$), capital letters (e.g., $\mathbf{M},\tensor{T}$) for matrices and tensors respectively, and regular letters for nodes (e.g., $v$). Let  ${G}=(V,{E})$ be a graph, where $V$ is the node set and $E$ is the edge set with self-loop.
We use ${N}(v)$ to denote the neighborhood set of node $v$, i.e., ${N}(v)=\{u: e_{vu} \in {E}\}$. A node feature is a vector $\mathbf{x} \in \mathbb{R}^F$ defined on ${V}$, where $\mathbf{x}_v$ is defined on the node $v$. We use $\otimes$ to denote the Kronecker product, $\circ$ to denote the outer product, and $\odot$ to denote the Hadamard product, i.e., component-wise product, between vectors, matrices, and tensors. For any integer $k$, we use the notation $[k]=\{1,\cdots,k\}$.
\subsection{Tensors}
We introduce basic notions of tensor algebra, more details can be found in \cite{kolda2009tensor}.
A $k$-th order tensor $\tensor{T}\in \mathbb{R}^{N_1\times N_2 \times ... \times N_k}$ can simply be
seen as a multidimensional array.
The mode-$i$ fibers of $\tensor{T}$ are the vectors obtained by fixing all indices except the $i$-th one: $\tensor{T}_{n_1,n_2,...,n_{i-1},:,n_{i+1},...,n_k}\in \mathbb{R}^{N_i}$. 
The $i$-th mode matricization of a tensor is the matrix having its mode-$i$ fibers as columns and is denoted by $\tensor{T}_{(i)}$, e.g., $\tensor{T}_{(1)} \in \mathbb{R}^{N_1\times N_2\cdots N_k}$.
We use $\tensor{T}{\times}_i\mathbf{v}\in \mathbb{R}^{N_1\times \cdots \times N_{i-1} \times N_{i+1} \times \cdots \times N_k}$ to denote the mode-$i$ product between a tensor $\tensor{T}\in \mathbb{R}^{N_1\times \cdots \times N_k}$ and a vector $\mathbf{v}\in\mathbb{R}^{N_i}$, which is defined by 
$(\tensor{T}{\times}_i\mathbf{v})_{n_1,...,n_{i-1},n_{i+1},...,n_k}= \sum_{n_i=1}^{N_i}\tensor{T}_{n_1,...,n_k}\mathbf{v}_{n_i}.$ 
The following useful identity relates the mode-$i$ product with the Kronecker product: 
\begin{equation}
\label{eq:tenvec.product.and.kron}
    \T\times_1 \mat{v}_1 \times_2 \cdots \times_{k-1} \mat{v}_{k-1} = \tensor{T}_{(k)} (\mat{v}_{k-1}\otimes\cdots\otimes \mat{v}_1).
\end{equation}

\subsection{CANDECOMP/PARAFAC Decomposition}
\label{sec:cp_decomp}
We refer to $\mathbf{C}$ANDECOMP/$\mathbf{P}$ARAFAC decomposition of a tensor as CP decomposition \cite{kiers2000towards,hitchcock1927expression}. A Rank $R$ CP decomposition factorizes a $k-$th order tensor $\tensor{T}\in \mathbb{R}^{N_1\times...\times N_k}$ into the sum of $R$ rank one tensors as $\tensor{T}=\sum_{r=1}^{R} \mathbf{v}_{1r} \circ \mathbf{v}_{2r} \circ \dots \circ \mathbf{v}_{kr}$, where $\circ$ denotes the vector outer-product and $\mathbf{v}_{1r}\in \mathbb{R}^{N_1}, \mathbf{v}_{2r}\in \mathbb{R}^{N_2},...,\mathbf{v}_{kr}\in \mathbb{R}^{N_k}$ for every $r={1,2,...,R}$.

The decomposition vectors, $\mathbf{v}_{:r}$ for $r=1,...,R$, are equal in length, thus can be naturally gathered into  factor matrices $\mathbf{M}_1=[\mathbf{v}_{11},...,\mathbf{v}_{1R}]\in\mathbb{R}^{N_1\times R},...,\mathbf{M}_k = [\mathbf{v}_{k1},...,\mathbf{v}_{kR}]\in\mathbb{R}^{N_k\times R}$. Using the factor matrices, we denote the CP decomposition of  $\tensor{T}$  as 
$$\tensor{T}=\sum_{r=1}^{R} \mathbf{v}_{1r} \circ \mathbf{v}_{2r} \circ \dots \circ \mathbf{v}_{kr}=\CP{\mathbf{M}_1,\mathbf{M}_2,...,\mathbf{M}_k}.$$

The $k$-th order tensor $\tensor{T}$ is \emph{cubical} if all its modes have the same size, i.e., $ N_1=N_2=...=N_k:=N$.
A tensor $\tensor{T}$ is symmetric if it is cubical and is invariant under permutation of its indices:
$$
\tensor{T}_{{n_{\phi(1)}},...,{n_{\phi(k)}}}= \tensor{T}_{{n_1},...,{n_k}}, \   n_1,\cdots,n_k\in[N].
$$
for any permutation $\phi:[k]\to[k]$.
\begin{figure}
\centering
{
\includegraphics[width=0.5\textwidth]{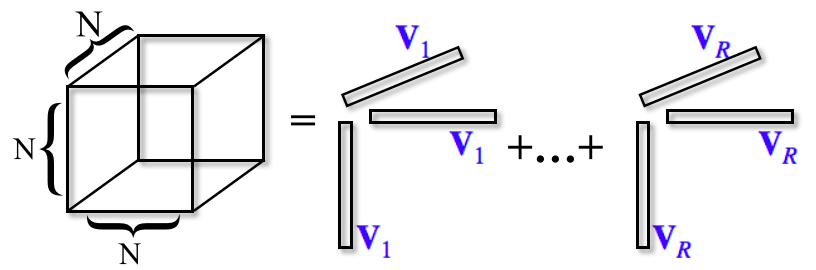}}
{%
  \caption{Example of a rank $R$ symmetric CP decomposition of a symmetric $3$-order tensor $\tensor{T}\in \mathbb{R}^{N\times N \times N}$ such that $\tensor{T}=\Sigma_{r=1}^R \mathbf{v}_r\circ \mathbf{v}_r\circ \mathbf{v}_r$.} 
  \label{fig:CPdecomp}
}
\end{figure}
A rank $R$ symmetric CP decomposition of a symmetric tensor $\tensor{T}$ is a decomposition of the form 
$\tensor{T} =\CP{\mathbf{M},\cdots,\mathbf{M}}$ with $\mathbf{M}\in\mathbb{R}^{N\times R}$. It is well known that any symmetric tensor admits a symmetric CP decomposition~\cite{comon2008symmetric}, we illustrate a rank $R$ symmetric CP decomposition in Fig.~\ref{fig:CPdecomp}.

We say that a tensor $\tensor{T}$ is partially symmetric if it is symmetric in a subset of its modes
\cite{kolda2015numerical}. 
For example, a $3$-rd order tensor $\tensor{T}\in \mathbb{R}^{N_1\times N_1 \times N_3}$ is partially symmetric w.r.t. modes 1 and 2 if it has symmetric frontal slices; i.e., $\tensor{T}_{:,:,k}$ is a symmetric matrix for all $k\in[N_3]$. We prove the fact that any partially symmetric tensor admits a partially symmetric CP decomposition in Lemma~\ref{lem:partial.CP.exists}, e.g., if $\tensor{T}\in \mathbb{R}^{N_1\times N_1 \times N_3}$ is partially symmetric w.r.t.  modes 1 and 2, there exist $\mathbf{M}\in\mathbb{R}^{N_1\times R}$ and $\mathbf{W}\in\mathbb{R}^{N_3\times R}$ such that $\tensor{T}=\CP{\mathbf{M},\mathbf{M},\mathbf{W}}$.

\begin{lemma}\label{lem:partial.CP.exists}
Any partially symmetric tensor admits a partially symmetric CP decomposition.
\end{lemma}
\begin{proof}
We show the results for 3-rd order tensors that are partially symmetric w.r.t. their two first modes. The proof can be straightforwardly extended to tensors of arbitrary order that are partially symmetric w.r.t. any subset of modes. 

Let $\T\in\mathbb{R}^{m\times m \times n }$ be partially symmetric w.r.t. modes $1$ and $2$. We have that $\T_{:,:,i}$ is a symmetric tensor for each $i\in [n]$. By Lemma 4.2 in~\cite{comon2008symmetric}, each tensor $\T_{:,:,i}$ admits a symmetric CP decomposition:
$$\T_{:,:,i} = \CP{\mat{A}^{(i)} ,\mat{A}^{(i)}},\ \ i\in[n]$$
where $\mat{A}^{(i)}\in\mathbb{R}^{m\times R_i}$ and $R_i$ is the symmetric CP rank of $\T_{:,:,i}$. 

By defining $R= \sum_{i=1}^n R_i$ and $\mat{A} = [\mat{A}^{(1)}\ \mat{A}^{(2)}\ \cdots\ \mat{A}^{(n)}]\in\mathbb{R}^{m\times R}$, one can easily check that $\T$ admits the partially symmetric CP decomposition $\T = \CP{\mat{A},\mat{A},\mat{\Delta}}$ where $\mat{\Delta}\in\mathbb{R}^{m\times R}$ is defined by
$$
\mat{\Delta}_{i,r} = \begin{cases}
1&\text{ if } R_1 + \cdots + R_{i-1} < r \leq R_{1} + \cdots + R_{i}\\
0&\text{ otherwise.}
\end{cases}
$$
\end{proof}

\subsection{Graph Neural Networks and Pooling Functions}
Given a graph ${G}=({V},{E})$, a graph neural network always aggregates information in a neighborhood to give node-level representations. During each message-passing iteration, the embedding $\mathbf{h}_v$ corresponding to node $v\in V$ is generated by aggregating features from ${N}(v)$~\cite{hamilton2020graph}.
Formally, at the $l$-th layer of a graph neural network,
\begin{equation}
\begin{aligned}
\label{eq:message.and.update}
&\mathbf{m}^{(l)}_{{N}(v)}= \mbox{AGGREGATE}^{(l)}(\{\mathbf{h}^{(l-1)}_u,\forall u\in N(v)\}), \mathbf{h}_v^{(l)} = \mbox{UPDATE}^{(l)}(\mathbf{h}^{(l-1)}_v, \mathbf{m}^{(l)}_{{N}(v)}),
\end{aligned}
\end{equation}
where $\mbox{AGGREGATE}^{(l)}(\cdot)$ and $\mbox{UPDATE}^{(l)}(\cdot)$ are differentiable functions, the former being permutation-invariant. In words, $\mbox{AGGREGATE}^{(l)}(\cdot)$ first aggregates information from ${N}(v)$, then $\mbox{UPDATE}^{(l)}(\cdot)$ combines the aggregated message and previous node embedding $\mathbf{h}^{(l-1)}_v$ to give a new embedding.

\section{Tensorized Graph Neural Network}
\label{sec:tgnn.theory}
In this section, we introduce the CP-layer and tensorized GNNs~(tGNN). For convenience, we let $\{\mathbf{x}_1,\mathbf{x}_2,...,\mathbf{x}_k\}$ denote features of a node $v$ and its $1$-hop neighbors $N(v)$ such that $|\{v\}\cup N(v)|=k$.

\subsection{Motivation and Method}\label{sec:heterophily_analysis}
\label{sec:time_complexity}

\begin{figure*}[ht!]
\centering
{
\includegraphics[width=0.9\textwidth]{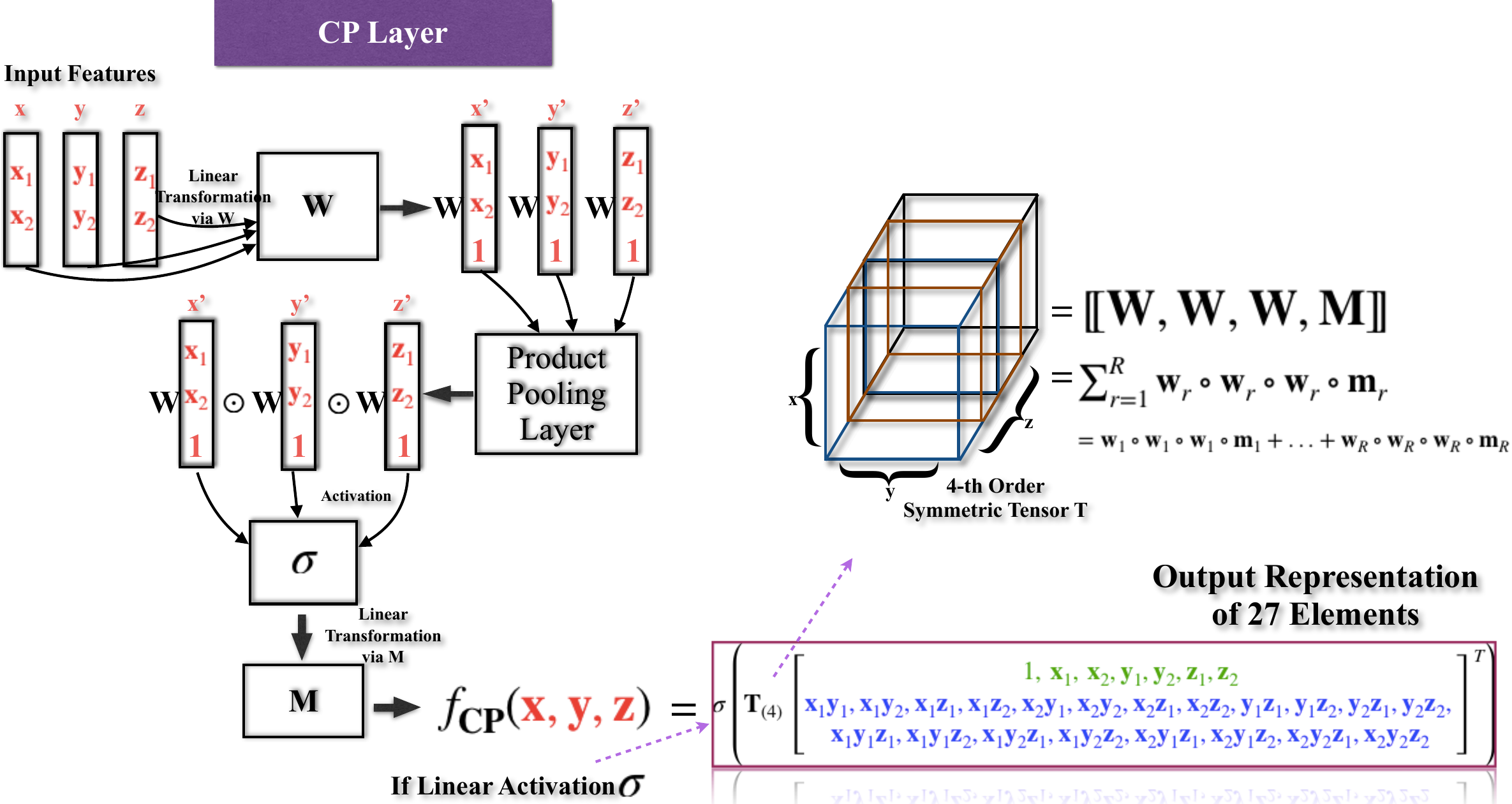}}
{%
  \caption{
  (Left) Sum pooling followed by a FC layer: the output takes individual components of the input into account. 
  (Right) The CP layer can be interpreted as a combination of product pooling with linear layers~(with weight matrices $\mat{W}$ and $\mat{M}$) and non-linearities. 
  The weight matrices of a CP layer corresponds to a partially symmetric CP decomposition of a weight tensor $\tensor{T} = \CP{\mat{W},\mat{W},\mat{W},\mat{M}}$. It  shows that the output of a CP layer takes high-order multiplicative interactions of the inputs' components into account~(in contrast with sum pooling that only considers 1st order terms).
  }%
  \label{fig:tGNNarchi}
}
\end{figure*}

We leverage the symmetric CP decomposition to design an efficient parameterization of permutation-invariant multilinear maps for aggregation operations in  graph neural networks, Tensorized Graph Neural Network (tGNN), resulting in a more expressive high-order node interaction scheme. We visualize the CP pooling layer and  compare it with sum pooling in Fig.~\ref{fig:tGNNarchi}.

Let $\tensor{T}\in \mathbb{R}^{N\times N \times \cdots \times N \times M}$ of order $k+1$ be a  tensor which is partially symmetric w.r.t. its first $k$ modes.  We can parameterize $\tensor{T}$ using a rank $R$ partially symmetric CP decomposition~(see Section~\ref{sec:cp_decomp}):
$
\tensor{T} = \CP{{\mathbf{W},\cdots,\mathbf{W}},\mathbf{M}}
$
where $\mathbf{W}\in\mathbb{R}^{N\times R}$ and $\mathbf{M}\in\mathbb{R}^{M\times R}$. Such a tensor naturally defines a map from $(\mathbb{R}^{N})^k$ to $\mathbb{R}^{M}$ using contractions over the first $k$ modes: 
\begin{equation}
\label{eq:partial.sym.CP.times.vecs}
f(\x_1,\cdots,\x_k) = \tensor{T} \times_1 \x_1 \times_2 \cdots \times_k \x_k = \CP{\underbrace{\mathbf{W},\cdots,\mathbf{W}}_{k \text{ times}},\mathbf{M}} \times_1 \x_1 \times_2 \cdots \times_k \x_k.
\end{equation}
This map satisfies two very important properties for GNNs: it is \emph{permutation-invariant}~(due to the partial symmetry of $\T$) and \emph{its number of parameters is independent of $k$}~(due to the partially symmetric CP parameterization). Thus, using only two parameter matrices of fixed size, the map in Eq.~\ref{eq:partial.sym.CP.times.vecs} can be applied to sets of $N$-dimensional vectors of arbitrary cardinality. In particular, we will show that it can be leveraged to replace both the AGGREGATE and UPDATE functions in GNNs.

There are several way to interpret the map in Eq.~\ref{eq:partial.sym.CP.times.vecs}. First, from Eq.~\ref{eq:tenvec.product.and.kron} we have
\begin{align*}
f(\x_1,\cdots,\x_k) &= \tensor{T} \times_1 \x_1 \times_2 \cdots \times_k \x_k = \tensor{T}_{(k+1)} (\x_k \otimes \x_{k-1} \otimes \cdots \otimes \x_1),
\end{align*}
where $\tensor{T}_{(k+1)} \in\mathbb{R}^{M \times N^k}$ is the mode-$(k+1)$ matricization of $\T$. This shows that each element of the output 
$f(\x_1,\cdots,\x_k)$ is a linear combinations of terms of the form $(\x_1)_{i_1}(\x_2)_{i_2}\cdots(\x_k)_{i_k}$~($k$-th order multiplicative interactions between the components of the vectors $\x_1,\cdots,\x_k$). That is, $f$ is a multivariate polynomial map of order $k$ involving only $k$-th order interactions. By using homogeneous coordinates, i.e., appending an entry equal to one to each of the input tensors $\x_i$, the map $f$ becomes a more general polynomial map taking into account all multiplicative interactions between the $\x_i$ \emph{up to} the $k$-th order:
\begin{align*}
f(\x_1,\cdots,\x_k) &= \tensor{T} \times_1 \left[\x_k\atop 1\right] \times_2 \cdots \times_k \left[\x_1\atop 1\right] = \tensor{T}_{(k+1)} \left(\left[\x_k\atop 1\right] \otimes  \cdots \otimes \left[\x_1\atop 1\right]\right)
\end{align*}
where $\T$ is now of size $(N+1)\times\cdots \times (N+1) \times M$ and can still be parameterized using the partially symmetric CP decomposition $\T=\CP{\mathbf{W}, \cdots, \mathbf{W}, \mathbf{M}}$ with $\mathbf{W}\in\mathbb{R}^{(N+1)\times R}$ and $\mathbf{M}\in\mathbb{R}^{M\times R}$. With this parameterization, one can check that 
$$
f(\x_1,\cdots,\x_k) = \mathbf{M}\left(\!\left(\mathbf{W}^\top\begin{bmatrix}\mathbf{x}_1 \\ 1 \end{bmatrix}\right)\odot  \cdots \odot\left(\mathbf{W}^\top\begin{bmatrix}\mathbf{x}_{k} \\ 1 \end{bmatrix}\right)\!\right)
$$
where $\odot$ denotes the component-wise product between vectors. The map $f$ can thus be seen as the composition of a linear layer with weight $\mathbf{W}$, a multiplicative pooling layer, and another linear map $\mathbf{M}$. Since it is permutation-invariant and can be applied to any number of input vectors, this map can be used as both the aggregation, update, and readout functions of a GNN using non-linear activation functions, which leads us to introduce the novel \emph{CP layer} for GNN. 

\begin{definition}(CP layer)
Given parameter matrices $\mathbf{M}\in \mathbb{R}^{d\times R}$ and $\mathbf{W}\in \mathbb{R}^{F+1\times R}$ and activation functions $\sigma, \sigma'$, a rank $R$ CP layer computes the function  $f_{CP}: \cup_{i\geq 1} (\mathbb{R}^{F})^i \to \mathbb{R}^{d}$ defined by 
\begin{equation*}
f_{\mathbf{CP}}(\x_1,\cdots,\x_k) = \sigma'\left(\mathbf{M}\left(\!\ \sigma\left(\mathbf{W}^\top\begin{bmatrix}\mathbf{x}_1 \\ 1 \end{bmatrix}\odot  \cdots \odot\mathbf{W}^\top\begin{bmatrix}\mathbf{x}_{k} \\ 1 \end{bmatrix}\right)\!\right)\!\right)\ \label{eq:cp.layer.def}
\end{equation*}
for any $k\geq 1$ and any $\x_1,\cdots,\x_k\in\mathbb{R}^{F}$.
\end{definition}
The rank $R$ of a CP layer is a hyperparameter controlling the trade-off between parameter efficiency and expressiveness.
Note that the CP layer computes AGGREGATE and UPDATE~(see Eq.~\ref{eq:message.and.update}) in one step. One can think of the component-wise product of the $\mathbf{W}^\top [\mathbf{x}_{i}\ 1]^\top$  as AGGREGATE, while the UPDATE corresponds to the two non-linear activation functions and linear transformation $\mathbf{M}$. We observed in our experiments that the non-linearity $\sigma$ is crucial to avoid numerical instabilities during training caused by repeated
products of $\mathbf{W}$.
In practice, we use $\textit{Tanh}$ for $\sigma$ and $\textit{ReLU}$ for $\sigma'$. Fig.~\ref{fig:tGNNarchi} graphically explains the computational process of a CP layer, comparing it with a classical sum pooling operation. We intuitively see in this figure that the CP layer is able to capture high order multiplicative interactions that are not modeled by simple aggregation functions such as the sum or the mean. In the next section, we theoretically formalize this intuition. 

\paragraph{Complexity Analysis}
The sum, mean and max poolings result in $O(F_{in}(N+F_{out}))$ time complexity, while CP pooling is $O(R(NF_{in}+F_{out)})$, where $N$ denotes the number of nodes, $F_{in}$ is the input feature dimension, $F_{out}$ is out feature dimension, and $R$ is the CP decomposition rank. In Sec.~\ref{sec:ablation}, we experimentally compare tGNN and CP pooling with various GNNs and pooling techniques to show the model efficiency with limited computation and time budgets.

\subsection{Theoretical Analysis}

We now  analyze the expressive power of CP layers. In order to characterize the set of functions that can be computed by CP layers, we first introduce the notion of multilinear polynomial. A multilinear polynomial is a special kind of vector-valued multivariate polynomial in which no variables appears with a power of 2 or higher. More formally, we have the following definition.  

\begin{definition}
A function $g:\mathbb{R}^k \to \mathbb{R}$ is called a \emph{univariate multilinear polynomial} if it can be written as
$$g(a_1,a_2,\cdots, a_k) = \sum_{i_1=0}^1\cdots\sum_{i_k=0}^1 \tau_{i_1i_2\cdots i_k}a_{1}^{i_1}a_{2}^{i_2} \cdots a_{k}^{i_k}$$
where each $\tau_{i_1i_2\cdots i_k} \in\mathbb{R}$.
The \emph{degree} of a univariate multilinear polynomial is the maximum number of distinct variables occurring in any of the non-zero monomials $\tau_{i_1i_2\cdots i_n}a_{1}^{i_1}a_{2}^{i_2} \cdots a_{k}^{i_k}$.

A function $f: (\mathbb{R}^d)^k \to \mathbb{R}^p$ is called a \emph{ multilinear polynomial map} if there exist univariate multilinear polynomials $g_{i,j_1,\cdots, j_n}$ for $j_1,\cdots,j_k\in [d]$ and $i\in [p]$ such that
$$f(\x_1,\cdots,\x_k)_i = \sum_{j_1,\cdots,j_k = 1}^d g_{i,j_1,\cdots, j_k}((\x_1)_{j_1},\cdots,(\x_k)_{j_k})$$
for all $\x_1,\cdots,\x_k\in \mathbb{R}^d$ and all $i\in [p]$.
The degree of $f$ is the highest degree of the multilinear polynomials $g_{i,j_1,\cdots, j_k}$.
\end{definition}

\begin{figure}[ht!]
  \begin{center}
    \includegraphics[width=0.5\textwidth]{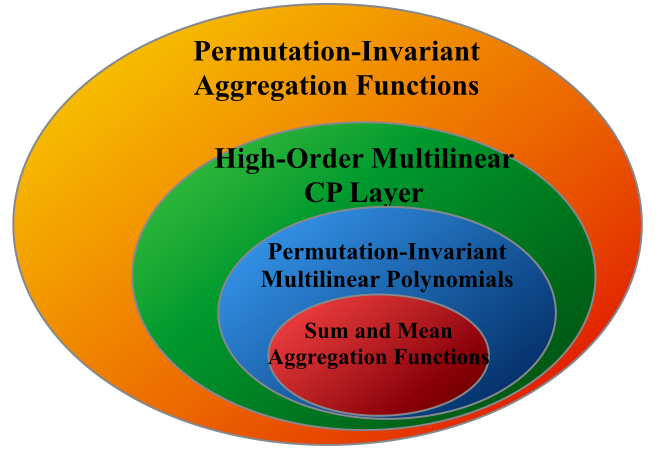}
  \end{center}
  \caption{Visualization of relations of  \{permutation-invariant function space\} $\supseteq$ \{CP function space\} $\supseteq$ \{permutation-invariant multilinear polynomial space\} $\supseteq$ \{sum and mean aggregation functions\}.}%
  \label{fig:aggregationSpace}
\end{figure}

The following theorem shows that CP layers can compute any permutation-invariant multilinear polynomial map. We also visually represent  the expressive power of CP layers in Fig.~\ref{fig:aggregationSpace}, showing that the class of functions computed by CP layer subsumes multilinear polynomials (including sum and mean aggregation functions). 

\begin{theorem}\label{thm:CP.layer.eq.multilin.polynom}
The function computed by a CP layer~(Eq.~\eqref{eq:cp.layer.def}) is permutation-invariant. In addition, any \emph{permutation-invariant}  multilinear polynomial $f:(\mathbb{R}^F)^k\to\mathbb{R}^d$ can be computed  by a CP layer~(with a linear activation function). 
\end{theorem}

\begin{proof}
The fact that the function computed by a CP layer is permutation-invariant directly follows from the definition of the CP layer and the fact that the Hadamard product is commutative.

We now show the second part of the theorem. 
Let $f:(\mathbb{R}^F)^k\to \mathbb{R}^d$ be a permutation-invariant multilinear polynomial map. Then, there exists permutation-invariant univariate multilinear polynomials  $g_{i,j_1,\cdots, j_k}$ for $j_1,\cdots,j_k\in [F]$ and $i\in [d]$ such that
$$f(\x_1,\cdots,\x_k)_i = \sum_{j_1=1}^d \cdots\sum_{j_k = 1}^d g_{i,j_1,\cdots, j_k}((\x_1)_{j_1},\cdots,(\x_k)_{j_k}).$$
Moreover, by definition, each such polynomial satisfies
$$
g_{i,j_1,\cdots, j_k}(a_1,a_2,\cdots, a_k) = \sum_{i_1=0}^1\cdots\sum_{i_k=0}^1 \tau^{(i,j_1,\cdots,j_k)}_{i_1,\cdots, i_k}a_{1}^{i_1}a_{2}^{i_2} \cdots a_{k}^{i_k}$$
for some scalars $\tau^{(i,j_1,\cdots,j_k)}_{i_1,\cdots, i_k}$. Putting those two expressions together, we get
$$
f(\x_1,\cdots,\x_k)_i =\sum_{j_1=1}^d \cdots\sum_{j_k = 1}^d \sum_{i_1=0}^1\cdots\sum_{i_k=0}^1 \tau^{(i,j_1,\cdots,j_k)}_{i_1,\cdots, i_k} (\x_1)_{j_1}^{i_1}\cdots (\x_k)_{j_k}^{i_k}.
$$
We can then group together the coefficients $\tau^{(i,j_1,\cdots,j_k)}_{i_1,\cdots, i_k}$ corresponding to the same monomials $(\x_1)_{j_1}^{i_1}\cdots (\x_k)_{j_k}^{i_k}$. E.g., the coefficients $\tau^{(i,j_1,\cdots,j_k)}_{1,0,\cdots, 0}$ all correspond to the same monomial $(\x_1)_{j_1}$ for any values of $j_2,j_3,\cdots,j_k$. Similarly, all the coefficients $\tau^{(i,j_1,\cdots,j_k)}_{0,1,\cdots, 1}$ all correspond to the same monomial $(\x_2)_{j_2}\cdots (\x_k)_{j_k}$ for any values of $j_1$. By grouping and summing together these coefficients into a tensor $\T\in\mathbb{R}^{(F+1)\times \cdots \times (F+1) \times d}$, where $\T_{j_1,\cdots,j_k,i}$ is equal to the sum  $\displaystyle\sum_{j_\ell\in [F] : j_\ell = F+1 \land i_\ell=0 }\tau^{(i,j_1,\cdots,j_k)}_{i_1,\cdots, i_k}$, we obtain
$$
f(\x_1,\cdots,\x_k)_i 
=
\sum_{j_1=1}^{F+1}\sum_{j_2=1}^{F+1}\cdots\sum_{j_k=1}^{F+1} \T_{j_1,j_2,\cdots,j_k,i} 
\left(\left[\x_{1} \atop 1\right]\right)_{j_1}\left(\left[\x_{2} \atop 1\right]\right)_{j_2}\cdots\left(\left[\x_{k} \atop 1\right]\right)_{j_k}.
$$
Since $f$ is permutation-invariant, the tensor $\T$ is partially symmetric w.r.t.  its first $k$ modes. Thus, by Lemma~\ref{lem:partial.CP.exists}, there exist matrices $\mat{A} \in \mathbb{R}^{R\times F}$ and $\mat{B}\in\mathbb{R}^{d\times R}$ such that
\begin{align*}
f(\x_1,\cdots,\x_k) 
&= 
\T\times_1\left[\x_{1} \atop 1\right]\times_2 \cdots\times_k \left[\x_{k} \atop 1\right]\\
&= \CP{\mat{A},\cdots,\mat{A},\mat{B}} \times_1\left[\x_{1} \atop 1\right]\times_2 \cdots\times_k \left[\x_{k} \atop 1\right] \\
&= \mat{B}\ \sigma\left(\mathbf{A}^\top\begin{bmatrix}\mathbf{x}_1 \\ 1 \end{bmatrix}\odot  \cdots \odot\mathbf{A}^\top\begin{bmatrix}\mathbf{x}_{k} \\ 1 \end{bmatrix}\right)
\end{align*}
with $\sigma$ being the identity function, which concludes the proof.

\end{proof}

Note also that in Fig.~\ref{fig:aggregationSpace} the CP layer is strictly more expressive than permutation invariant multilinear polynomials due to the non-linear activation functions in Def.~\ref{eq:cp.layer.def}.
Since the classical sum and mean pooling aggregation functions are degree 1 multilinear polynomial maps, it readily follows from the previous theorem that the CP layer is more expressive than these standards aggregation functions. However, it is natural to ask how many parameters a CP layer needs to compute sums and means. We answer this question in the following theorem. 

\begin{theorem}\label{thm:CP.layer.vs.sum.pooling}
A CP layer of rank $F\cdot k$  can compute the sum and mean aggregation functions over $k$ vectors in $\mathbb{R}^F$.

Consequently, for any $k\geq 1$ and any GNN $\mathcal{N}$ using mean or sum pooling with feature and embedding dimensions bounded by $F$, there exists a GNN with CP layers of rank $F\cdot k$ computing the same function as $\mathcal{N}$ over all graphs of uniform degree $k$.
\end{theorem}

\begin{proof}
We will show that the function $f:\left(\mathbb{R}^F\right)^k\to\mathbb{R}^d$ defined by
$$f(\x_1,\cdots,\x_k) = \sum_{i=1}^k \alpha \x_i$$
where $\alpha\in\mathbb{R}$, can be computed by a CP layer of rank $F k$, which will show the first part of the theorem. The second part of the theorem directly follows by letting $\alpha=1$ for the sum aggregation and $\alpha=1/k$ for the mean aggregation.

Let $\tilde{\mat{e}}_1,\cdots,\tilde{\mat{e}}_{F+1}$ be the canonical basis of $\mathbb{R}^{F+1}$, and let ${\mat{e}_1},\cdots,{\mat{e}_{F}}$ be the canonical basis of $\mathbb{R}^{F}$. We define the tensor $\T\in\mathbb{R}^{(F+1)\times\cdots\times(F+1)\times d}$  by
$$
\T = \sum_{j=1}^d \sum_{\ell=1}^k 
\alpha 
\underbrace{\tilde{\mat{e}}_{F+1}\circ\cdots\circ \tilde{\mat{e}}_{F+1}}_{\ell-1\text{ times}}
\circ\ \tilde{\mat{e}}_j \circ
\underbrace{\tilde{\mat{e}}_{F+1}\circ\cdots\circ \tilde{\mat{e}}_{F+1}}_{k-\ell\text{ times}} \circ\ \mat{e}_j
$$
where $\circ$ denotes the outer~(or tensor) product between vectors. We start by showing that contracting $k$ vectors in homogeneous coordinates along the first $k$ modes of $\T$ results in the sum of those vectors weighted by $\alpha$. For any vectors $\x_1,\cdots,\x_k\in\mathbb{R}^F$, we have
\begin{align*}
&\T  \times_1\left[\x_{1} \atop 1\right]\times_2 \cdots \times_k \left[\x_{k} \atop 1\right] \\
&=\left(\sum_{j=1}^d \sum_{\ell=1}^k 
\alpha 
\underbrace{\tilde{\mat{e}}_{F+1}\circ\cdots\circ \tilde{\mat{e}}_{F+1}}_{\ell-1\text{ times}}
\circ\ \tilde{\mat{e}}_j \circ
\underbrace{\tilde{\mat{e}}_{F+1}\circ\cdots\circ \tilde{\mat{e}}_{F+1}}_{k-\ell\text{ times}} \circ\ \mat{e}_j\right)
\times_1\left[\x_{1} \atop 1\right]\times_2 \cdots\times_k \left[\x_{k} \atop 1\right] \\
&=
\sum_{j=1}^d \sum_{\ell=1}^k \alpha 
\left\langle \tilde{\mat{e}}_{F+1}, \left[\x_{1} \atop 1\right] \right\rangle \cdots \left\langle\tilde{\mat{e}}_{F+1},\left[\x_{\ell-1} \atop 1\right]\right\rangle
 \left\langle\tilde{\mat{e}}_j,\left[\x_{\ell} \atop 1\right]\right\rangle
\left\langle\tilde{\mat{e}}_{F+1},\left[\x_{\ell+1} \atop 1\right]\right\rangle\cdots \left\langle\tilde{\mat{e}}_{F+1},\left[\x_{k} \atop 1\right]\right\rangle  \mat{e}_j \\
&=
\sum_{j=1}^d \sum_{\ell=1}^k (\alpha \cdot
1 \cdots 1 \cdot \langle \x_\ell,\mat{e}_j \rangle \cdot 1 \cdots 1 )\ \mat{e}_j \\
&= \sum_{j=1}^d \left\langle  \sum_{\ell=1}^k \alpha \x_\ell , \mat{e}_j \right\rangle\mat{e}_j \\
&=  \sum_{\ell=1}^k \alpha \x_\ell = f(\x_1,\cdots,\x_k).
\end{align*}

To show that $f$ can be computed by a CP layer of rank $Fk$, it thus remains to show that $\T$ admits a partially symmetric CP decomposition of rank $Fk$. This follows from Corollary~4.3 in~\cite{zhang2016comon}, which states that any $k$th order tensor $\tensor{A}$ of CP rank less than $k$ has symmetric CP rank bounded by $k$. Indeed, consider the tensors 
$$\tensor{A}^{(j)} = \sum_{\ell=1}^k 
\alpha 
\underbrace{\tilde{\mat{e}}_{F+1}\circ\cdots\circ \tilde{\mat{e}}_{F+1}}_{\ell-1\text{ times}}
\circ\ \tilde{\mat{e}}_j \circ
\underbrace{\tilde{\mat{e}}_{F+1}\circ\cdots\circ \tilde{\mat{e}}_{F+1}}_{k-\ell\text{ times}}
$$
for $j\in [d]$. They are all $k$-th order tensor of CP rank bounded by $k$. Thus, by Corollary~4.3 in~\cite{zhang2016comon}, they all admit a symmetric CP decomposition of rank at most $k$, from which it directly follows that the tensor $\T = \sum_{j=1}^F \tensor{A}^{(j)} \circ\ \mat{e}_j $ admits a partially symmetric CP decomposition of rank at most $Fk$.
\end{proof}

It follows from this theorem  that a CP layer with $2F^2k$ can compute sum and mean aggregation over sets of $k$ vectors.
While Theorem~\ref{thm:CP.layer.vs.sum.pooling} shows that any function using sum and mean aggregation can be computed by a CP layer, the next theorem shows that the converse is not true, i.e., the CP layer is a strictly more expressive aggregator than using the mean or sum. 

\begin{theorem}\label{thm:sum.pooling.le.CP.layer}
With probability one, any function $f_{CP}:(\mathbb{R}^{F})^k \to \mathbb{R}^{d}$ computed by a CP layer~(of any rank) whose parameters are drawn randomly~(from a  distribution which is continuous w.r.t. the Lebesgue measure) cannot be computed by a function of the form
$$g_{\textit{sum}} : \x_1,\cdots,\x_k \mapsto  \sigma'\left(\mat{M}\left( \sigma\left(\sum_{i=1}^k \mat{W}^\top \x_i \right)\right)\right)$$
where $\mat{M}\in\mathbb{R}^{d\times R},\ \mat{W}\in\mathbb{R}^{F\times R}$ and $\sigma$,  $\sigma'$ are component-wise activation function.
\end{theorem}
\begin{proof}
Let $f_{CP}:(\mathbb{R}^{F})^k \to \mathbb{R}^{d}$ be the function computed by a random CP layer. I.e.,
$$
f_{CP}(\x_1,\cdots,\x_k) = \mathbf{A}\left(\!\left(\mathbf{B}^\top\begin{bmatrix}\mathbf{x}_1 \\ 1 \end{bmatrix}\right)\odot  \cdots \odot\left(\mathbf{B}^\top\begin{bmatrix}\mathbf{x}_{k} \\ 1 \end{bmatrix}\right)\!\right)
$$
where the entries of the matrices $\mathbf{B}\in \mathbb{R}^{(N+1)\times R}$ and $\mathbf{A}\in \mathbb{R}^{M\times R}$ are identically and independently drawn from a distribution which is continuous w.r.t. the Lebesgue measure. It is well know that since the entries of the two parameter matrices are drawn from a continuous distribution, all the entries of $\mat{A}$ and $\mat{B}$ are non-zero and distinct with probability one. It follows that all the entries of the vector $f_{CP}(\x_1,\cdots,\x_k)$ are  $k$-th order multilinear polynomials of the entries of the input vectors $\x_1,\cdots,\x_k$, which, with probability one, have non trivial high-order interactions that cannot be computed by a linear map. In particular, with probability one, the map $f_{CP}$ cannot be computed by any map of the form $g_{\textit{sum}} : \x_1,\cdots,\x_k \mapsto  \sigma'\left(\mat{M}\left( \sigma\left(\sum_{i=1}^k \mat{W}^\top \x_i \right)\right)\right)$ since the sum pooling aggregates the inputs in a way that prevents modeling independent higher order multiplicative interactions, despite the non-linear activation functions. 
\end{proof}

This  theorem not only shows that there exist functions computed by CP layers that cannot be computed using sum pooling, but that this is the case for \emph{almost all} functions that can be computed by~(even rank-one) CP layers.

From an expressive power viewpoint, we showed that a CP layer is able to leverage both low and high-order multiplicative interactions. However, from a learning perspective, it is clear that the CP layer has a natural inductive bias towards capturing high-order interactions. We are not enforcing any sparsity in the tensor parameterizing the polynomial, thus the number and magnitude
of weights corresponding to high-order terms will dominate the result (intuitively, learning a low order polynomial would imply setting most of these weights to zero). In order to counterbalance this bias, we complement the CP layer with simple but efficient linear low-order interactions (reminiscent of the idea behind residual networks~\cite{he2016deep}) when using it in tGNN: 
\begin{equation}
 f(\x_1,\cdots,\x_k) = \sigma'\left(\mathbf{M}\left(\!\ \sigma\left(\mathbf{W}_1^\top\begin{bmatrix}\mathbf{x}_1 \\ 1 \end{bmatrix}\odot  \cdots \odot\mathbf{W}_1^\top\begin{bmatrix}\mathbf{x}_{k} \\ 1 \end{bmatrix}\right)\!\right)\!\right) +\sigma''(\mathbf{W}_2^\top\mathbf{x}_1+...+\mathbf{W}_2^\top\mathbf{x}_{k})\label{eq:cp_sum}
\end{equation}
where the first term corresponds to the CP layer and the second one to a standard sum pooling layer~(with $\sigma$, $\sigma'$ and $\sigma''$ being activation functions).

\section{Experiments on Real-World Datasets}
\label{sec:tgnn.experiments}
In this section, we evaluate Tensorized Graph Neural Net on real-world node- and graph-level datasets. We introduce experiment setup in \ref{sec:setup}, compare tGNN with the state-of-the-arts models in \ref{sec:node_task}, and conduct ablation study on model performance and efficiency in \ref{sec:ablation}. The hyperparameter and computing resources  are attached in Tab~\ref{tab:searching.range}, \ref{tab:tGNN.hyper}. Dataset information can be found in Tab.~\ref{tab:node.stats}, \ref{tab:graph.stats}.

\subsection{Experiment Setup}
\label{sec:setup}
In this work, we conduct experiments on three citation networks (\textit{Cora}, \textit{Citeseer}, \textit{Pubmed}) and three \textit{OGB} datasets (\textit{PRODUCTS}, \textit{ARXIV}, \textit{PROTEINS})~\cite{hu2020open} for node-level tasks, one \textit{OGB} dataset (\textit{MolHIV})~\cite{hu2020open} and three \textit{benchmarking} datasets (\textit{ZINC}, \textit{CIFAR10}, \textit{MNIST})~\cite{dwivedi2020benchmarking} for graph-level tasks. 

\begin{table}[ht!]
  \centering
  \caption{Statistics of node graphs.}
    \begin{tabular}{|l|r|r|r|r|r|}
    \hline
    Dataset & \multicolumn{1}{l|}{\#Nodes} & \multicolumn{1}{l|}{\#Edges} & \multicolumn{1}{l|}{\#Node Features} & \multicolumn{1}{l|}{\#Edge Features} & \multicolumn{1}{l|}{\#Classes}    \\
    \hline
    \textit{Cora}  & 2,708 & 5,429 & 1,433 & /     & 7    \\
    \hline
    \textit{Citeseer} & 3,327 & 4,732 & 3,703 & /     & 6    \\
    \hline
    \textit{Pubmed} & 19,717 & 44,338 & 500   & /     & 3    \\
    \hline
    \textit{PRODUCTS} & 2,449,029 & 61,859,140 & 100   & /     & 47    \\
    \hline
    \textit{ARXIV} & 169,343 & 1,166,243 & 128   & /     & 40    \\
    \hline
    \textit{PROTEINS} & 132,534 & 9,561,252 & 8     & 8     & 112    \\
    \hline
    \end{tabular}%
  \label{tab:node.stats}%
\end{table}%

\begin{table}[ht!]
  \centering
  \caption{Statistics of graph datasets}
    \begin{tabular}{|l|r|r|r|}
    \hline
    Dataset & \multicolumn{1}{l|}{\#Graphs} & \multicolumn{1}{l|}{\#Node Features} & \multicolumn{1}{l|}{\#Classes}    \\
    \hline
    \textit{ZINC}  & 12,000 & 28    & /    \\
    \hline
    \textit{CIFAR10} & 60,000 & 5     & 10    \\
    \hline
    \textit{MNIST} & 70,000 & 3     & 10    \\
    \hline
    \textit{MolHIV} & 41,127 & 9     & 2    \\
    \hline
    \end{tabular}%
  \label{tab:graph.stats}%
\end{table}%

\paragraph{Training Procedure}
For three citation networks (\textit{Cora}, \textit{Citeseer}, \textit{Pubmed}), we run experiments 10 times on each dataset with 60\%/20\%/20\% random splits used in \cite{chien2021adaptive}, and report accuracy with standard deviation 
in Tab.~\ref{tab:node_task}.
For data splits of \textit{OGB} node and graph datasets, we follow \cite{hu2020open}, run experiments 5 times on each dataset (due to training cost), and report results in Tab.~\ref{tab:node_task}, \ref{tab:graph_classification}. For \textit{benchmarking} datasets, we run experiments 5 times on each dataset with data split used in \cite{dwivedi2020benchmarking}, and report results in Tab.~\ref{tab:graph_classification}. To avoid numerical instability and floating point exception in tGNN training, we sample 5 neighbors for each node.
For graph datasets, we do not sample because the training is already
in batch thus numerical instability can be avoided, and we apply the CP pooling at both node-level aggregation and graph-level readout.

\paragraph{Training Procedure}
For three citation networks (\textit{Cora}, \textit{Citeseer}, \textit{Pubmed}), we run experiments 10 times on each dataset with 60\%/20\%/20\% random splits used in \cite{chien2021adaptive}, and report accuracy with standard deviation 
in Tab.~\ref{tab:node_task}.
For data splits of \textit{OGB} node and graph datasets, we follow \cite{hu2020open}, run experiments 5 times on each dataset (due to training cost), and report results in Tab.~\ref{tab:node_task}, \ref{tab:graph_classification}. For \textit{benchmarking} datasets, we run experiments 5 times on each dataset with data split used in \cite{dwivedi2020benchmarking}, and report results in Tab.~\ref{tab:graph_classification}. To avoid numerical instability and floating point exception in tGNN training, we sample 5 neighbors for each node.
For graph datasets, we do not sample because the training is already
in batch thus numerical instability can be avoided, and we apply the CP pooling at both node-level aggregation and graph-level readout.

\paragraph{Model Comparison}
tGNN has two hyperparameters, hidden unit and decomposition rank, we fix hidden unit and explore decomposition rank. For citation networks, we compare 2-layer GNNs with 32 hidden units. And for \textit{OGB} and \textit{benchmarking} datasets, we use 32 hidden units for tGNN, and the results for all other methods are reported from the leaderboards and corresponding references.

Particularly, tGNN and CP pooling are more effective and expressive than existing pooling techniques for GNNs on two citation networks, two \textit{OGB} node datasets, and one \textit{OGB} graph dataset in Tab.~\ref{tab:node_task}, \ref{tab:graph_classification}.

\begin{table}[ht!]
  \centering
  \caption{Hyperparameter searching range corresponding to Section~\ref{sec:experiments}.}
    \begin{tabular}{|l|c|c|c|c|}
    \hline
    Hyperparameter & \multicolumn{4}{c|}{Searing Range}  \\
    \hline
    learning rate    & \multicolumn{4}{c|}{\{0.01, 0.001, 0.0001, 0.05, 0.005, 0.0005, 0.003\}}  \\
    \hline
    weight decay    & \multicolumn{4}{c|}{\{5e-5 ,5e-4, 5e-3, 1e-5, 1e-4, 1e-3, 0\}}  \\
    \hline
    dropout & \multicolumn{4}{c|}{\{0, 0.1, 0.3, 0.5, 0.7, 0.8, 0.9\}}  \\
    \hline
    $R$     & \multicolumn{4}{c|}{\{32, 64, 128, 256, 512, 25, 50, 75, 100, 200\}}  \\
    \hline
    \end{tabular}%
  \label{tab:searching.range}%
\end{table}%

\begin{table}[ht!]
  \centering
  \caption{Hyperparameters for tGNN corresponding to Section~\ref{sec:experiments}.}
    \begin{tabular}{|l|r|r|r|r|}
    \hline Dataset
          & \multicolumn{1}{l|}{learning rate} & \multicolumn{1}{l|}{weight decay} & \multicolumn{1}{l|}{dropout} & \multicolumn{1}{l|}{$R$}  \\
    \hline
    \textit{Cora}  & 0.001 & 5.00E-05 & 0.9   & 512  \\
    \hline
    \textit{Citeseer} & 0.001 & 1.00E-04 & 0     & 512  \\
    \hline
    \textit{Pubmed} & 0.005 & 5.00E-04 & 0.1   & 512  \\
    \hline
    \textit{PRODUCTS} & 0.001 & 5.00E-05 & 0.3   & 128  \\
    \hline
    \textit{ARXIV} & 0.003 & 5.00E-05 & 0     & 512  \\
    \hline
    \textit{PROTEINS} & 0.0005 & 5.00E-04 & 0.9   & 50  \\
    \hline
    \textit{ZINC}  & 0.005 & 5.00E-04 & 0     & 100  \\
    \hline
    \textit{CIFAR10} & 0.005 & 1.00E-04 & 0     & 100  \\
    \hline
    \textit{MNIST} & 0.005 & 5.00E-05 & 0     & 75  \\
    \hline
    \textit{MolHIV} & 0.001 & 5.00E-05 & 0.8   & 100  \\
    \hline
    \end{tabular}%
  \label{tab:tGNN.hyper}%
\end{table}%

\subsection{Real-world Datasets}
\label{sec:node_task}
In this section, we present tGNN performance on node- and graph-level tasks. We compare tGNN with several classic baseline models under the same training setting. For three citation networks, we compare tGNN with several baselines including GCN \cite{kipf2016classification}, GAT \cite{velivckovic2017attention}, GraphSAGE \cite{hamilton2017inductive}, H$_2$GCN \cite{zhu2020beyond}, GPRGNN \cite{chien2021adaptive}, APPNP \cite{klicpera2018predict} and MixHop \cite{abu2019mixhop}; for three \textit{OGB} node datasets, we compare tGNN with MLP, Node2vec \cite{grover2016node2vec}, GCN \cite{kipf2016classification}, GraphSAGE \cite{hamilton2017inductive} and DeeperGCN \cite{li2020deepergcn}.
And for graph-level tasks, we compare tGNN with several baselines including MLP, GCN \cite{kipf2016classification}, GIN \cite{xu2018powerful}, DiffPool \cite{ying2018hierarchical}, GAT \cite{velivckovic2017attention}, MoNet \cite{monti2017geometric}, GatedGCN \cite{bresson2017residual}, PNA \cite{corso2020principal}, PHMGNN \cite{le2021parameterized} and DGN \cite{beani2021directional}.

From Tab.~\ref{tab:node_task}, we can observe that tGNN outperforms all classic baselines on \textit{Cora}, \textit{Pubmed}, \textit{PRODUCTS} and \textit{ARXIV}, and have slight improvements on the other datasets but underperforms GCN on \textit{Citeseer} and DeeperGCN on \textit{PROTEINS}. On the citation networks, tGNN outperforms others on 2 out of 3 datasets. Moreover, on the \textit{OGB} node datasets, even when tGNN
is not ranked first, it is still very competitive (top 3 for all
datasets). We believe it is reasonable and expected that tGNN does not outperform all methods
on all datasets. But overall tGNN shows very competitive performance and deliver significant improvement on challenging graph benchmarks compared to popular commonly used pooling methods (with comparable computational cost).
\begin{table}[ht!]
      \centering
    \caption{Results of node-level tasks. \textbf{Left Table}: tGNN in comparison with GNN architectures on citation networks. \textbf{Right Table}: tGNN in comparison with GNN
    architectures on \textit{OGB} datasets.}
    \begin{minipage}{.5\linewidth}
        \resizebox{.9\textwidth}{!}{
    \begin{tabular}{|c|ccc|}
    \hline
    \rowcolor[rgb]{ 0,  0,  0} \multicolumn{1}{|c}{\textcolor[rgb]{ 1,  1,  1}{DATASET}} & \cellcolor[rgb]{ .502,  .392,  .635}\textcolor[rgb]{ 1,  1,  1}{\textit{Cora}} & \cellcolor[rgb]{ .376,  .286,  .478}\textcolor[rgb]{ 1,  1,  1}{\textit{Citeseer}} & \cellcolor[rgb]{ .592,  .278,  .024}\textcolor[rgb]{ 1,  1,  1}{\textit{Pubmed}}\\
    \rowcolor[rgb]{ 0,  0,  0} \multicolumn{1}{|c}{\textcolor[rgb]{ 1,  1,  1}{MODEL}} & \textcolor[rgb]{ 1,  1,  1}{Acc} & \textcolor[rgb]{ 1,  1,  1}{Acc} & \textcolor[rgb]{ 1,  1,  1}{Acc} \\
\cline{1-1}    GCN   & \cellcolor[rgb]{ .392,  .749,  .486}0.8778$\pm$0.0096 & \cellcolor[rgb]{ .388,  .745,  .482}0.8139$\pm$0.0123 & \cellcolor[rgb]{ .624,  .816,  .498}0.8890$\pm$0.0032 \\
\cline{1-1}    GAT   & \cellcolor[rgb]{ .592,  .792,  .494}0.8686$\pm$0.0042 & \cellcolor[rgb]{ .992,  .804,  .494}0.6720$\pm$0.0046 & \cellcolor[rgb]{ .973,  .412,  .42}0.8328$\pm$0.0012 \\
\cline{1-1}    GraphSAGE  & \cellcolor[rgb]{ .549,  .792,  .494}0.8658$\pm$0.0026 & \cellcolor[rgb]{ .741,  .847,  .506}0.7624$\pm$0.0030 & \cellcolor[rgb]{ .996,  .886,  .51}0.8658$\pm$0.0011 \\
\cline{1-1}    H$_{2}$GCN & \cellcolor[rgb]{ .427,  .757,  .486}0.8752$\pm$0.0061 & \cellcolor[rgb]{ .486,  .776,  .49}0.7997$\pm$0.0069 & \cellcolor[rgb]{ .827,  .875,  .51}0.8778$\pm$0.0028 \\
\cline{1-1}    GPRGNN & \cellcolor[rgb]{ .992,  .816,  .494}0.7951$\pm$0.0036 & \cellcolor[rgb]{ .992,  .812,  .494}0.6763$\pm$0.0038 & \cellcolor[rgb]{ .984,  .667,  .467}0.8507$\pm$0.0009 \\
\cline{1-1}    APPNP & \cellcolor[rgb]{ .992,  .812,  .494}0.7941$\pm$0.0038 & \cellcolor[rgb]{ .992,  .835,  .498}0.6859$\pm$0.0030 & \cellcolor[rgb]{ .984,  .663,  .467}0.8502$\pm$0.0009 \\
\cline{1-1}    MixHop & \cellcolor[rgb]{ .973,  .412,  .42}0.6565$\pm$0.1131 & \cellcolor[rgb]{ .973,  .412,  .42}0.4952$\pm$0.1335 & \cellcolor[rgb]{ .961,  .91,  .518}0.8704$\pm$0.0410 \\
\cline{1-1}    tGNN  & \cellcolor[rgb]{ .388,  .745,  .482}0.8808$\pm$0.0131 & \cellcolor[rgb]{ .463,  .769,  .49}0.80.51$\pm$0.0192 & \cellcolor[rgb]{ .388,  .745,  .482}0.9080$\pm$0.0018 \\
    \hline
    \end{tabular}
    }
    \end{minipage}%
    \begin{minipage}{.5\linewidth}
        \resizebox{1\textwidth}{!}{
    \begin{tabular}{|c|ccc|}
    \hline
    \rowcolor[rgb]{ 0,  0,  0} \textcolor[rgb]{ 1,  1,  1}{DATASET} & \multicolumn{1}{c|}{\cellcolor[rgb]{ .502,  0,  0}\textcolor[rgb]{ 1,  1,  1}{\textit{PRODUCTS}}} & \multicolumn{1}{c|}{\cellcolor[rgb]{ .4,  0,  .4}\textcolor[rgb]{ 1,  1,  1}{\textit{ARXIV}}} & \cellcolor[rgb]{ .98,  .749,  .561}\textcolor[rgb]{ 1,  1,  1}{\textit{PROTEINS}} \\
    \hline
    \rowcolor[rgb]{ 0,  0,  0} \textcolor[rgb]{ 1,  1,  1}{MODEL} & \multicolumn{1}{c|}{\textcolor[rgb]{ 1,  1,  1}{Acc}} & \multicolumn{1}{c|}{\textcolor[rgb]{ 1,  1,  1}{Acc}} & \textcolor[rgb]{ 1,  1,  1}{AUC} \\
    \hline
    MLP   & \cellcolor[rgb]{ .973,  .412,  .42}\textcolor[rgb]{ .129,  .129,  .129}{0.6106$\pm$0.0008} & \cellcolor[rgb]{ .973,  .412,  .42}\textcolor[rgb]{ .129,  .129,  .129}{0.5550$\pm$0.0023} & \cellcolor[rgb]{ .984,  .671,  .467}\textcolor[rgb]{ .129,  .129,  .129}{0.7204$\pm$0.0048} \\
\cline{1-1}    Node2vec & \cellcolor[rgb]{ .988,  .773,  .486}\textcolor[rgb]{ .129,  .129,  .129}{0.7249$\pm$0.0010} & \cellcolor[rgb]{ .996,  .871,  .506}\textcolor[rgb]{ .129,  .129,  .129}{0.7007$\pm$0.0013} & \cellcolor[rgb]{ .973,  .412,  .42}\textcolor[rgb]{ .129,  .129,  .129}{0.6881$\pm$0.0065} \\
\cline{1-1}    GCN   & \cellcolor[rgb]{ .996,  .875,  .506}\textcolor[rgb]{ .129,  .129,  .129}{0.7564$\pm$0.0021} & \cellcolor[rgb]{ .988,  .918,  .518}\textcolor[rgb]{ .129,  .129,  .129}{0.7174$\pm$0.0029} & \cellcolor[rgb]{ .988,  .71,  .475}\textcolor[rgb]{ .129,  .129,  .129}{0.7251$\pm$0.0035} \\
\cline{1-1}    GraphSAGE & \cellcolor[rgb]{ .816,  .871,  .51}\textcolor[rgb]{ .129,  .129,  .129}{0.7850$\pm$0.0016} & \cellcolor[rgb]{ .996,  .914,  .514}\textcolor[rgb]{ .129,  .129,  .129}{0.7149$\pm$0.0027} & \cellcolor[rgb]{ .855,  .882,  .51}\textcolor[rgb]{ .129,  .129,  .129}{0.7768$\pm$0.0020} \\
\cline{1-1}    DeeperGCN & \cellcolor[rgb]{ .494,  .776,  .49}\textcolor[rgb]{ .129,  .129,  .129}{0.8098$\pm$0.0020} & \cellcolor[rgb]{ .965,  .914,  .518}\textcolor[rgb]{ .129,  .129,  .129}{0.7192$\pm$0.0016} & \cellcolor[rgb]{ .388,  .745,  .482}\textcolor[rgb]{ .129,  .129,  .129}{0.8580$\pm$0.0017} \\
\cline{1-1}    tGNN  & \cellcolor[rgb]{ .388,  .745,  .482}\textcolor[rgb]{ .129,  .129,  .129}{0.8179$\pm$0.0054} & \cellcolor[rgb]{ .388,  .745,  .482}\textcolor[rgb]{ .129,  .129,  .129}{0.7538$\pm$0.0015} & \cellcolor[rgb]{ .576,  .8,  .494}\textcolor[rgb]{ .129,  .129,  .129}{0.8255$\pm$0.0049} \\
    \hline
    \end{tabular}
    }
    \end{minipage} 
    \label{tab:node_task}
\end{table}

In Tab.~\ref{tab:graph_classification}, we present tGNN performance on graph property prediction tasks. tGNN achieves state-of-the-arts results on  \textit{MolHIV}, and have slight improvements on other three \textit{Benchmarking} graph datasets. Overall tGNN achieves more effective and accurate results on 5 out of 10 datasets comparing with existing pooling techniques, which suggests that high-order CP pooling can leverage a GNN to generalize better node embeddings and graph representations.
\begin{table}[htbp!]
    \begin{minipage}{.5\linewidth}
      \caption{Results of tGNN on \\ graph-level tasks in comparison with \\ GNN architectures.}
        \resizebox{.9\textwidth}{!}{
        \begin{tabular}{|c|c|cccc|}
    \rowcolor[rgb]{ 0,  0,  0} \multicolumn{2}{c}{\textcolor[rgb]{ 1,  1,  1}{DATASET}} & \multicolumn{1}{c}{\cellcolor[rgb]{ 1,  0,  0}\textcolor[rgb]{ 1,  1,  1}{\textit{ZINC}}} & \multicolumn{1}{c}{\cellcolor[rgb]{ .31,  .506,  .741}\textcolor[rgb]{ 1,  1,  1}{\textit{CIFAR10}}} & \multicolumn{1}{c}{\cellcolor[rgb]{ .969,  .588,  .275}\textcolor[rgb]{ 1,  1,  1}{\textit{MNIST}}} & \multicolumn{1}{c}{\cellcolor[rgb]{ 0,  .502,  0}\textcolor[rgb]{ 1,  1,  1}{\textit{MolHIV}}} \\
    \rowcolor[rgb]{ 0,  0,  0} \multicolumn{2}{c}{\multirow{3}[1]{*}{\textcolor[rgb]{ 1,  1,  1}{MODEL}}} & \multicolumn{1}{c}{\textcolor[rgb]{ 1,  1,  1}{No edge}} &  \multicolumn{1}{c}{\textcolor[rgb]{ 1,  1,  1}{No edge}} &  \multicolumn{1}{c}{\textcolor[rgb]{ 1,  1,  1}{No edge}} &  \multicolumn{1}{c}{\textcolor[rgb]{ 1,  1,  1}{No edge}} \\
    \rowcolor[rgb]{ 0,  0,  0} \multicolumn{2}{c}{} & \multicolumn{1}{c}{\textcolor[rgb]{ 1,  1,  1}{features}} &  \multicolumn{1}{c}{\textcolor[rgb]{ 1,  1,  1}{features}} &  \multicolumn{1}{c}{\textcolor[rgb]{ 1,  1,  1}{features}} &  \multicolumn{1}{c}{\textcolor[rgb]{ 1,  1,  1}{features}} \\
    \rowcolor[rgb]{ 0,  0,  0} \multicolumn{2}{c}{} & \multicolumn{1}{c}{\textcolor[rgb]{ 1,  1,  1}{MAE}} &  \multicolumn{1}{c}{\textcolor[rgb]{ 1,  1,  1}{Acc}} &  \multicolumn{1}{c}{\textcolor[rgb]{ 1,  1,  1}{Acc}} &  \multicolumn{1}{c}{\textcolor[rgb]{ 1,  1,  1}{AUC}}  \\
    \hline
          & MLP   & \cellcolor[rgb]{ .973,  .412,  .42}0.710$\pm$0.001 & \cellcolor[rgb]{ .98,  .573,  .451}0.560$\pm$0.009 & \cellcolor[rgb]{ .996,  .863,  .506}0.945$\pm$0.003 &    \\
\cline{2-2}    Dwivedi & GCN   & \cellcolor[rgb]{ .996,  .831,  .502}0.469$\pm$0.002 & \cellcolor[rgb]{ .973,  .482,  .431}0.545$\pm$0.001 & \cellcolor[rgb]{ .973,  .412,  .42}0.899$\pm$0.002& \cellcolor[rgb]{ .973,  .478,  .431}0.761$\pm$0.009  \\
\cline{2-2}    et al. & GIN   & \cellcolor[rgb]{ .976,  .914,  .514}0.408$\pm$0.008 & \cellcolor[rgb]{ .973,  .412,  .42}0.533$\pm$0.037 & \cellcolor[rgb]{ .992,  .812,  .494}0.939$\pm$0.013 & \cellcolor[rgb]{ .973,  .412,  .42}0.756$\pm$0.014  \\
\cline{2-2}    and Hu & DiffPool & \cellcolor[rgb]{ .996,  .835,  .502}0.466$\pm$0.006 & \cellcolor[rgb]{ .984,  .694,  .471}0.579$\pm$0.005 &        \cellcolor[rgb]{ 1,  .922,  .518}0.950$\pm$0.004 &   \\
\cline{2-2}    et al. & GAT   & \cellcolor[rgb]{ .996,  .839,  .502}0.463$\pm$0.002 & \cellcolor[rgb]{ .792,  .863,  .506}0.655$\pm$0.003 &       \cellcolor[rgb]{ .847,  .878,  .51}0.956$\pm$0.001 &        \\
\cline{2-2}    \multirow{2}[4]{*}{} & MoNet & \cellcolor[rgb]{ .973,  .914,  .514}0.407$\pm$0.007 & \cellcolor[rgb]{ .973,  .42,  .42}0.534$\pm$0.004 & \cellcolor[rgb]{ .973,  .447,  .424}0.904$\pm$0.005 &   \\
\cline{2-2}          & GatedGCN & \cellcolor[rgb]{ 1,  .91,  .518}0.422$\pm$0.006 &  \cellcolor[rgb]{ .584,  .804,  .494}0.692$\pm$0.003 &  \cellcolor[rgb]{ .388,  .745,  .482}0.974$\pm$0.001  &   \\
\cline{1-2}    Corso et al. & PNA   & \cellcolor[rgb]{ .702,  .835,  .498}0.320$\pm$0.032 & \cellcolor[rgb]{ .529,  .788,  .494}0.702$\pm$0.002 & \cellcolor[rgb]{ .435,  .761,  .486}0.972$\pm$0.001 & \cellcolor[rgb]{ .996,  .898,  .51}0.791$\pm$0.013  \\
\cline{1-2}    Le et al. & PHM-GNN &       &  &  & \cellcolor[rgb]{ .89,  .89,  .514}0.793$\pm$0.012  \\
\cline{1-2}    Beaini et al. & DGN   & \cellcolor[rgb]{ .388,  .745,  .482}0.219$\pm$0.010 &  \cellcolor[rgb]{ .388,  .745,  .482}0.727$\pm$0.005 &  & \cellcolor[rgb]{ .612,  .812,  .498}0.797$\pm$0.009  \\
\cline{1-2}    Ours  & tGNN  & \cellcolor[rgb]{ .643,  .816,  .494}0.301$\pm$0.008 & \cellcolor[rgb]{ .631,  .816,  .498}0.684$\pm$0.006  & \cellcolor[rgb]{ .624,  .816,  .498}0.965$\pm$0.002 &  \cellcolor[rgb]{ .388,  .745,  .482}0.799$\pm$0.016  \\
    \hline
    \end{tabular}
    \label{tab:graph_classification}
    }
    \end{minipage}%
    \begin{minipage}{.5\linewidth}
    \caption{Results of the tGNN ablation study on two node- and one graph-level tasks. tGNN in comparison with high-order CP pooling and low-order linear sum pooling.}
        \resizebox{1\textwidth}{!}{
    \begin{tabular}{|cc|ccc|}
    \hline
    \rowcolor[rgb]{ 0,  0,  0} \multicolumn{2}{|c}{\textcolor[rgb]{ 1,  1,  1}{DATASET}} & \cellcolor[rgb]{ .502,  .392,  .635}\textcolor[rgb]{ 1,  1,  1}{\textit{Cora}} & \cellcolor[rgb]{ .592,  .278,  .024}\textcolor[rgb]{ 1,  1,  1}{\textit{Pubmed}} & \cellcolor[rgb]{ 1,  0,  0}\textcolor[rgb]{ 1,  1,  1}{\textit{ZINC}} \\
    \rowcolor[rgb]{ 0,  0,  0} \multicolumn{2}{|c}{\textcolor[rgb]{ 1,  1,  1}{MODEL}} & \textcolor[rgb]{ 1,  1,  1}{Acc} & \textcolor[rgb]{ 1,  1,  1}{Acc} & \textcolor[rgb]{ 1,  1,  1}{MAE} \\
\cline{1-2}    \multicolumn{2}{|c|}{{Non-Linear} CP {Pooling}} & \cellcolor[rgb]{ .878,  .902,  .588}0.8655$\pm$0.0375 & \cellcolor[rgb]{ .816,  .882,  .573}0.8679$\pm$0.0103 & \cellcolor[rgb]{ .851,  .89,  .58}0.407$\pm$0.025  \\
\cline{1-2}    \multicolumn{2}{|c|}{{Linear Sum Pooling}} & \cellcolor[rgb]{ 1,  .937,  .612}0.8623$\pm$0.0107 & \cellcolor[rgb]{ 1,  .937,  .612}0.8531$\pm$0.0009 & \cellcolor[rgb]{ 1,  .937,  .612}0.440$\pm$0.010  \\
\cline{1-2}    \multicolumn{2}{|c|}{{Non-Linear} CP + {Linear Sum}} & \cellcolor[rgb]{ .388,  .745,  .482}0.8780$\pm$0.0158 & \cellcolor[rgb]{ .388,  .745,  .482}0.9018$\pm$0.0015 & \cellcolor[rgb]{ .388,  .745,  .482}0.301$\pm$0.008  \\
    \hline
    \end{tabular}
    \label{tab:ablationstudy}
    }
    \end{minipage} 
\end{table}

\subsection{Ablation and Efficiency Study}
\label{sec:ablation}

In the ablation study, we first investigate the effectiveness of having the high-order non-linear CP pooling and adding the linear low-order transformation in Tab.~\ref{tab:ablationstudy}, then investigate the relations of the model performance, efficiency, and tensor decomposition rank in Fig.~\ref{fig:ablation_figures}. Moreover, we compare tGNN with different GNN architectures and aggregation functions to show the efficiency in Tab.~\ref{tab:ablation_table1}, \ref{tab:ablation_table2} by showing the number of model parameters, computation time, and accuracy.

In Tab.~\ref{tab:ablationstudy}, we test each component, high-order non-linear CP pooling, low-order linear sum pooling, and two pooling techniques combined, separately. We fix 2-layer GNNs with 32 hidden unit and 64 decomposition rank, and run experiments 10 times on \textit{Cora} and \textit{Pubmed} with 60\%/20\%/20\% random splits used in~\cite{chien2021adaptive}, run \textit{ZINC} 5 times with 10,000/1,000/1,000 graph split used in~\cite{dwivedi2020benchmarking}.

From the results, we can see that adding the linear low-order interactions helps put essential weights on them. Ablation results show that high-order CP pooling has the advantage over low-order linear pooling for generating expressive node and graph representations, moreover, tGNN is more expressive with the combination of high-order pooling and low-order aggregation. This illustrates the necessity of learning high-order components and low-order interactions simultaneously in tGNN.

\paragraph{Computational Aspect}
\begin{figure}[htbp!]
    \begin{minipage}{.5\linewidth}
      \centering
{
\includegraphics[width=1\textwidth]{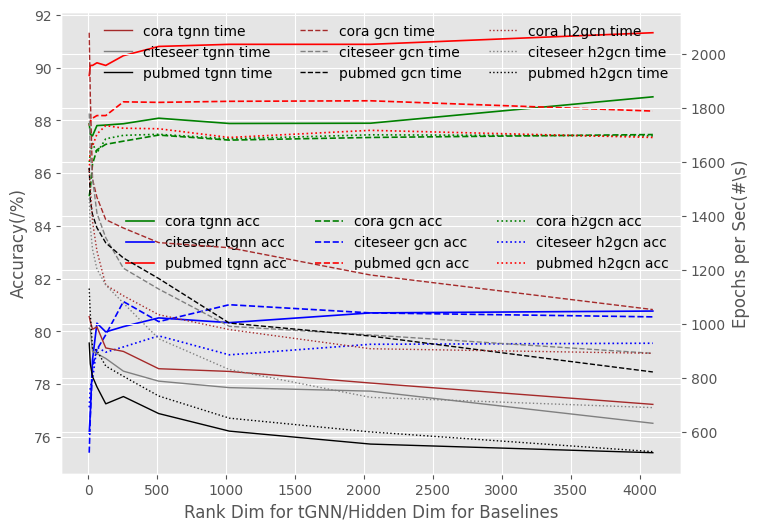}}
    \end{minipage}%
    \begin{minipage}{.5\linewidth}
      \centering
{
\includegraphics[width=1\textwidth]{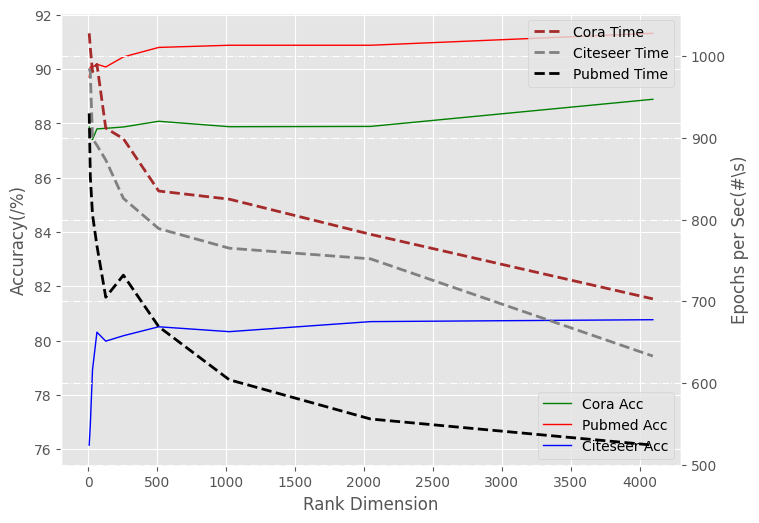}}
    \end{minipage} 
    \caption{Results of node classification with increasing rank dimension on three citation datasets. \textbf{Left Figure}: Left axis shows accuracy, right axis shows \#training epochs second, and horizontal axis indicates rank dim of tGNN or hidden dim of baselines. \textbf{Right Figure}: Left axis shows accuracy, right axis shows \#training epochs second, and horizontal axis indicates rank dim of tGNN.}
    \label{fig:ablation_figures}
\end{figure}
In Fig.~\ref{fig:ablation_figures}, we compare model performance with computation costs. We fix a 2-layer tGNN with 32 hidden channels with 8,16,...,2048,4096 rank, and 2-layer baselines with 8,16,...,2048,4096 hidden dim. We run 10 times on each citation network with the same 60\%/20\%/20\% random splits for train/validation/test and draw the relations of average test accuracy, rank/hidden dim, and computation costs. From the figure, we can see that the model performance can be improved with higher ranks (i.e., Tensor $\tensor{T}$ is more accurately computed as the rank $R$ gets larger), but training time is also increased, thus it is a trade-off between classification accuracy and computation efficiency. And in comparison with baselines, tGNN still has marginal improvements with higher ranks while the baselines stop improving with larger hidden dimensions.

\begin{table}[htbp!]
  \centering
  \caption{Study of Efficiency vs. Performance on Cora. tGNN in comparison with GNN architectures.}
    \centering
        \resizebox{1\textwidth}{!}{
    \renewcommand{\arraystretch}{1}
    \begin{tabular}{|c|ccccccccccccc|}
    \hline
          & Dropout & LR    & Weight Decay & Hidden & Rank  & Head  & Sampling & \#Params & Time(s) & Epoch & Epoch/s & Acc   & Std   \\
    \hline
    tGNN  & 0     & 0.005 & 5.00E-05 & 32    & 8     & \_    & 3     & \cellcolor[rgb]{ .522,  .78,  .486}58128 & 290.9774 & 1389  & \cellcolor[rgb]{ .765,  .855,  .506}4.7736 & \cellcolor[rgb]{ .941,  .906,  .518}85.55 & 1.33  \\
    tGNN  & 0     & 0.005 & 5.00E-05 & 32    & 32    & \_    & 3     & \cellcolor[rgb]{ .929,  .898,  .51}94272 & 790.0373 & 1321  & \cellcolor[rgb]{ .984,  .639,  .463}1.6721 & \cellcolor[rgb]{ .753,  .851,  .506}86.25 & 0.58 \\
    tGNN  & 0     & 0.005 & 5.00E-05 & 32    & 64    & \_    & 3     & \cellcolor[rgb]{ 1,  .918,  .518}142464 & 383.2255 & 1343  & \cellcolor[rgb]{ .969,  .914,  .518}3.5045 & \cellcolor[rgb]{ .804,  .867,  .51}86.06 & 1.08 \\
    tGNN  & 0     & 0.005 & 5.00E-05 & 32    & 128   & \_    & 3     & \cellcolor[rgb]{ 1,  .906,  .518}238848 & 999.1514 & 1247  & \cellcolor[rgb]{ .98,  .569,  .447}1.2481 & \cellcolor[rgb]{ .62,  .812,  .498}86.76 & 1.19 \\
    tGNN  & 0     & 0.005 & 5.00E-05 & 32    & 256   & \_    & 3     & \cellcolor[rgb]{ 1,  .882,  .51}431616 & 1193.7131 & 1272  & \cellcolor[rgb]{ .976,  .537,  .443}1.0656 & \cellcolor[rgb]{ .565,  .796,  .494}86.97 & 1.24 \\
    tGNN  & 0     & 0.005 & 5.00E-05 & 32    & 512   & \_    & 3     & \cellcolor[rgb]{ .996,  .831,  .502}817152 & 1621.9083 & 1332  & \cellcolor[rgb]{ .976,  .494,  .435}0.8213 & \cellcolor[rgb]{ .467,  .769,  .49}87.33 & 1.83 \\
    tGNN  & 0     & 0.005 & 5.00E-05 & 32    & 1024  & \_    & 3     & \cellcolor[rgb]{ .992,  .733,  .482}1588224 & 2377.5139 & 1265  & \cellcolor[rgb]{ .973,  .443,  .424}0.5321 & \cellcolor[rgb]{ .388,  .745,  .482}87.62 & 1.63 \\
    GCN   & 0     & 0.005 & 5.00E-05 & 32    & \_    & \_    & 3     & \cellcolor[rgb]{ .388,  .745,  .482}46080 & 212.3576 & 1509  & \cellcolor[rgb]{ .388,  .745,  .482}7.1059 & \cellcolor[rgb]{ .984,  .675,  .471}84.29 & 1.02 \\
    GCN   & 0     & 0.005 & 5.00E-05 & 32    & \_    & \_    & Full  & \cellcolor[rgb]{ .388,  .745,  .482}46080 & 205.0601 & 1276  & \cellcolor[rgb]{ .533,  .788,  .494}6.2226 & \cellcolor[rgb]{ .996,  .902,  .514}85.24 & 1.69 \\
    GCN   & 0     & 0.005 & 5.00E-05 & 64    & \_    & \_    & 3     & \cellcolor[rgb]{ .906,  .894,  .51}92160 & 316.6461 & 1240  & \cellcolor[rgb]{ .902,  .894,  .514}3.916 & \cellcolor[rgb]{ .996,  .871,  .506}85.12 & 2.11 \\
    GCN   & 0     & 0.005 & 5.00E-05 & 64    & \_    & \_    & Full  & \cellcolor[rgb]{ .906,  .894,  .51}92160 & 318.3962 & 1161  & \cellcolor[rgb]{ .945,  .906,  .518}3.6464 & \cellcolor[rgb]{ .929,  .902,  .514}85.59 & 2.03 \\
    GAT   & 0     & 0.005 & 5.00E-05 & 32    & \_    & 1     & 3     & \cellcolor[rgb]{ .906,  .894,  .51}92238 & 605.3269 & 1998  & \cellcolor[rgb]{ 1,  .922,  .518}3.3007 & \cellcolor[rgb]{ .976,  .525,  .439}83.66 & 1.54 \\
    GAT   & 0     & 0.005 & 5.00E-05 & 32    & \_    & 1     & Full  & \cellcolor[rgb]{ .906,  .894,  .51}92238 & 548.7319 & 1638  & \cellcolor[rgb]{ .996,  .867,  .506}2.9851 & \cellcolor[rgb]{ .992,  .792,  .49}84.79 & 2.26 \\
    GAT   & 0     & 0.005 & 5.00E-05 & 32    & \_    & 8     & 3     & \cellcolor[rgb]{ .996,  .839,  .502}762992 & 1491.5643 & 1594  & \cellcolor[rgb]{ .976,  .537,  .443}1.069 & \cellcolor[rgb]{ .753,  .851,  .506}86.26 & 1.35 \\
    GAT   & 0     & 0.005 & 5.00E-05 & 32    & \_    & 8     & Full  & \cellcolor[rgb]{ .996,  .839,  .502}762992 & 1524.4348 & 1276  & \cellcolor[rgb]{ .976,  .498,  .435}0.837 & \cellcolor[rgb]{ .537,  .788,  .494}87.07 & 1.64 \\
    GAT   & 0     & 0.005 & 5.00E-05 & 64    & \_    & 1     & 3     & \cellcolor[rgb]{ 1,  .914,  .518}184462 & 626.4011 & 1740  & \cellcolor[rgb]{ .992,  .831,  .498}2.7778 & \cellcolor[rgb]{ .984,  .643,  .463}84.15 & 1.29 \\
    GAT   & 0     & 0.005 & 5.00E-05 & 64    & \_    & 1     & Full  & \cellcolor[rgb]{ 1,  .914,  .518}184462 & 682.7776 & 1465  & \cellcolor[rgb]{ .988,  .722,  .478}2.1456 & \cellcolor[rgb]{ .804,  .867,  .51}86.07 & 2.55 \\
    GAT   & 0     & 0.005 & 5.00E-05 & 64    & \_    & 8     & 3     & \cellcolor[rgb]{ .992,  .741,  .486}1525872 & 2164.689 & 1348  & \cellcolor[rgb]{ .973,  .459,  .427}0.6227 & \cellcolor[rgb]{ 1,  .922,  .518}85.32 & 1.31 \\
    GAT   & 0     & 0.005 & 5.00E-05 & 64    & \_    & 8     & Full  & \cellcolor[rgb]{ .992,  .741,  .486}1525872 & 2143.036 & 1105  & \cellcolor[rgb]{ .973,  .443,  .424}0.5156 & \cellcolor[rgb]{ .553,  .792,  .494}87.01 & 0.96 \\
    GCN2  & 0     & 0.005 & 5.00E-05 & 32    & \_    & \_    & 3     & \cellcolor[rgb]{ .408,  .749,  .482}48128 & 256.9575 & 1708  & \cellcolor[rgb]{ .463,  .769,  .49}6.647 & \cellcolor[rgb]{ .973,  .412,  .42}83.17 & 1.5 \\
    GCN2  & 0     & 0.005 & 5.00E-05 & 32    & \_    & \_    & Full  & \cellcolor[rgb]{ .408,  .749,  .482}48128 & 210.7702 & 1302  & \cellcolor[rgb]{ .541,  .788,  .494}6.1773 & \cellcolor[rgb]{ .984,  .698,  .475}84.38 & 2.03 \\
    GCN2  & 0     & 0.005 & 5.00E-05 & 64    & \_    & \_    & 3     & \cellcolor[rgb]{ 1,  .922,  .518}100352 & 353.0055 & 1581  & \cellcolor[rgb]{ .812,  .871,  .51}4.4787 & \cellcolor[rgb]{ .988,  .773,  .486}84.7 & 1.13 \\
    GCN2  & 0     & 0.005 & 5.00E-05 & 64    & \_    & \_    & Full  & \cellcolor[rgb]{ 1,  .922,  .518}100352 & 307.7913 & 1219  & \cellcolor[rgb]{ .894,  .894,  .514}3.9605 & \cellcolor[rgb]{ .992,  .792,  .49}84.79 & 1.64 \\
    GCN2  & 0     & 0.005 & 5.00E-05 & Input Dim & \_    & \_    & 3     & \cellcolor[rgb]{ .973,  .412,  .42}4117009 & 3013.6082 & 1051  & \cellcolor[rgb]{ .973,  .412,  .42}0.3488 & \cellcolor[rgb]{ .631,  .816,  .498}86.72 & 1.82 \\
    GCN2  & 0     & 0.005 & 5.00E-05 & Input Dim & \_    & \_    & Full  & \cellcolor[rgb]{ .973,  .412,  .42}4117009 & 3091.4392 & 1013  & \cellcolor[rgb]{ .973,  .412,  .42}0.3277 & \cellcolor[rgb]{ .412,  .753,  .486}87.54 & 1.66\\
    \hline
    \end{tabular}%
    }
  \label{tab:ablation_table1}%
\end{table}%

\begin{table}[htbp!]
  \centering
  \caption{Study of Efficiency vs. Performance on Cora. CP pooling in comparison with classical pooling techniques.}
        \resizebox{1\textwidth}{!}{
    \renewcommand{\arraystretch}{1}
    \begin{tabular}{|c|cccccccccccc|}
    \hline
          & Dropout & LR    & Weight Decay & Hidden & Rank  & Sampling & \#Params & Time(s) & Epoch & Epoch/s & Acc   & Std    \\
    \hline
    tGNN  & 0     & 0.005 & 5.00E-05 & 32    & 8     & 3     & \cellcolor[rgb]{ 1,  .922,  .518}58128 & 290.9774 & 1389  & \cellcolor[rgb]{ 1,  .922,  .518}4.7736 & \cellcolor[rgb]{ .388,  .745,  .482}85.55 & 1.33\\
    Mean  & 0     & 0.005 & 5.00E-05 & 32    & \_    & Full  & \cellcolor[rgb]{ .388,  .745,  .482}46080 & 449.2177 & 2352  & \cellcolor[rgb]{ .388,  .745,  .482}5.2358 & \cellcolor[rgb]{ .973,  .412,  .42}83.26 & 1.06 \\
    Mean  & 0     & 0.005 & 5.00E-05 & 64    & \_    & Full  & \cellcolor[rgb]{ .973,  .412,  .42}92160 & 398.3229 & 1496  & \cellcolor[rgb]{ .973,  .412,  .42}3.7747 & \cellcolor[rgb]{ .937,  .906,  .518}83.88 & 1.77 \\
    Max   & 0     & 0.005 & 5.00E-05 & 32    & \_    & Full  & \cellcolor[rgb]{ .388,  .745,  .482}46080 & 464.0722 & 2371  & \cellcolor[rgb]{ .616,  .812,  .498}5.0655 & \cellcolor[rgb]{ .976,  .494,  .435}83.33 & 1.96 \\
    Max   & 0     & 0.005 & 5.00E-05 & 64    & \_    & Full  & \cellcolor[rgb]{ .973,  .412,  .42}92160 & 394.1546 & 1496  & \cellcolor[rgb]{ .973,  .42,  .42}3.7955 & \cellcolor[rgb]{ 1,  .922,  .518}83.68 & 2.13\\
    \hline
    \end{tabular}%
    }
  \label{tab:ablation_table2}%
\end{table}%

\paragraph{Efficiency Study}
In Sec.~\ref{sec:time_complexity}, we theoretically discuss the time complexity of tGNN. Here, we experimentally assess model efficiency by comparing
tGNN with GCN~\cite{kipf2016classification}, GAT~\cite{velivckovic2017attention}, GCN2~\cite{chen2020simple}, and mean, max poolings on \textit{Cora} on a CPU over 10 runtimes, and compare the number of model parameters, training epochs per second, and accuracy.  The experiments show that tGNN is more competitive in terms of running time and better accuracy with a fixed number of parameters and the same time budget.

In Tab.~\ref{tab:ablation_table1}, we study efficiency by comparing
tGNN with models on Cora on a CPU over 10 runtimes,
and compare the number of model paramterts, number of training epochs per second, and accuracy. Sampling
means we sample ’3’ neighbors for each node or we use
’Full’ neighborhood. Notice that the average node degree of Cora is 3.9, which means if a node has a number of neighbors less than 10, some of its neighbor nodes will get resampled until it hits 10. We can see that the model performance is not heavily affected by the number of sampled neighbor nodes because the average node degree is not that high on the majority of network datasets, and '3' or '5' would be sufficient.

In Tab.~\ref{tab:ablation_table2}, we study and compare CP pooling with classical pooling methods on Cora on a CPU over 10 runtimes,
and compare the number of model paramterts, number of training epochs per second, and accuracy. Sampling
means we sample ’3’ neighbors for each node or we use
’Full’ neighborhood. 

\section{Discussion}
\subsection{Rank and Performance}
Generally speaking, the low-rank decomposition methods will sacrifice some expressivity for some computation costs, so as we increase the rank, we will re-gain some expressivity but lose some computation costs. As talked about in our paper, it is a trade-off between expressivity (model performance) and computation costs. This idea can be found in previous works \cite{peng2012rasl, rabusseau2016low, cao2017tensor} where they test on computer vision datasets (with visualization of recovering images from noises, a high-rank model can always better recover an image than a low-rank model does).

In this work, we use the low-rank decomposition technique which can potentially reduce the number of learnable parameters by a lower number of ranks. The low-rank method saves the computation costs but the use of low-rank decomposition for learnable parameters will sacrifice some expressivity in computation. If we increase the rank, the learnable parameters can be principally recovered, in this case, we no longer sacrifice its expressivity but computation costs (as the Figures shown in the ablation study).

In computer vision, low-rank decomposition methods were introduced to replace a fully-connected layer. And in their papers \cite{peng2012rasl, rabusseau2016low, cao2017tensor} (in both theory and experiments), they show that a model is better with higher ranks because higher ranks can lead to higher expressivity than low ranks do. They conclude and show that the model performance is positively correlated to the tensor decomposition rank. In \cite{peng2012rasl, rabusseau2016low, cao2017tensor}, they experimentally show that higher ranks can result in lower test errors as the model becomes more expressive (plotting a curve of decomposition rank or model parameter vs. error). Moreover, in \cite{peng2012rasl, rabusseau2016low}, they visualize the relationship between the model performance and model rank. They aim to recover the original pictures from noises with different tensor ranks (from low to high), and the recovered pictures are more clear with high ranks than with low ranks. To conclude our findings in the ablation study and previous works on low-rank decomposition, high-rank models can usually have better performance and stronger expressivity than low-rank models, and higher ranks can usually lead to better performance (in terms of high accuracy or low error) than low ranks do.

\subsection{Overview}
\label{sec:futurework}
In this chapter, we introduce our tGNN work \cite{hua2022high}. We theoretically develop a high-order permutation-invariant multilinear map for node aggregation and graph pooling via tensor parameterization. We show its powerful ability to compute any permutation-invariant multilinear polynomial including sum and mean pooling functions. Experiments demonstrate that tGNN is more effective and accurate on 5 out of 10 datasets, showcasing the  relevance of tensor methods for high-order graph neural network models.

In particular, the success of tGNN in tasks involving chemical graphs and biological networks underscores its effectiveness in capturing complex relationships within structured data. This chapter has offered valuable insights into GNN design, especially for graph-level classification tasks, with a focus on leveraging high-order pooling functions. The experimental section of this chapter demonstrates that GNNs excel at capturing intricate relationships between atoms and bonds in a molecule, showcasing superior performance in predicting properties based on chemical and geometric information. Moreover, current GNN designs have evolved to address 3D molecular graphs, incorporating physics properties such as group equivariance to handle geometric transformations like rotation and translation. Equivariant GNNs, a subset exhibiting equivariance with respect to symmetries, are particularly adept at learning representations invariant to specific transformations, making them suitable for tasks involving geometric structure \cite{satorras2021n, satorras2021n2, tholke2022torchmd}. The integration of graph models and equivariance in graph-structured data has significantly advanced the design and discovery of molecules. Building on these foundational concepts and applications, the subsequent chapter shifts focus to the practical use of GNNs in addressing real-world challenges, specifically in the generation of valid and stable chemical compounds.
\chapter{Graph Neural Networks for \\ Molecule Generation}
\label{sec:mol.gen}
After covering fundamental concepts about GNNs in previous chapters, this section introduces GNNs that incorporate physical properties, specifically dealing with molecular graphs in the 3D Euclidean space. GNNs, equipped with group equivariance, are effective in learning equivariant and invariant features for 3D objects, such as molecules. In this chapter, we propose a novel GNN model, named MUformer, as the backbone of our generative model, MUDiff. MUDiff and MUformer can learn both invariant and equivariant aspects of a molecule, remaining robust to geometric transformations. The primary innovation of our denoising transformer model lies in designing an attention mechanism that facilitates interaction between 2D and 3D structural information. This design enables our transformer to concurrently compute graph connectivity and spatial atom arrangements. Moreover, when computing 3D geometric structures, the denoising transformer model adheres to critical 3D roto-translation equivariance constraints. This compliance ensures insensitivity to geometric transformations on the molecule, allowing the entire diffusion process to adhere to these constraints. Through these novel designs, our model can generate and learn a comprehensive molecular representation that captures both 2D and 3D structures, addressing the aforementioned limitations.

\section{Preliminaries}
We introduce notations that are particularly used for this chapter, providing an overview of generative diffusion models and graph representations tailored for molecular data, encompassing both 2D and 3D aspects.

\subsection{Diffusion Models}
\label{sec:diffusion.processes}
Diffusion models consist of a noising model and a denoising network. The noising model $q$ adds noise to a data point $\mathbf{X}$ to generate a sequence of noisy points $\{\mathbf{\Tilde{X}}_t\}^{T}_{t=0}$. This process follows the Markov property, $q(\mathbf{\Tilde{X}}_0,\ldots,\mathbf{\Tilde{X}}_T | \mathbf{X})=q(\mathbf{\Tilde{X}}_0|\mathbf{X})\prod_{t=1}^T q(\mathbf{\Tilde{X}}_t|\mathbf{\Tilde{X}}_{t-1})$.
The denoising network $\psi_{\theta}$ aims to reverse the noising process: given a noisy point $\mathbf{\Tilde{X}}_t$, it predicts a clean estimate $\mathbf{\hat{X}}=\psi_{\theta}(\mathbf{\Tilde{X}}_t, t)$ of $\mathbf{X}$.

\paragraph{Continuous Data}
The \textit{noising process} in diffusion models for continuous data point $\mathbf{X}$ can be represented by a multivariate normal distribution,
$
    q(\mathbf{\Tilde{X}}_t|\mathbf{X}) = \mathcal{N}(\mathbf{\Tilde{X}}_t|\alpha_t\mathbf{X}, \sigma^2_t\mathbf{I}),
$
where $\alpha_t\in \mathbb{R}^+$ controls the amount of signal retained and $\sigma_t\in \mathbb{R}^+$ represents the amount of Gaussian noise added, and $\alpha_t$ smoothly transitions from $ 1$ to $0$. Following \cite{ho2020denoising}, we choose $\alpha_t = \sqrt{1-\sigma^2_t}$ in order to obtain a variance preserving process, and the signal-to-noise ratio SNR$(t) = {\alpha^2_t}/{\sigma^2_t}$ is defined by \cite{kingma2021variational}.
For every two time steps $t, t-1$, the noising process is
$
q(\mathbf{\Tilde{X}}_t|\mathbf{\Tilde{X}}_{t-1}) = \mathcal{N}(\mathbf{\Tilde{X}}_t|\alpha_{t|t-1}\mathbf{\Tilde{X}}_{t-1}, \sigma^2_{t|t-1}\mathbf{I}),
$
where $\alpha_{t|t-1}=\alpha_t/\alpha_{t-1}$ and $\sigma^2_{t|t-1} = \sigma^2_t - \alpha_{t|t-1}^2\sigma^2_{t-1}$.

The posterior of the transitions gives the \textit{denoising process},
\begin{equation}
    q(\mathbf{\hat{X}}_{t-1}|\mathbf{\hat{X}}_{t}, \mathbf{{X}}) = \mathcal{N}(\mathbf{\hat{X}}_{t-1}|\bm{\mu}_{t\to t-1}(\mathbf{\Tilde{X}}_t, \mathbf{{X}}), \sigma^2_{t\to t-1}\mathbf{I}),
\end{equation}
where the functions are defined as $\bm{\mu}_{t\to t-1}(\mathbf{\Tilde{X}}_t, \mathbf{{X}})=\frac{\alpha_{t|t-1}\sigma^2_{t-1}}{\sigma^2_t}\mathbf{\Tilde{X}}_t + \frac{\alpha_{t-1}\sigma^2_{t|t-1}}{\sigma^2_t}\mathbf{{X}}, \sigma_{t\to t-1}=\frac{\sigma_{t|t-1}\sigma_{t-1}}{\sigma_t}$. 
In the true denoising process, the clean data $\mathbf{{X}}$ is unknown, so it is replaced by the network approximation $\mathbf{\hat{X}}$ as,
\begin{equation}
\label{eq:posterior.atom}
    p(\mathbf{\hat{X}}_{t-1}|\mathbf{\hat{X}}_{t}) = \mathcal{N}(\mathbf{\hat{X}}_{t-1}|\bm{\mu}_{t\to t-1}(\mathbf{\Tilde{X}}_t, \mathbf{\hat{X}}), \sigma^2_{t\to t-1}\mathbf{I}).
\end{equation}

\paragraph{Discrete Data} 
Discrete objects, such as graph structures, may not be well suited for Gaussian noise models as they can destroy sparsity and connectivity, as suggested by \cite{vignac2022digress}.
Following \cite{austin2021structured, vignac2022digress}, we use a series of transition matrices $\{Q_t\}_{t=0}^T$ to represent noise on one-hot encoded discrete data points, where ${Q}_{{t}_{{ij}}}=q(\mathbf{X}_t=j|\mathbf{X}_{t-1}=i)$ 
is the probability of transitioning from state $i$ to state $j$.
We obtain noisy data by multiplying the clean data point with a transition matrix, 
$
        q(\mathbf{\Tilde{X}}_t | \mathbf{\Tilde{X}}_{t-1}) = \mathbf{\Tilde{X}}_tQ_{t}, \ q(\mathbf{\Tilde{X}}_t | \mathbf{X}) = \mathbf{X}\bar{Q}_{t},
$
where $\bar{Q}_t = Q_tQ_{t-1}\ldots Q_0$.
The posterior distribution is computed using the Bayes rule,
$
    q(\mathbf{\Tilde{X}}_{t-1}| \mathbf{\Tilde{X}}_{t}, \mathbf{{X}}) \propto \mathbf{\Tilde{X}}_{t} Q_t^T \odot \mathbf{X}\bar{Q}_{t-1},
$
where $Q^T$ represents the transpose of the transition matrix $Q$, and $\odot$ denotes the Hadamard product.

\subsection{The Basics of Molecules}
\label{sec:basics}
A molecule is a group of atoms held together by chemical bonds, which can be classified into various types based on the nature of the bond. The structure of a molecule can be visualized and represented in both 2D and 3D forms, with the 2D representation showing the connectivity of the atoms and the 3D representation showing the arrangement of the atoms in space. To completely describe a molecule, we represent it as a tuple $\mathbf{M}=(\mathbf{H, E, X})$, where $\mathbf{H} \in \mathbb{R}^{n \times d}$ denotes the collection of atoms, $n$ is the number of atoms, and $d$ is the feature dimension; $\mathbf{E} \in \mathbb{R}^{n \times n \times b}$ is the 2D graph representation for chemical bonds, the bond type is represented by one-hot encoding, and $b$ is the number of bond(edge) types; $\mathbf{X} \in \mathbb{R}^{n \times 3}$ represents the 3D geometric structure and each row indicates the position of the atom in the Euclidean space. For simplicity, we assume that molecules are fully connected and include \textit{no-bond} as a special edge type.  To account for symmetry, we use symmetric edge representations, i.e., $\mathbf{E}=\mathbf{E}^T$. 

\paragraph{3D Molecule Representation}
For each pair of atoms $(i,j)$ in the 3D space, we process the Euclidean distance between them using an exponential normal radial basis function \cite{schutt2017schnet}, i.e., $f_{\text{RBF}^k}(d_{ij})=\exp{(-\beta_k(\exp{(-d_{ij})-\mu_k})^2)}$, where $k$ is the number of basis kernels, $d_{ij}$ is the distance between atoms $i$ and $j$, $\beta_k$ and $\mu_k$ are fixed parameters determining the function's center, and width, respectively. These parameters are initialized as per \cite{unke2019physnet}.

To smooth out the transition to $0$ as the distance $d_{ij}$ approaches a cutoff distance of $d_\text{cut}=5$\r{A}, we also apply a cosine cutoff function $f_{\cos}(d_{ij})$, i.e., $f_{\cos}(d_{ij})=\frac{1}{2}(\cos{(\frac{\pi d_{ij}}{d_{\text{cut}}})}+1)$ if $d_{ij} \leq d_{\text{cut}}$, and $f_{\cos}(d_{ij})=0$ if $d_{ij} > d_{\text{cut}}$. In the model, we will use both $f_{\cos}(d_{ij})$ and  $f_{\cos}(f_{\text{RBF}^k}(d_{ij}))$.

\section{MUDiff: \underline{M}olecule \underline{U}nified \underline{Diff}usion}
\label{sec:diffusion.process}
Both the continuous and discrete elements of molecules are essential in order to depict a comprehensive molecular representation, however, the existing models \cite{satorras2021n, xu2022geodiff, hoogeboom2022equivariant, vignac2022digress} have only been able to generate a portion of these components.
Our diffusion model is designed to denoise continuous and discrete aspects of a molecule separately. The continuous aspects encompass atom features and 3D coordinates, while the discrete aspects include molecular structure. This separation allows for independent handling of atoms and edges, a similar approach is shown to be successful in image diffusion models \cite{austin2021structured}, but unexplored for molecules. By jointly generating the continuous 3D geometry and discrete 2D graph representation, our model enhances the representation of atom and edge features, yielding a more comprehensive and holistic understanding of the molecule that incorporates both geometric and topological information.

\subsection{Diffusion Process}
Our diffusion model distinctly applies noises to atom and edge representations to enhance the generative process. Specifically, we apply continuous noises to atom representations, encompassing both atom features and coordinates, while introducing discrete noises to edge representations, which correspond to the graph structure. This targeted approach differentiates our diffusion process from previous molecule diffusion models, allowing for a more comprehensive generation of molecules that captures both geometric and topological information.

\paragraph{Atom Features and Coordinates}
As introduced in Sec~\ref{sec:diffusion.processes}, we add Gaussian noise to atom features and coordinates at each time step $t$, with ${\bm{\epsilon}}_{\mathbf{H}}^t\sim \mathcal{N}_{\mathbf{H}}(\mathbf{0, I}) \in \mathbb{R}^{n\times d}, {\bm{\epsilon}}_{\mathbf{X}}^t\sim \mathcal{N}_{\mathbf{X}}(\mathbf{0, I}) \in \mathbb{R}^{n\times 3}$, where $n$ is the number of atoms, $d$ denotes the feature dimension. For the 3D coordinates, we follow \cite{kohler2020equivariant} to use the linear subspace of zero center of gravity for ${\mathcal{N}_{\mathbf{X}}}$ such that $\sum_{i}\mathbf{x}_i=0$. This leads to noisy atom features and coordinates,
\begin{equation}
\label{eq:noisy.atom.features}
    \begin{aligned}
        &\mathbf{\tilde{H}}_t = \alpha_t \mathbf{H} + \sigma_t {\bm{\epsilon}}^t_{\mathbf{H}},\; \mathbf{\tilde{X}}_t = \alpha_t \mathbf{X} + \sigma_t {\bm{\epsilon}}^t_{\mathbf{X}}.
    \end{aligned}
\end{equation}
This method ensures that the perturbations applied to the 3D coordinates do not affect the center of gravity of the molecule, allowing for the denoising process to be invariant with respect to to translations.

\paragraph{Edge Features}
Following Sec~\ref{sec:diffusion.processes}, we transform the discrete clean edge type to obtain noisy ones,
\begin{equation}
    \mathbf{\tilde{E}}_t=\mathbf{E}\bar{Q}_t.
\end{equation}
where the transition matrix $\bar{Q}_t$ is obtained by $\bar{Q}_t = \alpha_t \mathbf{I} + (1-\alpha_t){{1}_b{1}^t_b}/b \in \mathbb{R}^{b\times b}$. We use uniform transitions over the number of edge types $b$ \cite{austin2021structured, vignac2022digress}, resulting in a uniform limit distribution $q_{\infty}$ over edge categories. 
\cite{vignac2022digress} suggests that the limit distribution $q_{\infty}=\lim_{T\to \infty}q(\mathbf{\Tilde{E}}_T|\mathbf{E})$
should be independent of clean data $\mathbf{E}$ for efficient diffusion models. In our diffusion model, the discrete process for noising/denoising discrete graph structures $\mathbf{E}\in \mathbb{R}^{n\times n\times b}$, we use a sequence of transition matrices $\{\bar{Q}_t\}_{t=0}^T$ to add noise to $\mathbf{E}$. In our choice, we follow \cite{austin2021structured, vignac2022digress} to use the simplest uniform transition parameterized by
\begin{equation}
\begin{aligned}
&\bar{Q}_t = \alpha_t \mathbf{I} + (1-\alpha_t)\frac{{1}_b{1}^T_b}{b} \in \mathbb{R}^{b\times b}\\ 
&\mathbf{\Tilde{E}}_{t} = \mathbf{E}\bar{Q}_t \in \mathbb{R}^{n\times n\times b}
\end{aligned}
\end{equation}
with $\alpha_t$ smoothly transition from $1\to 0$ as $t$ goes from $0\to T$. When $t$ gradually goes to $\infty$,
\begin{equation}
\begin{aligned}
\lim_{t\to \infty} \Bar{Q}_t &= \lim_{t\to \infty} \alpha_t \mathbf{I} + (1-\alpha_t)\frac{{1}_b{1}^T_b}{b} \\
&= \lim_{t\to \infty} \alpha_t \mathbf{I} + (1- \lim_{t\to \infty} \alpha_t)\frac{{1}_b{1}^T_b}{b} \\
&= \frac{{1}_b{1}^T_b}{b}.
\end{aligned}
\end{equation}
It suggests that $q(\mathbf{\Tilde{E}}_t|\mathbf{E})$ converges to a uniform distribution as $t\to T$, and the limit distribution $q_{\infty}$ is just a uniform distribution over the edge categories independently of $\mathbf{E}$.
Additionally, since molecules are always undirected graphs, we only apply noise to the upper triangular of the edge representation matrix and then symmetrize the matrix, which ensures that changes made on the edges are consistent across the graph.

\subsection{Denoising Process}
To date, no existing models can simultaneously predict the features of atoms $\mathbf{H}$, their coordinates $\mathbf{X}$, and the structures of molecules $\mathbf{E}$. To address this gap, we introduce a novel denoising network, named MUformer, which learns the denoising process to make predictions for the comprehensive representation of molecules. This model is unique in its ability to consider all aspects of the molecule in a unified manner while ensuring that the denoising process is equivariant, as suggested by \cite{xu2022geodiff}.

\paragraph{Network Estimation}
Instead of directly predicting the atom representations $\mathbf{\hat{H}, \hat{X}}$, the network attempts to predict the Gaussian noises for atom features and coordinates $\mathbf{\hat{{\bm{\epsilon}}}}_{\mathbf{H}}, \mathbf{\hat{{\bm{\epsilon}}}}_{\mathbf{X}}$, as it has been shown to be easier to optimize in \cite{ho2020denoising}. This approach allows the network to differentiate between the noise added by the noising process and the ground-truth representations, $\mathbf{{H}},\mathbf{{X}}$. The network takes as input a noisy molecule, where atom features are concatenated with the normalized time step $\frac{t}{T}$, and predicts the probability of edge features, as well as the estimates of noises for atom features and coordinates,
\begin{equation}
\label{eq:uncond}
\mathbf{\hat{{\bm{\epsilon}}}}^t_{\mathbf{H}}, \mathbf{\hat{{\bm{\epsilon}}}}^t_{\mathbf{X}}, p(\mathbf{\hat{E}})=
\psi_{\theta}([\mathbf{\tilde{H}}_t, \frac{t}{T}], \mathbf{\tilde{X}}_t, \mathbf{\tilde{E}}_t) - (\mathbf{0}, \mathbf{\tilde{X}}_t, \mathbf{0}),
\end{equation}
where the input coordinates are then subtracted from the estimated noise for coordinates to ensure that the outputs lie on the zero center of gravity subspace, as suggested by \cite{hoogeboom2022equivariant}. We subsequently obtain estimates of atom features and coordinates by
\begin{equation}
\label{eq.network.estimate.h.x}
\begin{aligned}
    & \mathbf{\hat{H}} = \frac{1}{\alpha_t} \mathbf{\Tilde{H}}_t - \frac{\sigma_t}{\alpha_t}\mathbf{\hat{{\bm{\epsilon}}}}^t_{\mathbf{H}} ,\; \mathbf{\hat{X}} = \frac{1}{\alpha_t} \mathbf{\Tilde{X}}_t - \frac{\sigma_t}{\alpha_t}\mathbf{\hat{{\bm{\epsilon}}}}^t_{\mathbf{X}}.
\end{aligned}
\end{equation}

\begin{algorithm}[H]
    \caption{Training MUDiff}
    \label{algo:training}
    \footnotesize
    \begin{algorithmic}[1]
        \STATE \textbf{Input:} A complete molecule $\mathbf{M=(H,E,X)}$
        \STATE Sample $t\sim \mathcal{U}(1,\cdots,T)$ 
        \STATE Sample ${\bm{\epsilon}}_\mathbf{H}, {\bm{\epsilon}}_\mathbf{X} \sim \mathcal{N}(\mathbf{0, I})$ 
        \STATE Subtract center of gravity from ${\bm{\epsilon}}_\mathbf{X}$ 
        \STATE Compute $\mathbf{\tilde{H}}_t = \alpha_t \mathbf{H} + \sigma_t {\bm{\epsilon}}^t_{\mathbf{H}}, \ \mathbf{\tilde{X}}_t = \alpha_t \mathbf{X} + \sigma_t {\bm{\epsilon}}^t_{\mathbf{X}}$
        \STATE Sample $\mathbf{\Tilde{E}}_t \sim \mathbf{E}\Tilde{Q}_t$ 
        \STATE Compute $\mathbf{\hat{{\bm{\epsilon}}}}^t_{\mathbf{H}}, \mathbf{\hat{{\bm{\epsilon}}}}^t_{\mathbf{X}}, p(\mathbf{\hat{E}}) =
        \psi_{\theta}([\mathbf{\tilde{H}}_t, \frac{t}{T}], \mathbf{\tilde{X}}_t, \mathbf{\tilde{E}}_t) - (\mathbf{0}, \mathbf{\tilde{X}}_t, \mathbf{0})$ 
        \STATE Minimize $\|\mathbf{{{\bm{\epsilon}}}}^t_{\mathbf{H}}-\mathbf{\hat{{\bm{\epsilon}}}}^t_{\mathbf{H}}\|^2 + \|\mathbf{{{\bm{\epsilon}}}}^t_{\mathbf{X}}-\mathbf{\hat{{\bm{\epsilon}}}}^t_{\mathbf{X}}\|^2+\text{CE}(\mathbf{{E}}, p(\mathbf{\hat{E}}))$
    \end{algorithmic}
\end{algorithm}

\paragraph{Training Objective}
For atom features and coordinates, the objective is to accurately predict the true noise present in the observations of atom features and coordinates.
To achieve this, we follow the approach outlined in \cite{hoogeboom2022equivariant} and minimize the distance between the true noise and the estimates of noise predicted by the network $\psi_\theta$. The objectives for atoms are defined as,
\begin{equation}
\begin{aligned}
\label{eq:training.atom.loss}
    \mathcal{L}^{\mathbf{H}}_{{t}} &= \frac{1}{2} \mathbb{E}_{{\bm{\epsilon}}_{\mathbf{H}}^t \sim N_{\mathbf{H}}(\mathbf{0,I})}\left[ \omega(t) \| {\bm{\epsilon}}_{\mathbf{H}}^t-\mathbf{\hat{{\bm{\epsilon}}}}^t_{\mathbf{H}} \|^2\right], \
    \mathcal{L}^{\mathbf{X}}_{{t}} &= \frac{1}{2} \mathbb{E}_{{\bm{\epsilon}}_{\mathbf{X}}^t \sim N_{\mathbf{X}}(\mathbf{0,I})}\left[\omega(t) \| {\bm{\epsilon}}_{\mathbf{X}}^t-\mathbf{\hat{{\bm{\epsilon}}}}^t_{\mathbf{X}} \|^2\right],
\end{aligned}
\end{equation}
where $\omega(t) = (1 - \text{SNR}(t-1)/\text{SNR}(t))$. To stabilize the training process, we set $\omega(t)=1$ during the training phase, as suggested by \cite{ho2020denoising, hoogeboom2022equivariant}.

To handle edge features, we approach it as a classification problem and minimize the cross-entropy loss for each atom pair $(i,j)\in \mathbf{E}$. The loss is calculated between the actual edge type and the predicted edge probability distribution,
\begin{equation}
\label{eq:training.edge.loss}
    \mathcal{L}_t^{\mathbf{E}} = \mathbb{E}_{(i,j) \sim \mathbf{E}}\left[\mathbf{E}_{ij} \log(p(\mathbf{\hat{E}}_{ij}))\right].
\end{equation}
At every time step $t$, the total loss is computed as the sum of the three losses,
$\mathcal{L}_t =  \mathcal{L}^{\mathbf{H}}_{{t}} +  \mathcal{L}^{\mathbf{E}}_{{t}} + \mathcal{L}^{\mathbf{X}}_{{t}}$.
The entire training process is described in Algorithm~\ref{algo:training}.

\begin{algorithm}[H]
    \centering
    \caption{Sampling from MUDiff}
    \label{algo:sampling}
    \footnotesize
    \begin{algorithmic}[1]
        \STATE Sample $\mathbf{\Tilde{M}}_T$: $\mathbf{\Tilde{H}}_T, \mathbf{\Tilde{X}}_T \sim \mathcal{N}(\mathbf{0, I}), \ \mathbf{\Tilde{E}}_T \sim q_{\infty}$
        \FOR{$t=T,T-1,\ldots,1$}
        \STATE Compute $\mathbf{\hat{{\bm{\epsilon}}}}^t_{\mathbf{H}}, \mathbf{\hat{{\bm{\epsilon}}}}^t_{\mathbf{X}}, \mathbf{\hat{E}} =
        \psi_{\theta}([\mathbf{\tilde{H}}_t, \frac{t}{T}], \mathbf{\tilde{X}}_t, \mathbf{\tilde{E}}_t) - (\mathbf{0}, \mathbf{\tilde{X}}_t, \mathbf{0})$ 
        \STATE Sample $\mathbf{\Tilde{E}}_{t-1} \sim p(\mathbf{\Tilde{E}}_{t-1}|\mathbf{\Tilde{E}}_t)$
        \STATE Sample ${\bm{\epsilon}}_{\mathbf{H}}, {\bm{\epsilon}}_{\mathbf{X}}\sim \mathcal{N}(\mathbf{0, I})$
        \STATE Compute $\mathbf{\Tilde{H}}_{t-1}=\frac{\mathbf{\Tilde{H}}_{t}}{\alpha_{t|t-1}}-\frac{\sigma^2_{t|t-1} \mathbf{\hat{{\bm{\epsilon}}}}^t_{\mathbf{H}}}{\alpha_{t|t-1} \sigma_t} + \sigma_{t\to t-1}\bm{\epsilon}_{\mathbf{H}}$
        \STATE Subtract center of gravity from ${\bm{\epsilon}}_\mathbf{X}$ 
        \STATE Compute $\mathbf{\Tilde{X}}_{t-1}=\frac{\mathbf{\Tilde{X}}_{t}}{\alpha_{t|t-1}}-\frac{\sigma^2_{t|t-1} \mathbf{\hat{{\bm{\epsilon}}}}^t_{\mathbf{X}}}{\alpha_{t|t-1} \sigma_t} + \sigma_{t\to t-1}\bm{\epsilon}_{\mathbf{X}}$
        \ENDFOR
        \STATE Sample $\mathbf{M}\sim p(\mathbf{M}|\mathbf{\Tilde{M}}_0)$
    \end{algorithmic}
\end{algorithm}

\paragraph{Sampling}
Once the model is trained, it can be used to sample new molecules. 
The true sampling process $p(\mathbf{\Tilde{M}}_{t-1}|\mathbf{\Tilde{M}}_{t})$ uses the approximation of a complete molecule $\mathbf{\hat{M} = (\hat{H}, \hat{E}, \hat{X})}$.
The complete molecule is sampled by taking the product of the posterior distributions of atom features, coordinates, and edge features as 
\begin{equation}
    p(\mathbf{\Tilde{M}}_{t-1}|\mathbf{\Tilde{M}}_{t}) = p(\mathbf{\Tilde{H}}_{t-1} | \mathbf{\Tilde{H}}_{t})
    p(\mathbf{\Tilde{E}}_{t-1} | \mathbf{\Tilde{E}}_{t})
    p(\mathbf{\Tilde{X}}_{t-1} | \mathbf{\Tilde{X}}_{t}),
\end{equation}
with the posterior distributions of atom features and coordinates from Eq~\ref{eq:posterior.atom} and the posterior distribution of edges defined in derived in Eq.~\ref{eq:edge.posterior}. 
The sampling process is described in Algorithm~\ref{algo:sampling}.

\paragraph{Posterior Distribution of Edge Features $p(\mathbf{\Tilde{E}}_{t-1}|\mathbf{\Tilde{E}}_t)$}
\label{app:posterior.edge}
For simplicity, we use $\mathbf{M}_t=(\mathbf{H_t,E_t,X_t})$ to denote the noisy molecule $\mathbf{\Tilde{M}}_t=(\mathbf{\Tilde{H}}_t,\mathbf{\Tilde{E}}_t,\mathbf{\Tilde{X}}_t)$ at time $t$.
The posterior distribution of a molecule is calculated by,
\begin{equation}
\begin{aligned}
    p(\mathbf{{M}}_{t-1}|\mathbf{{M}}_{t}) & = p(\mathbf{{H}}_{t-1}, \mathbf{{E}}_{t-1}, \mathbf{{X}}_{t-1} | \mathbf{{H}}_{t}, \mathbf{{E}}_{t}, \mathbf{{X}}_{t}) \quad\quad \textit{$\mathbf{H, E, X}$ are independent} \\ 
    & = p(\mathbf{{H}}_{t-1} | \mathbf{{H}}_{t})
    p(\mathbf{{E}}_{t-1} | \mathbf{{E}}_{t})
    p(\mathbf{{X}}_{t-1} | \mathbf{{X}}_{t}).
\end{aligned}
\end{equation}
Posterior distributions of atom features and coordinates are simple to compute as they are derived from normal distributions for continuous data (see Eq~\ref{eq:posterior.atom}). Here, we compute the posterior distribution for edge features,
\begin{equation}
\label{eq:edge.posterior}
\begin{aligned}
    p(\mathbf{{E}}_{t-1} | \mathbf{{E}}_{t}) & = \prod_{(i,j)\in \mathbf{E}} p(\mathbf{{e}}_{{t-1}_{ij}} | \mathbf{{e}}_{{t}_{ij}}) \\
    p(\mathbf{{e}}_{{t-1}_{ij}} | \mathbf{{e}}_{{t}_{ij}}) 
    & = \sum_{\mathbf{\hat{e}}_{ij} \in \mathbf{\hat{E}}} p(\mathbf{{e}}_{{t-1}_{ij}} | \mathbf{{e}}_{{t}_{ij}}, \mathbf{\hat{e}}_{ij}) p(\mathbf{\hat{e}}_{ij} | \mathbf{{e}}_{{t}_{ij}}) \\
    & = \sum_{\mathbf{\hat{e}}_{ij} \in \mathbf{\hat{E}}} p(\mathbf{{e}}_{{t-1}_{ij}} | \mathbf{{e}}_{{t}_{ij}}, \mathbf{\hat{e}}_{ij}) p(\mathbf{\hat{e}}_{ij}),
\end{aligned}
\end{equation}
where we choose
\[
    p(\mathbf{{e}}_{{t-1}_{ij}}| \mathbf{{e}}_{{t}_{ij}}, \mathbf{\hat{e}}_{ij}) = 
\begin{dcases}
     q(\mathbf{{e}}_{{t-1}_{ij}} | \mathbf{{e}}_{{t}_{ij}}, \mathbf{\hat{e}}_{ij}),  & \text{if } q(\mathbf{{e}}_{{t}_{ij}} | \mathbf{\hat{e}}_{ij}) > 0\\
    0,              & \text{otherwise.}
\end{dcases}
\]
The posterior distribution for discrete objects is given in Sec~\ref{sec:diffusion.processes}, as $q(\mathbf{{E}}_{t-1}| \mathbf{{E}}_{t}, \mathbf{{E}})$, but since the clean edge features $\mathbf{E}$ are unknown during sampling, we substitute it with the network approximation $\mathbf{\hat{E}}$, resulting in the posterior distribution $q(\mathbf{{E}}_{t-1}| \mathbf{{E}}_{t}, \mathbf{\hat{E}})$.

\section{MUformer: \underline{M}olecule \underline{U}nified Trans\underline{former}}
\label{sec:graph.transformer}
\begin{figure}[ht!]
\centering
{
\includegraphics[width=1.\textwidth]{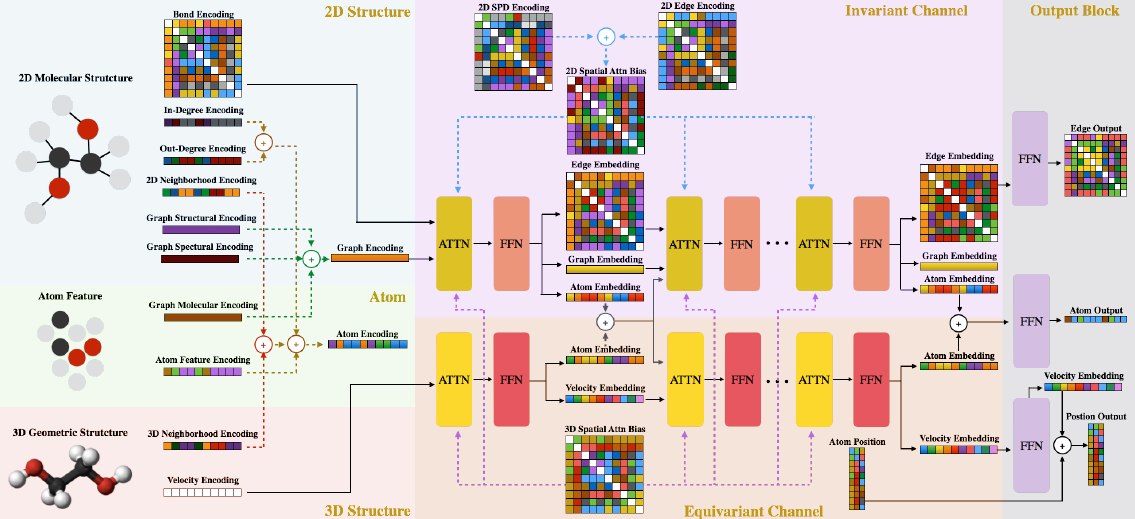}}
{%
  \caption{The figure showcases our MUformer for processing 2D and 3D molecular data. Within the Transformer backbone, two channels exist: purple for 2D data and brown for 3D data. The blue part encodes 2D molecular structures, while the green part handles atom-level information and the red part processes 3D geometric structures. With missing 2D or 3D structures, the model activates either the invariant (purple) or equivariant (brown) channel. The invariant channel predicts atom and edge features, while the equivariant channel offers geometric transformation robustness and predicts atom features and positions. When both channels are operational, the model maintains robustness to geometric transformations and predicts a complete molecule, and final atom features are derived by merging outputs from both channels and feeding the combined data through an output network.}%
  \label{fig:mudiff}
}
\end{figure}
To learn the complete molecular representation, in this section, we propose a novel equivariant graph transformer MUformer, denoted by $\psi_{\theta}$ (visualized in Fig~\ref{fig:mudiff}), which contains 6 encoding functions (Sec~\ref{sec:muformer.encoding}), 4 attention biases (Sec~\ref{sec:muformer.attention.bias}), and 2 attention channels (Sec~\ref{sec:transformer_channels}). It takes as input a complete molecule $\mathbf{M} = (\mathbf{H, E, X})$ with $\mathbf{H} \in \mathbb{R}^{n \times d}, \mathbf{E} \in \mathbb{R}^{n \times n \times b}, \mathbf{X} \in \mathbb{R}^{n \times 3}$, and outputs the predicted molecule $\mathbf{\hat{M}} = (\mathbf{\hat{H}, \hat{E}, \hat{X}})$. For clarity, we refer to the 2D molecular structure as $\mathbf{M}^{\text{2D}}=(\mathbf{H}, \mathbf{E})$, and the 3D geometric structure as $\mathbf{M}^{\text{3D}}=(\mathbf{H}, \mathbf{X})$. In the following subsections, we will introduce each component of our MUformer in details.

Our MUformer architecture can be used under different input conditions. When only 2D molecular information is available, only the invariant channel is activated, and the model makes predictions for atom and edge features only. When only 3D molecular information is available, only the equivariant channel is activated, and the model makes predictions for atom features and coordinates only. When both 2D and 3D molecular data are provided, the invariant and equivariant channels are activated, and the model can make predictions for the complete molecule, including atom features, molecular structure, and geometric structure.

\subsection{Encodings}
\label{sec:muformer.encoding}

The MUformer consists of 6 encoding functions, with 3 being message-passing based, designed to incorporate atomic, positional, and structural information into a concise and expressive representation, particularly suited for handling graph-structured inputs.

\noindent \textbf{1. Atom Encoding} \quad
Incorporating centrality information into the atom representations is crucial as it helps to highlight the importance of individual atoms in the molecular structure. The authors in \cite{ying2021transformers} propose a method of utilizing in-degree $\text{deg}^-$ and out-degree $\text{deg}^+$ obtained from 2D molecular graphs $\mathbf{M}^{\text{2D}}$ to incorporate centrality information into the atom-wise encoding process. This allows for a detailed and accurate representation of the molecular structure, taking into account the relative importance of each atom. After the centrality encoding, the new representation $\mathbf{z}^{\text{1D}}_{\mathbf{h}_i}$ for node $i$ is,
\begin{equation}
    \mathbf{z}^{\text{1D}}_{\mathbf{h}_i} = W_{\text{atom}_1}{\mathbf{h}_i} + W_{\text{in-deg}_1}\text{deg}_i^- + W_{\text{out-deg}_1}\text{deg}_i^+,
\end{equation}
where $W_{\text{atom}_1} \in \mathbb{R}^{d \times fh_1}, W_{\text{in-deg}_1}  \in \mathbb{R}^{1 \times fh_1}, W_{\text{out-deg}_1}  \in \mathbb{R}^{1 \times fh_1}$ are designated learnable parameters for the atom features, in-degree $\text{deg}^-$, and out-degree $\text{deg}^+$. The resulting atom embedding for atom $i$, $\mathbf{z}^{\text{1D}}_{\mathbf{h}_i} \in \mathbb{R}^{fh_1}$, includes the degree centrality information.

\noindent \textbf{2. Bond Encoding} \quad
To obtain a richer edge representation, we incorporate the pair-wise atom information into the edge encoding with message-passing information. For each edge $\mathbf{e}_{ij} \in \mathbf{E}$, we use a \textit{permutation-invariant} function to generate the embedded edge representations, ensuring consistency in the learned representation regardless of the order of the atoms,
\begin{equation}
    \mathbf{z}_{\mathbf{e}_{ij}} = W_{\text{comb}_1}([W_{\text{atom}_2}\mathbf{h}_{i} + W_{\text{edge}_1}\mathbf{e}_{ij} + W_{\text{atom}_2}\mathbf{h}_{j}]) + b_{\text{comb}_1}.
\end{equation}
where $W_{\text{atom}_2} \in \mathbb{R}^{d \times fe_{\text{in}}}$ and $W_{\text{edge}_1} \in \mathbb{R}^{b \times fe_{\text{in}}}$ are learnable parameters that handle the atom features and edge types respectively, $W_{\text{comb}_1} \in \mathbb{R}^{fe_{\text{in}} \times fe_{\text{in}}}$ combines the representations with bias $b_{\text{comb}_1}$. In addition, in order to make the edge representation symmetric, we calculate the edge representation as $\mathbf{z}_{\mathbf{e}_{ij}} = {(\mathbf{z}_{\mathbf{e}_{ij}} + \mathbf{z}_{\mathbf{e}_{ji}})}/{2}$. The resulting edge embedding is $\mathbf{Z_E}\in \mathbb{R}^{n\times n \times fe_{\text{in}}}$.

\noindent \textbf{3. Graph Encoding} \quad
Standard GNNs have limitations in detecting simple substructures such as cycles \cite{chen2020can}, which can hinder their ability to accurately capture the properties of the data distribution. To overcome this limitation, we enhance our model by extra features as follows.

In the graph encoding process, we encode graph-level structural $\mathbf{h}_{\text{struct}}$, spectral $\mathbf{h}_{\text{spect}}$ and molecular information $\mathbf{h}_{\text{mol}}$ to a molecule ${\mathbf{M}}$. As suggested in \cite{beaini2021directional, vignac2022digress}, we add cycle counts for $\mathbf{h}_{\text{struct}}$, the number of cycles of up to size 6 and the number of connected components; and we add eigenvalue features to $\mathbf{h}_{\text{spect}}$, including the multiplicity of eigenvalue 0, as well as the first 5 nonzero eigenvalues. For $\mathbf{h}_{\text{mol}}$, we include the current valency of each atom and the current molecular weight of the entire molecule as features. For every molecule $\textbf{M}$, the graph-level representation is given by,
\begin{equation}
    \mathbf{Z_M} = W_{\text{graph}}([\mathbf{h}_{\text{struct}}, \mathbf{h}_{\text{spect}}, \mathbf{h}_{\text{mol}}]) + b_{\text{graph}},
\end{equation}
where $W_{\text{graph}} \in \mathbb{R}^{13 \times f_{\text{in}}}$ combines the encoded information with bias $b_{\text{graph}}$. The resulting graph representation, $\mathbf{Z_M} \in \mathbb{R}^{1\times f_{\text{in}}}$, encapsulates all the information from the structural, spectral and molecular features.

\noindent \textbf{4. 2D Neighborhood Encoding} \quad
To get local neighborhood information, we use message passing to aggregate information from the immediate neighbors of each atom in the 2D graph $\mathbf{M}^{\text{2D}}$. Specifically, for $\mathbf{h}_i \in \mathbf{H}$ and $\mathbf{e}_{ij} \in E$, the aggregated representation of atom $i$ is calculated by
\begin{equation}
\begin{aligned}
    &\mathbf{m}_{j\to i} = \left( W_{\text{atom}_3}\mathbf{h}_j \right) \odot \left( W_{\text{edge}_2}\mathbf{e}_{ij} \right)  , \ \mathbf{z}^{\text{2D}}_{\mathbf{h}_i} = W_{\text{atom}_3}\mathbf{h}_i + \frac{1}{|N(i)|}\sum_{j\in N(i)} \mathbf{m}_{j\to i} ,
\end{aligned}
\end{equation}
where $N(i)$ denotes the neighbors of atom $i$, $W_{\text{atom}_3} \in \mathbb{R}^{d \times fh_2} $ and $W_{\text{edge}_2} \in \mathbb{R}^{b \times fh_2}$ are learnable parameters for atom features and edge types, respectively. To formulate the messages $\mathbf{m}$, we use the Hadamard product, $\odot$, between the atom embedding and edge embedding. This new representation $\mathbf{z}^{\text{2D}}_{\mathbf{h}_i}\in \mathbb{R}^{fh_2}$ takes into account not only the atom features, but also the features of its neighboring atoms and the edges connecting them.

\noindent \textbf{5. 3D Neighborhood Encoding} \quad
We use message passing to collect the atom information in the vicinity of the given atom to get the 3D neighborhood encoding $\mathbf{M}^{\text{3D}}$ as suggested by \cite{satorras2021n}. For atom $i$, its representation $\mathbf{z}^{\text{3D}}_{\mathbf{h}_i}$ is computed by following steps,
\begin{equation}
\begin{aligned}
\label{app.eq:3d.neigh.encoding}
    & {d}_{ij} = \|{x}_i - {x}_j\|_2, \ \mathbf{e}_{ij} = \text{SiLU} \left(W_{\text{edge}_3}\left(f_{\text{RBF}_1}\left( d_{ij} \right)\right) \right) \odot f_{\cos}(d_{ij}) \\
    &\mathbf{m}_{j\to i}  = (W_{\text{atom}_4}\mathbf{h}_j) \odot \mathbf{e}_{ij} , \ \mathbf{z}^{\text{3D}}_{\mathbf{h}_i} = W_{\text{atom}_4}\mathbf{h}_i + \frac{1}{|N(i)|}\sum_{j\in N(i)} \mathbf{m}_{j\to i}.
\end{aligned}
\end{equation}
where $d_{ij}$ is the distance from atom $i$ to $j$ in the Euclidean space, $f_{\text{RBF}_1}(\cdot)$ is the exponential radial basis function, $f_{\cos}(\cdot)$ is the cosine cutoff, SiLU($\cdot$) is the activation function, $\odot$ denotes the Hadamard product, $W_{\text{atom}_4} \in \mathbb{R}^{d\times fh_3}$ and $W_{\text{edge}_3} \in \mathbb{R}^{k \times fh_3 }$ are designated learnable parameters for atom features and edge features, $k$ is the number of basis kernels as mentioned in Sec~\ref{sec:basics}. To calculate messages $\mathbf{m}$, we use the Hadamard product, $\odot$, between the atom and edge embeddings.

We use the cosine cutoff $f_{\cos}(d_{ij})$ to determine which atoms in the Euclidean space given by $\mathbf{M}^{\text{3D}}$ should be considered as part of the neighborhood of atom $i$. This provides a smooth and differentiable way to incorporate spatial locality in the message-passing process, focusing on atoms that are closer in the 3D space while ignoring distant ones. This allows the model to better capture local geometric information and reduce computational complexity by not considering interactions between atoms that are too far apart, which would be less relevant for the molecular properties under investigation. Only atoms $j$ for which $f_{\cos}(d_{ij}) > 0$ are included in the message passing aggregation process, resulting in a new node representation $\mathbf{z}^{\text{3D}}_{\mathbf{h}_i} \in \mathbb{R}^{fh_3}$ for atom $i$. 

\noindent \textbf{Remark} \quad We can integrate 2D graph information by assigning $f_{\cos}(d_{ij})=1$ when an edge is present between atoms $i$ and $j$ in the 2D molecular structure $\mathbf{M}^{\text{2D}}$. This ensures that messages between these atoms are not influenced by the smooth transition. By doing so, we effectively incorporate locality information from both the 2D molecular structure and the 3D geometric structure, providing a more comprehensive representation of the molecule.

\noindent \textbf{6. Combine Encoding} \quad
In the final step of the encoding process, we compute the atomic embedding $\mathbf{Z_H}$ by concatenating the centrality embedding, 2D neighborhood embedding, and 3D neighborhood embedding. This concatenated representation is then passed through a learnable parameter, $W_{\text{comb}2} \in \mathbb{R}^{(fh_1 + fh_2 + fh_3) \times f{\text{in}}}$. The final equation for this calculation is given by
\begin{equation}
    \mathbf{Z_H} = W_{\text{comb}_2}([\mathbf{Z^{\text{1D}}_H}, \mathbf{Z^{\text{2D}}_H}, \mathbf{Z^{\text{3D}}_H}]) + b_{\text{comb}_2}.
\end{equation}
This combination of different embeddings, $\mathbf{Z_H}\in\mathbb{R}^{n\times f_{\text{in}}}$, enables the incorporation of various molecular structure features, including centrality, 2D, and 3D neighborhood information, into a unified, comprehensive representation.

Additionally, the bond encoding $\mathbf{Z_E}$ and graph encoding $\mathbf{Z_M}$ are utilized to calculate attentions, with details discussed in Sec~\ref{sec:transformer_channels}.

\subsection{Attention Biases}
\label{sec:muformer.attention.bias}
The MUformer incorporates 4 attention biases, which serve to encode spatial relationships in both 2D molecular structure and 3D geometric arrangement. These biases are integrated into the attention mechanism, enhancing the model's ability to process and understand molecular representations. A detailed discussion of the importance and advantages of employing these attention biases for computing attentions in the MUformer can be found in Sec~\ref{sec:discussion.attention.bias}.

\noindent \textbf{1. 2D Spatial Attention Bias} \quad
To encode the structural relationships between atoms in the molecule graph $\mathbf{M}^{\text{2D}}$, we usethe shortest path distance (SPD) encoding  \cite{ying2021transformers}, denoted as $\Phi^{\text{2D}}_{\text{SPD}}(i,j): V\times V \to \mathbb{R}$, which calculates the distance between atoms $i$ and $j$ in $\mathbf{M}^{\text{2D}}$, providing valuable information about the structural relationships between atoms in the 2D molecular graph.

Additionally, following the approach of \cite{ying2021transformers}, we incorporate edge-type information along the shortest path between atoms $i$ and $j$ to reflect edge characteristics. This inclusion of edge-type information further enriches the model's understanding of the 2D molecular structure. To achieve this, we determine the shortest path $SP_{ij} = (\mathbf{e}_1, \mathbf{e}_2, ..., \mathbf{e}_N)$, where $N$ is the longest shortest path distance for all pairs of atoms $i$ and $j$. The edge encoding is computed using the following equation,
\begin{equation}
\Phi^{\text{2D}}_{\text{ENC}_{ij}} = \frac{1}{N}\sum_{n=1}^N\mathbf{e}_n(w_n)^T \in \mathbb{R} 
\end{equation}
where $w_n \in \mathbb{R}^{b \times 1}$ is a learnable vector with the same dimension as the edge feature. Both 2D spatial biases, $\Phi^{\text{2D}}_{\text{SPD}}$ and $\Phi^{\text{2D}}_{\text{ENC}}$, are in $\mathbb{R}^{n\times n\times 1}$, and the combined bias is calculated as $\Phi^{\text{2D}}_\mathbf{E} = \Phi^{\text{2D}}_{\text{SPD}} + \Phi^{\text{2D}}_{\text{ENC}}$. This 2D spatial attention bias enables the model to better capture the intricate relationships between atoms and their surroundings in the 2D molecular graph, ultimately improving its performance.

\noindent \textbf{2. 3D Spatial Attention Bias} \quad
The 3D spatial relationships between atom pairs in $\mathbf{M}^{\text{3D}}$ can be effectively encoded using the Euclidean distance and an exponential radial basis function, $f_{\text{RBF}2}(\cdot)$ \cite{schutt2017schnet}. By incorporating this 3D spatial attention bias, the model is able to account for the geometric arrangement of atoms in the molecule, which is essential for understanding the 3D structure and its impact on molecular properties. The 3D spatial bias, $\Phi^{\text{3D}}_{\mathbf{E}}$, is calculated using the following equation
\begin{equation}
\begin{aligned}
    & {d}_{ij} = \|{x}_i - {x}_j\|_2 , \ \Phi^{\text{3D}}_{\mathbf{e}_{ij}} = W_{\text{3D}_2}\left(\text{SiLU}(W_{\text{3D}_1}\left(f_{\text{RBF}_2}\left( d_{ij} \right)\right))\right),
\end{aligned}
\end{equation}
where $W_{\text{3D}_1}\in \mathbb{R}^{k \times k}, W_{\text{3D}_2}\in \mathbb{R}^{k \times 1}$ are learnable parameters, $k$ is the number of basis kernels, and the resulting 3D spatial bias, $\Phi^{\text{3D}}_{\mathbf{E}}$, is in $\mathbb{R}^{n\times n \times 1}$. By including this 3D spatial attention bias in the model, the MUformer is able to better capture the complex 3D spatial relationships between atoms, leading to improved performance in tasks involving 3D molecular structures.

\noindent \textbf{Edge Feature \& Graph Feature} \quad
The embeddings obtained from the 2D structure, $\mathbf{Z_E}$, and the molecular graph information, $\mathbf{Z_M}$, can be further projected and employed as additional attention biases, enhancing the model's understanding of the molecular structure and relationships. This allows the MUformer to better capture the complexities of the molecular system and improve its performance in various tasks. A detailed explanation of this process can be found in the subsequent section.

\subsection{Transformer Channels}
\label{sec:transformer_channels}
Our MUformer architecture draws inspiration from Transformer-M \cite{luo2022one}, which employs two separate channels to process 2D and 3D molecular data, respectively. However, our model takes a different approach to processing the invariant and equivariant features of molecular data. While Transformer-M is limited to predicting atom features only, and is invariant to geometric transformations, our model can predict atom features, molecule structures, and atom positions, and is equivariant to geometric transformations. Our model can be used under different conditions of the input data. When the input data only contains 2D molecular information $\mathbf{M}^{\text{2D}}=(\mathbf{H}, \mathbf{E})$ and the geometric structure is missing, only the invariant channel is applied, and the model predicts invariant features including atom features $\mathbf{H}_{\text{out}}$ and molecular structure $\mathbf{E}_{\text{out}}$. Similarly, when the input data only contains 3D molecular information $\mathbf{M}^{\text{3D}}=(\mathbf{H}, \mathbf{X})$ and the molecular structure is missing, only the equivariant channel is used and the model becomes insensitive to geometric information, predicting atom features $\mathbf{H}_{\text{out}}$ and coordinates $\mathbf{X}_{\text{out}}$. Finally, when both 2D and 3D molecular data are provided as input, both the invariant and equivariant channels are activated. The model is equivariant to geometric transformations, predicting the complete molecule including atom features $\mathbf{H}_{\text{out}}$, molecular structure $\mathbf{E}_{\text{out}}$, and geometric structure $\mathbf{X}_{\text{out}}$.

The MUformer architecture utilizes two channels, the \textit{invariant channel} and the \textit{equivariant channel}, to learn the 2D molecular structure and 3D geometric structure, respectively. 
We simplify the notation by omitting the indices of the attention head $h$ and layer $l$.

\paragraph{Invariant Channel}
The invariant channel is an improved version of the one presented in \cite{ying2021transformers}, which is specifically designed to extract the inherent characteristics of the input molecule graph $\mathbf{M}^{\text{2D}}$, and is utilized to make predictions for atom and edge features by leveraging the underlying graph structure.

We enhance the MUformer's attention mechanism by incorporating pair-wise information in the invariant channel. 
First, we calculate intermediate representations for the edge and graph features
\begin{equation}
    \begin{aligned}
        &\mathbf{Z_{E_1}} = W_{\mathbf{E_1}}\mathbf{Z_E}, \ \mathbf{Z_{E_2}} = W_{\mathbf{E_2}}\mathbf{Z_E} \\
        &\mathbf{Z_{M_1}} = W_{\mathbf{M_1}}\mathbf{Z_M}, \ \mathbf{Z_{M_2}} = W_{\mathbf{M_2}}\mathbf{Z_M}, \ 
    \end{aligned}
\end{equation}
using weight matrices $W_{\mathbf{E_1}}, W_{\mathbf{E_2}}, W_{\mathbf{M_1}}, W_{\mathbf{M_2}}$. We then compute the attention weights by taking the dot product of the query and key, and modify them by multiplying and adding the intermediate representations,
\begin{equation}
    \begin{aligned}
       &\mathbf{A} = \frac{(W_Q\mathbf{Z_H})^T(W_K\mathbf{Z_H})}{\sqrt{F}} \in \mathbb{R}^{n\times n\times F} \\
        &\mathbf{A} = \mathbf{A} \times (\mathbf{Z_{E_1}+1)} + \mathbf{Z_{E_2}}.
    \end{aligned}
\end{equation}
And the predicted edge and graph representations, $\mathbf{\hat{Z}_E}, \mathbf{\hat{Z}_M}$, are computed from the attention weights as
\begin{equation}
    \begin{aligned}
       &\mathbf{\hat{Z}_E} = W_{\mathbf{E}_{\text{out}}}((\mathbf{A} \times (\mathbf{Z_{M_1}+1)} + \mathbf{Z_{M_2}}) \\
        &\mathbf{\hat{Z}_M} = W_{\mathbf{M}_{\text{out}}}(f_{\text{Node2Graph}}(\mathbf{Z_H})+ f_{\text{Edge2Graph}}(\mathbf{Z_E})+ \mathbf{{Z}_M}),
    \end{aligned}
\end{equation}
where $f_{\text{Node2Graph}}(\cdot)$ and $f_{\text{Edge2Graph}}(\cdot)$ are designated functions (see Eq.~\ref{eq:node2graph},\ref{eq:edge2graph}) that map node and edge features to graph features, respectively.  
Finally, the spatial relationships in 2D and 3D are added to the attention weights, the attention is passed through a softmax function and the predicted representation $ \mathbf{\hat{Z}_F}$ is obtained by the equation,
\begin{equation}
    \begin{aligned}
       & \mathbf{A} = \text{softmax}(\mathbf{A} + \Phi^{\text{2D}}_{\mathbf{E}} + \Phi^{\text{3D}}_{\mathbf{E}}) \in \mathbb{R}^{n\times n \times F}\\
        & \mathbf{\hat{Z}_H} = \mathbf{{Z}_H} + W_{\mathbf{H}_\text{out}}((W_V\mathbf{Z_H})\mathbf{A}).
    \end{aligned}
\end{equation}
The invariant channel can capture the inherent features of the input molecule graph, allowing for predictions of discrete 2D structures, as well as invariant atom and graph features.

\paragraph{Node2Graph \& Edge2Graph Functions}
Node2Graph $f_{\text{Node2Graph}}(\cdot)$ and Edge2Graph $f_{\text{Edge2Graph}}(\cdot)$ functions map node- and edge-level features to graph-level features, respectively.

The Node2Graph function transforms the node features $\mathbf{H} \in \mathbb{R}^{n\times F_{\text{in}}}$ by computing the mean, max, and min values for each node, then concatenating them and applying a linear transformation with weight matrix $W_{\text{Node2Graph}}$ and bias $b_{\text{Node2Graph}}$,
\begin{equation}
\label{eq:node2graph}
\begin{aligned}
    & \mathbf{H}_{\text{mean}} = \text{mean}(\mathbf{H}) \in \mathbb{R}^{1\times F_{\text{in}}} \\
    & \mathbf{H}_{\text{max}} = \text{max}(\mathbf{H})  \\
    & \mathbf{H}_{\text{min}} = \text{min}(\mathbf{H}) \\
    & \mathbf{H}_{\text{out}} = W_{\text{Node2Graph}}([\mathbf{H}_{\text{mean}}, \mathbf{H}_{\text{max}}, \mathbf{H}_{\text{min}}]) + b_{\text{Node2Graph}} \in \mathbb{R}^{1\times F_{\text{out}}}.
\end{aligned}
\end{equation}

The Edge2Graph function transforms the node features $\mathbf{E}\in \mathbb{R}^{n\times n \times F_{\text{in}}}$ by computing the mean, max, and min values for each node pair $(i, j)$, then concatenating them and applying a linear transformation with weight matrix $W_{\text{Edge2Graph}}$ and bias $b_{\text{Edge2Graph}}$,
\begin{equation}
\label{eq:edge2graph}
\begin{aligned}
    & \mathbf{E}_{\text{mean}} = \text{mean}(\mathbf{E}) \in \mathbb{R}^{1\times 1 \times F_{\text{in}}} \\
    & \mathbf{E}_{\text{max}} = \text{max}(\mathbf{E}) \\
    & \mathbf{E}_{\text{min}} = \text{min}(\mathbf{E}) \\
    & \mathbf{E}_{\text{out}} = W_{\text{Edge2Graph}}([\mathbf{E}_{\text{mean}}, \mathbf{E}_{\text{max}}, \mathbf{E}_{\text{min}}]) + b_{\text{Edge2Graph}} \mathbb{R}^{1\times 1 \times F_{\text{out}}}.
\end{aligned}
\end{equation}

\paragraph{Equivariant Channel}
The equivariant channel is an upgraded version of the one presented in \cite{tholke2022torchmd}. It is specifically engineered to extract the features of the input molecule graph $\mathbf{M}^{\text{3D}}$ that change under 3D rotations and translations. This channel is used to make predictions for atom features and coordinates by leveraging the 3D geometric structure of the molecule. It is activated when only the 3D geometric structure $\mathbf{M}^{\text{3D}}$ is provided. The velocity features ${\mathbf{v}}\in \mathbb{R}^{n\times 3 \times F}$ are initialized to $0$.

First, we calculate the distance between each pair of atoms, $d_{ij}$, and project them into a multidimensional filter for keys. The attention weights are calculated by taking the dot product of the query, key, and filter, and are modified by incorporating the 3D spatial relationship between atoms. The attention weights are then passed through a softmax function, and the cosine cutoff is applied to the weights to ensure that atoms with a distance larger than $d_\text{cut}$ do not interact.
\begin{equation}
    \begin{aligned}
       & {d}_{ij} = \|{x}_i - {x}_j\|_2 \\
       & D_{K} = \text{SiLU}(W_{\text{dist}_K}(f_{\text{RBF}_3}(d)) + b_{\text{dist}_K}) \in \mathbb{R}^{n\times n \times F}\\
       & \mathbf{A} = \frac{(W_Q\mathbf{Z_H})^T(W_K\mathbf{Z_H})\odot D_{K}}{\sqrt{F}} \\ 
       & \mathbf{A} = \text{softmax}(\mathbf{A}+\Phi^{\text{3D}}_{\mathbf{E}}) \odot f_{\cos}(d).
    \end{aligned}
\end{equation}
Here, the final attention weights can also include 2D spatial relationship by adding 2D spatial relationship term $\Phi^{\text{2D}}_{\mathbf{E}}$ to the equation, and setting $f_{\cos}(d_{ij})=1$ if an edge exists between atoms $i$ and $j$ in the 2D molecule graph $\mathbf{M}^{\text{2D}}$.

In order to incorporate interatomic distances into the features directly, we also project the distance between atoms into a multidimensional filter for values. This approach, which has been used in \cite{schutt2017schnet, tholke2022torchmd}, enables the model to not only consider interatomic distances in the attention weights, but also to incorporate this information into the features themselves.
\begin{equation}
    \begin{aligned}
        & D_{V} = \text{SiLU}(W_{\text{dist}_V}(f_{\text{RBF}_3}(d)) + b_{\text{dist}_V}) \in \mathbb{R}^{n\times n\times 3F}\\
       & \mathbf{Z}_{V} = W_{V}\mathbf{Z_H} \in \mathbb{R}^{n \times 3F} \\
       & \mathbf{Z}_{V_1},\mathbf{Z}_{V_2},\mathbf{Z}_{V_3} = \text{split}(\mathbf{Z}_V \odot D_V) \in \mathbb{R}^{n\times n\times F} \\
        & \mathbf{Z}_{O} = W_{O}(\mathbf{Z}_{V_1}\mathbf{A}) \in \mathbb{R}^{n \times 3F} \\
       & \mathbf{Z}_{O_1},\mathbf{Z}_{O_2},\mathbf{Z}_{O_3} = \text{split}(\mathbf{Z}_O) \in \mathbb{R}^{n\times F},
    \end{aligned}
\end{equation}
where the function split$(\cdot)$ divides the input into three equal-sized parts.

Then, we use a weight matrix $W_{{\mathbf{v}}}$ to project the velocity features ${\mathbf{v}}$ into three separate vectors,
\begin{equation}
    \begin{aligned}
       & \mathbf{Z}_{{\mathbf{v}}} = W_{{\mathbf{v}}}{\mathbf{v}} \in \mathbb{R}^{n\times 3 \times 3F} \\
       & \mathbf{Z}_{{{\mathbf{v}}}_1},\mathbf{Z}_{{{\mathbf{v}}}_2}, \mathbf{Z}_{{{\mathbf{v}}}_3} = \text{split}(\mathbf{Z}_{{\mathbf{v}}}) \in \mathbb{R}^{n\times 3\times F}.
    \end{aligned}
\end{equation}

Finally, new atom and velocity features, $\mathbf{\hat{Z}_H},\hat{{\mathbf{v}}}$, are calculated following the steps in \cite{tholke2022torchmd}. The atom features are updated by adding the residual of the scaled features $\mathbf{Z}_{O_1}$ and the inner product between velocity projections $\langle \mathbf{Z}_{{{\mathbf{v}}}_1}, \mathbf{Z}_{{{\mathbf{v}}}_2} \rangle$. The velocity features are updated by incorporating equivariant features using the edge directional information $d_{ij}$ and scaled vector features.
\begin{equation}
    \begin{aligned}
       & \mathbf{\hat{Z}}_{\mathbf{H}} = \mathbf{Z_H} + (\mathbf{Z}_{O_1} + \mathbf{Z}_{O_2} \odot \langle  \mathbf{Z}_{{{\mathbf{v}}}_1} , \mathbf{Z}_{{{\mathbf{v}}}_2} \rangle) \\
       & {\mathbf{w}}_i = \sum_{j \in N(i)} \mathbf{Z}_{V_{2, ij}} \odot {\mathbf{v}}_i + \mathbf{Z}_{V_{3, ij}} \odot d_{ij} \\
       & \hat{{\mathbf{v}}} =  {\mathbf{v}} + ({\mathbf{w}} + \mathbf{Z}_{O_3} \odot \mathbf{Z}_{{{\mathbf{v}}}_3}).
    \end{aligned}
\end{equation}

\noindent \textbf{Interaction Embedding} \quad
We denote the predicted atom features of the invariant channel and the equivariant channel as $\mathbf{Z}_{\mathbf{H}}^{\text{inv}}$ and $\mathbf{Z}_{\mathbf{H}}^{\text{eqv}}$, respectively. We combine these predictions by multiplying them with a weight matrix ${W}_{\text{comb}_3}$, and adding a bias term ${b}_{\text{comb}_3}$,
\begin{equation}
    \mathbf{\hat{Z}_H} = {W}_{\text{comb}_3}[\mathbf{Z}_{\mathbf{H}}^{\text{inv}}, \mathbf{Z}_{\mathbf{H}}^{\text{eqv}}] + {b}_{\text{comb}_3}.
\end{equation}
By doing so, we obtain a mixed feature that includes rich invariant representations. This mixed atom feature $\mathbf{\hat{Z}_H}$ is then fed into the next layer of the transformer channels or used as input for the final predictions.

\noindent \textbf{Output Block} \quad
The output block generates the final output by utilizing the embeddings from invariant and equivariant channels. Specifically, it takes in the atom $\hat{\mathbf{Z}}_{\mathbf{H}}$ and edge embeddings $\hat{\mathbf{Z}}_{\mathbf{E}}$, along with the velocity embedding $\hat{{\mathbf{v}}}$. Through feature extractions, the output block makes predictions for the atom features $\mathbf{H}_{\text{out}}$, edge features $\mathbf{E}_{\text{out}}$, and atom coordinates $\mathbf{X}_{\text{out}}$,
\begin{equation}
\begin{aligned}
    & \mathbf{H}_{\text{out}} = W_{{\mathbf{X}}_{\text{out}_2}}\left(\text{SiLU}(W_{{\mathbf{X}}_{\text{out}_1}}\hat{\mathbf{Z}}_{\mathbf{H}})\right) \in \mathbb{R}^{n\times d_{\text{out}}}\\
    & \mathbf{E}_{\text{out}} = W_{{\mathbf{E}}_{\text{out}_2}}\left(\text{SiLU}(W_{{\mathbf{E}}_{\text{out}_1}}\hat{\mathbf{Z}}_{\mathbf{E}})\right) \in \mathbb{R}^{n\times n \times b_{\text{out}}}\\
    & \mathbf{X}_{\text{out}} = \mathbf{X} + W_{{{\mathbf{v}}}_{\text{out}_2}}\left(\text{SiLU}(W_{{{\mathbf{v}}}_{\text{out}_1}}\hat{{\mathbf{v}}})\right) \in \mathbb{R}^{n\times 3}.
\end{aligned} 
\end{equation}

\noindent \textbf{Analysis of Memory Complexity} \quad
We compare the memory complexity of our method, MUDiff, with two existing methods: EDM~\cite{hoogeboom2022equivariant} and DiGress~\cite{vignac2022digress}. Considering atom features of size $n \times d$, atom positions of size $n \times 3$, and edge features of size $n \times n \times b$, where $n$ is the number of atoms, $d$ is the dimension of atom features, and $b$ is the dimension of edge features, EDM's memory complexity is $O(nd + 3n)$, and DiGress's is $O(nd + n^2b)$. MUDiff has a higher memory complexity of $O(nd + 3n + n^2b)$, but offers a more comprehensive molecular representation by including both 2D and 3D information for topological and geometric structures. For more on scalability issues and potential solutions, see Sec~\ref{sec:discussion}.

\section{Experiments}
\label{sec:experiments}
To evaluate our MUDiff framework, we conduct experiments on the QM9 dataset \cite{ramakrishnan2014quantum}, which contains 130k small molecules with up to 9 heavy atoms (29 atoms including hydrogens) and their associated molecular properties and structures.
We use the train/val/test splits from \cite{anderson2019cormorant}, consisting of 100K/18K/13K samples respectively, for evaluation. This protocol follows the method used in previous works such as \cite{satorras2021n, hoogeboom2022equivariant}. 

\subsection{Molecule Generation with Limited 3D Data}
\label{sec:molecule.generation.limit.3d}
\begin{table}[ht!]
\centering
  \caption{Negative log-likelihood, atom stability, and molecule stability are evaluated with standard deviation across 3 runs on QM9, using 10K samples from the model. 30K+70K means model trained with limited 3D data.}
    \footnotesize
    \begin{tabular}{lccc}
    \hline
    \hline
    \textbf{Method} & NLL & Atom Stable(\%) & Mol Stable(\%)\\
    \hline
    \textbf{EDM} & -$110.7 \pm 1.5$ & $ {98.7} \pm 0.1 $ & $82.0 \pm 0.4 $\\
    \textbf{DiGress} & - & $98.1 \pm 0.3 $  & $79.8 \pm 5.6 $\\
    \textbf{MUDiff} & ${-135.5} \pm 2.1 $  &   $ {98.8} \pm 0.2 $    &  $ {89.9} \pm 1.1 $\\
    \textbf{MUDiff} (30K+70K) & ${-120.6} \pm 3.4 $  &   $ {98.2} \pm 0.7 $    &  $ {84.5} \pm 2.5 $\\
    \hline
    \hline
    \end{tabular}%
  \label{tab:limited.3D.data}%
\end{table} 

In this section, we introduce a new molecule generation task that incorporates limited 3D data, as many real-world datasets lack complete 3D structures. To accomplish this, we randomly split the 100K training molecules into two sets: 30K with both 2D and 3D structures and 70K with only 2D structures. We train the model on the 30K samples using both the invariant and equivariant channels and validate on 18K samples until NLL converges.
We then fine-tune the trained model on the remaining 70K molecules with only 2D structures and validate/test on 18K/13K samples. Notably, this training framework with limited 3D data is only possible with MUDiff for now, because of the flexible two-channel design.

\paragraph{Results}
The results of the molecule generation task with limited 3D data are summarized in Table~\ref{tab:limited.3D.data}. MUDiff achieved competitive results in generating stable molecules, even with limited 3D information in the training set, compared to the baselines. These results suggest that MUDiff has the ability to leverage sufficient 2D structures to infer 3D geometry. This finding may motivate further research on the co-generation of 2D and 3D structures for molecules.

\subsection{Molecule Generation}
\label{sec:molecule.generation}
\begin{table}[ht!]
\centering
  \caption{Negative log-likelihood, atom stability, and molecule stability are evaluated with standard deviation across 3 runs on QM9, using 10K samples (with hydrogen) from the model. The results surpass those of previous models, as reported in \cite{hoogeboom2022equivariant, vignac2022digress}. }
    \footnotesize
    \begin{tabular}{lccc}
    \hline
    \hline
    \textbf{Method} & NLL & Atom Stable(\%) & Mol Stable(\%)\\
    \hline
    \textbf{Data} & - & $99.0$   & $95.2 $\\
    \hline
    \textbf{ENF} & -$59.7$ & $85.0 $   & $4.9 $\\
    \textbf{GSchnet} & - & $95.7$  & $68.1 $\\
    \textbf{GDM} & -$92.5$ & $97.6 $ & $71.6 $\\
    \textbf{EDM} & -$110.7 \pm 1.5$ & $ {98.7} \pm 0.1 $ & $82.0 \pm 0.4 $\\
    \textbf{DiGress} & - & $98.1 \pm 0.3 $  & $79.8 \pm 5.6 $\\
    \textbf{MDM} & - & $98.6 $  & $\textbf{91.9} $\\
    \textbf{GeoLDM} & - & $\textbf{98.9} \pm 0.1 $  & $89.4 \pm 0.5 $\\
    \textbf{MUDiff} (ours) & $\textbf{-135.5} \pm 2.1 $  &   $ \textbf{98.8} \pm 0.2 $    &  $ \textbf{89.9} \pm 1.1 $\\
    \hline
    \hline
    \end{tabular}%
  \label{tab:cond.generation.stable}%
\end{table}

\begin{table}[ht!]
  \centering
  \caption{Validity and uniqueness of over 10K molecules are shown with standard deviation across 3 runs, surpassing the results of previous models according to studies by \cite{hoogeboom2022equivariant, vignac2022digress}. }
  \footnotesize
    \begin{tabular}{lccc}
    \hline
    \hline
    \textbf{Method} & w/ Hydrogen & Valid (\%) & Unique (\%) \\
    \hline
    \textbf{Data}  & & $99.3$  & $100.0$ \\
    \hline
    \textbf{GraphVAE} & & $55.7$  & $42.3$ \\
    \textbf{Set2GraphVAE} &  & $59.9 \pm 1.7$  & $56.2 \pm 1.4$ \\
    \textbf{EDM}   & & $97.5 \pm 0.2$  & $94.3 \pm 0.2$ \\
    \textbf{DiGress} &  & $ \textbf{99.0} \pm 0.1$    & $96.2 \pm 0.1$ \\
    \textbf{MUDiff} (ours) &   &   $ \textbf{98.9} \pm 0.4$  &  $ \textbf{99.3} \pm 0.3$ \\
    \hline
    \textbf{Data}  & \checkmark & $97.8$  & $100.0$ \\
    \hline
    \textbf{ENF}   & \checkmark & $40.2$  & $39.4$ \\
    \textbf{GSchnet} & \checkmark & $85.5$  & $80.3$ \\
    \textbf{GDM}   & \checkmark & $90.4$  & $89.5$ \\
    \textbf{EDM}   & \checkmark & $91.9 \pm 0.5$  & $90.7 \pm 0.6$ \\
    \textbf{DiGress} & \checkmark & $\textbf{95.4} \pm 1.1$  & $97.6 \pm 0.4$ \\
    \textbf{GeoLDM} & \checkmark & ${93.8} \pm 0.4$  & $92.7 \pm 0.5$ \\
    \textbf{MUDiff} (ours) & \checkmark &  $\textbf{95.3} \pm 1.5$    &  $\textbf{99.1} \pm 0.5$ \\
    \hline
    \hline
    \end{tabular}%
  \label{tab:cond.generation.valid}%
\end{table} 

We compare the performance of our MUDiff  model with popular generative models, including GraphVAE \cite{kipf2016variational}, GSchenet \cite{gebauer2019symmetry}, Set2GraphVAE \cite{vignac2021top}, ENF \cite{satorras2021n}, GDM \cite{hoogeboom2022equivariant}, EDM \cite{hoogeboom2022equivariant}, DiGress \cite{vignac2022digress}, MDM \cite{huang2023mdm}, and GeoLDM \cite{xu2023geometric}. 
The results of the baseline models can be found in the studies by \cite{hoogeboom2022equivariant} and \cite{vignac2022digress}.

As outlined in \cite{satorras2021n}, we evaluate the atom and molecule stability of the generated compounds by measuring the proportion of atoms that have the correct valency for atom stability, and the proportion of generated molecules in which all atoms are stable for molecule stability. Additionally, we also measure the validity and uniqueness using the RDKit tool, as used in \cite{hoogeboom2022equivariant}. 

We would like to emphasize that the dataset statistics are not ideal, with atom stability at 99\%, molecule stability at 95.2\%, and molecule validity at 99.3\% in the original data. 
These statistics are not perfect, pointing to potential imperfections in the dataset. The imperfections of the dataset have also been acknowledged in \cite{satorras2021n, hoogeboom2022equivariant, vignac2022digress}.

Table~\ref{tab:cond.generation.stable} presents the evaluated results of atom and molecule stability for molecules generated by MUDiff and the baseline models. The reported average results and standard deviations are over 3 runs, using 10,000 samples from each model. The table shows that MUDiff can generate molecules that are significantly more stable than the baseline models in terms of negative log-likelihood and molecule stability and matches the performance of SOTA model with respect to atom stability.

Table~\ref{tab:cond.generation.valid} presents the results of the validity and uniqueness of the generated samples. It should be noted that, following the guidelines outlined in \cite{vignac2021top}, novelty is not reported in this table. The table shows that MUDiff generates a significantly higher rate of unique molecules than the baselines and matches the rate of valid molecules of SOTA models.

\subsection{Conditional Generation}
We follow the experimental setting in \cite{hoogeboom2022equivariant} to train the conditional MUDiff model on the QM9 dataset, conditioning the generation on properties $\alpha$, $\epsilon_{\text{homo}}$, $\epsilon_{\text{lumo}}$, $\Delta\epsilon$, $\mu$, and $C_v$, respectively. 

Additionally, we follow \cite{hoogeboom2022equivariant} to use a property classifier $\psi_c$ proposed in \cite{satorras2021n2}. We split QM9 training data into two halves, \textit{A} and \textit{B}, each containing 50K samples, and use \textit{A} subset to train $\psi_c$ and \textit{B} subset for training the conditional MUDiff. Then, $\psi_c$ is used to evaluate the generated samples of conditional MUDiff. Also, we follow \cite{hoogeboom2022equivariant} to report the loss of $\psi_c$ on \textit{B} as a lower bound (L-bound). 
The smaller the gap between MUDiff and L-bound, the more similar MUDiff generated samples to \textit{B}.

\begin{table}[ht!]
  \centering
  \caption{Mean Absolute Error for the prediction of molecular properties by the property classifier $\psi_c$ on a QM9 subset (L-bound), MUDiff samples and three baselines.}
  \resizebox{9.0cm}{!}{
  \footnotesize
    \begin{tabular}{lcccccc}
    \hline
    \hline
    \textbf{Property} & $\alpha$ & $\epsilon_{\text{homo}}$ & $\epsilon_{\text{lumo}}$ & $\Delta\epsilon$ & $\mu$ & $C_v$ \\
    \textbf{Units} & $a^3$ & $meV$ & $meV$ & $meV$ & $D$   & $meV$ \\
    \hline
    \textbf{U-bound} & 9.01  & 645   & 1457  & 1470  & 1.616 & 6.857 \\
    \textbf{\#Atoms} & 3.86  & 426   & 813   & 866   & 1.053 & 1.971 \\
    \textbf{EDM} & 2.76  & 356   & \textbf{584}   & 655   & 1.111 & 1.101 \\
    \textbf{MUDiff} (ours) &   \textbf{2.15}    &  \textbf{315}     &  597     &  \textbf{604}      &  \textbf{1.033}     &  \textbf{0.978} \\
    \textbf{L-bound} & 0.10   & 39    & 36    & 64    & 0.043 & 0.040 \\
    \hline
    \hline
    \end{tabular}%
    }
  \label{tab:cond.generation.MUDiff}%
\end{table}

We evaluate the performance against the baselines used in EDM \cite{hoogeboom2022equivariant}. In addition to L-bound, they also use two other baselines: U-bound and \#Atoms. The U-bound is obtained by shuffling the properties of molecules in the \textit{B} subset and evaluating $\psi_c$ on it. The \#Atoms baseline predicts the molecular properties in the \textit{B} subset by only using the number of atoms in the molecule.

\paragraph{Results}
Table~\ref{tab:cond.generation.MUDiff} showcases the results of conditional generation on the QM9 dataset. Evidently, conditional MUDiff generates samples that more closely resemble the molecules in subset B compared to the baselines, indicating that MUDiff outperforms baselines in generating molecules with desired properties and capturing structural similarities.

\subsection{Property Prediction}
\label{sec:qm9.property}
\begin{table}[ht!]
  \centering
  \caption{Results on all QM9 targets and comparison to previous literature. Scores are reported as mean absolute errors (MAE) with standard deviation. Results of different models are averaged over three runs.}
  \resizebox{\textwidth}{!}{
  \footnotesize
    \begin{tabular}{llcccccccc}
    \hline
    \hline
    \multicolumn{1}{c}{Target} & \multicolumn{1}{c}{Unit} & SchNet & EGNN  & PhysNet & DimeNet++ & Cormorant & PaiNN & ET    & MUformer \\
    \hline
    $\mu$ & $D$   & 0.033 & 0.029 & 0.0529 & 0.041 & 0.0297 & 0.012 & {0.011} & 0.013 $\pm$ 0.003\\
    $\alpha$ & $a^3_0$ & 0.235 & 0.071 & 0.0615 & {0.0435} & 0.085 & 0.045 & 0.059 & 0.041 $\pm$ 0.008 \\
    $\epsilon_{\text{HOMO}}$ & $meV$   & 41    & 29    & 32.9  & 24.6  & 34    & 27.6  & 20.3  & 24.7 $\pm$ 1.2 \\
    $\epsilon_{\text{LUMO}}$ & $meV$   & 34    & 25    & 24.7  & 19.5  & 38    & 20.4  & 17.5  & 	20.2 $\pm$ 0.8 \\
    $\Delta \epsilon$ & $meV$   & 63    & 48    & 42.5  & 32.6  & 61    & 45.7  & 36.1  & 30.3 $\pm$ 1.7 \\
    \textlangle$R^2$\textrangle     & $a^2_0$ & 0.073 & 0.106 & 0.765 & 0.331 & 0.961 & 0.066 & {0.033} & 0.117 $\pm$ 0.012 \\
    $ZPVE$  & $meV$   & 1.7   & 1.55  & 1.39  & {1.21}  & 2.027 & 1.28  & 1.84  & 1.76  $\pm$ 0.08 \\
    $\mathbf{U}_0$   & $meV$   & 14    & 11    & 8.15  & 6.32  & 22    & 5.85  & 6.15  & 6.11 $\pm$ 0.12 \\
    $\mathbf{U}$      & $meV$   & 19    & 12    & 8.34  & 6.28  & 21    & 5.83  & 6.38  & 6.04 $\pm$ 0.19 \\
    $\mathbf{H}$     & $meV$   & 14    & 12    & 8.42  & 6.53  & 21    & {5.98}  & 6.16  &  6.77 $\pm$ 0.07 \\
    $\mathbf{G}$      & $meV$   & 14    & 12    & 9.4   & 7.56  & 20    & 7.35  & 7.62  &  7.24 $\pm$ 0.08 \\
    $\mathbf{C}_v$   & $\frac{cal}{mol\ K}$ & 0.033 & 0.031 & 0.028 & {0.023} & 0.026 & 0.024 & 0.026 & 0.023 $\pm$ 0.002 \\
    \hline
    \hline
    \end{tabular}%
    }
  \label{tab:qm9.property}%
\end{table}

Additionally, we conduct a comprehensive comparison of our MUformer with several baselines on the QM9 dataset for property prediction, including SchNet \cite{schutt2017schnet}, EGNN \cite{satorras2021n}, PhysNet \cite{unke2019physnet}, DimeNet \cite{beani2021directional}, Cormorant \cite{anderson2019cormorant}, PaiNN \cite{schutt2021equivariant}, and ET \cite{tholke2022torchmd}. The dataset consists of molecules with various properties, and we estimate 12 chemical properties per molecule following \cite{satorras2021n}. The results, which can be found in Table~\ref{tab:qm9.property}, are obtained by averaging over three runs. We use a learning rate of 1e-3, 1e-4, 5e-4 and 1e-5, weight decay of 5e-5, 128 hidden dimensions, and 6 layers for our MUformer model. The results of the baseline models are taken from \cite{tholke2022torchmd}.

\subsection{Ablation Study}
\label{sec:ablation.study}
To improve efficiency, we conduct ablation studies using smaller models than those described previously. Specifically, these models consist of 4 layers, 64 embedding dimensions for atom- and edge-level features, 32 embedding dimensions for graph-level features, 8 attention heads, 100 feedforward dimensions for atom- and edge-level features, 50 feedforward dimensions for graph-level features, 0.3 dropout rate for all latent embeddings and attention values, SiLU activation function, 1e-4 learning rate, and Adam optimizer with 5e-5 weight decay. Additionally, we use a diffusion process with 500 time steps over the course of 3000 training epochs.

Additionally, if 2D structure encoding is not used, the edge type $\mathbf{E}$ is simply embedded by a weight matrix $W$ as
\begin{equation}
\begin{aligned}
\mathbf{Z_E} = W\mathbf{E} \in \mathbb{R}^{n\times n \times fe_{\text{in}}}.
\end{aligned}
\end{equation}

To reduce computation costs, we sample 1,000 molecules from each MUDiff ablation model instead of the 10,000 samples typically generated by the diffusion model during sampling. This allows us to evaluate the performance of the various MUDiff models while minimizing the computational resources required. 

Model variations for ablation study include: (1) no extra technique, (2) 2D structure encoding, (3) 2D structure encoding + 2D neighborhood encoding, (4) 2D structure encoding + 2D neighborhood encoding + 2D spatial attention bias + edge features as attention bias, (5) 2D structure encoding + 3D neighborhood encoding, (6) 2D structure encoding + 3D neighborhood encoding + 3D spatial attention bias, (7) 2D structure encoding + graph encoding + 2D neighborhood encoding + 3D neighborhood encoding, (8) 2D structure encoding + graph encoding + 2D neighborhood encoding + 3D neighborhood encoding + 3D spatial attention bias, (9) 2D structure encoding + graph encoding + 2D neighborhood encoding + 3D neighborhood encoding + 2D spatial attention bias + 3D spatial attention bias, (10) 2D structure encoding + graph encoding + 2D neighborhood encoding + 3D neighborhood encoding + 2D spatial attention bias + 3D spatial attention bias + edge features as attention bias + graph features as attention bias, (11) 2D structure encoding + graph encoding + 2D neighborhood encoding + 3D neighborhood encoding + 2D spatial attention bias + 3D spatial attention bias + edge features as attention bias + graph features as attention bias + 2D discrete graph structures into 3D geometric structures.

\begin{table}[ht!]
  \centering
  \caption{The performance of molecule stability and validity across different ablation models, as shown by the average value with standard deviation of 1K generated molecules (with hydrogen) from each model. Model variations include using $(\romannumeral 1)$ 2D structure encoding, $(\romannumeral 2)$ graph encoding, $(\romannumeral 3)$ 2D neighborhood encoding, $(\romannumeral 4)$ 3D neighborhood encoding, $(\romannumeral 5)$ 2D spatial attention bias, $(\romannumeral 6)$ 3D spatial attention bias, $(\romannumeral 7)$ edge features as attention bias, $(\romannumeral 8)$ graph features as attention bias, and $(\romannumeral 9)$ 2D discrete graph structures into 3D geometric structures.}
   \resizebox{10.0cm}{!}{
  \footnotesize
    \begin{tabular}{llll|llll|lcc}
    \hline
    \hline
    \multicolumn{4}{c|}{\textbf{Encoding}} & \multicolumn{4}{c|}{\textbf{Bias}} &       &       &  \\
    \hline
    $\romannumeral 1$     &  $\romannumeral 2$    &  $\romannumeral 3$    &  $\romannumeral 4$   &  $\romannumeral 5$   &  $\romannumeral 6$    &  $\romannumeral 7$    &  $\romannumeral 8$   &  $\romannumeral 9$   & Mol Stable (\%) & Valid (\%) \\
    \hline
          &       &       &       &       &       &       &       &       & $82.5 \pm 6.3$     & $90.7 \pm 2.2$ \\
    \checkmark &       &       &       &       &       &       &       &       & $82.7 \pm 1.9$     & $91.3 \pm 1.7$ \\
    \checkmark &       & \checkmark &       &       &       &       &       &       & $82.9 \pm 1.5$     & $91.6 \pm 2.1$ \\
    \checkmark &       & \checkmark &       & \checkmark &       & \checkmark &       &       & $84.7 \pm 1.2$     & $92.9 \pm 1.3$ \\
    \checkmark &       &       & \checkmark &       &       &       &       &       & $85.3 \pm 1.0$     & $93.6 \pm 1.9$ \\
    \checkmark &       &       & \checkmark &       & \checkmark &       &       &       & $86.1 \pm 1.7$     & $93.5 \pm 1.3$ \\
    \checkmark & \checkmark & \checkmark & \checkmark &  &       &  &       &       &  $87.2 \pm 1.6$     & $93.4 \pm 1.0$ \\
    \checkmark & \checkmark & \checkmark & \checkmark &  &  \checkmark     &  &       &       &  $87.1 \pm 1.5$     & $94.8 \pm 1.4$ \\
    \checkmark & \checkmark & \checkmark & \checkmark &  \checkmark   & \checkmark &       &       &       & $88.3 \pm 1.5$     & $94.5 \pm 1.3$ \\
    \checkmark & \checkmark & \checkmark & \checkmark & \checkmark & \checkmark & \checkmark & \checkmark &       & $88.4 \pm 1.6$     & $95.1 \pm 1.3$ \\
    \checkmark & \checkmark & \checkmark & \checkmark & \checkmark & \checkmark & \checkmark & \checkmark & \checkmark & $89.3 \pm 1.3$     & $95.3 \pm 1.2$ \\
    \hline
    \hline
    \end{tabular}%
    }
  \label{tab:ablation}%
\end{table}

\paragraph{Results}
The results in Table~\ref{tab:ablation} demonstrate the effectiveness of each component in MUformer. We can see that each component proposed in Sec~\ref{sec:graph.transformer} plays an indispensable role in learning all aspects of the molecule in a unified manner and the diffusion model with all components combined generates the most stable and valid molecules.

\section{Discussion}
\label{sec:discussion}
\subsection{2D and 3D Attention Biases}
\label{sec:discussion.attention.bias}
In Sec~\ref{sec:muformer.attention.bias}, we introduce two attention biases employed in MUformer to compute attentions, one for 2D molecular structure and another for 3D geometric structure. For the 2D spatial attention bias, we use the Shortest Path Distance (SPD) encoding to capture the distance between atoms $i$ and $j$ in the 2D molecular graph, providing vital structural relationship information. Additionally, we incorporate edge-type information along the shortest path, which reflects the connections between atoms $i$ and $j$, further enhancing the model's understanding of the 2D molecular graph. In the case of the 3D spatial attention bias, we calculate the bias using the Euclidean distance and an exponential radial basis function, which encodes the 3D spatial relationships between atom pairs in the 3D molecular structure. This approach helps the model account for the geometric arrangement of atoms in the molecule.
These biases are used to improve the model's representation of molecular structures by capturing essential structural features and relationships in both 2D and 3D spaces.

\paragraph{2D Attention Bias} The 2D spatial attention bias, which consists of the Shortest Path Distance (SPD) encoding and edge-type information, helps the model to capture the topological relationships between atoms in the 2D molecular graph. By incorporating this bias into the attention computation, the model can better understand the structural relationships and chemical properties of the molecule, leading to improved prediction and generation tasks.

\paragraph{3D Attention Bias} The 3D spatial attention bias encodes the 3D spatial relationships between atom pairs in the molecular structure using Euclidean distance and an exponential radial basis function. By including this information as a bias in the attention computation, the model can account for the geometric arrangement of atoms in the molecule. This allows the MUformer to recognize spatial patterns and interactions that are not apparent in the 2D graph representation alone.

 The attention biases introduced in Sec~\ref{sec:muformer.attention.bias} for both 2D and 3D molecular structures enhance the MUformer ability to capture essential structural features and relationships in both spaces. By incorporating these biases in the attention computation, the model can prioritize and focus on the most relevant connections between atoms, resulting in a more accurate and comprehensive molecular representation.

\subsection{Interdependence of Generation of 2D and 3D Structures}
\label{sec:discussion.2d3d.relation}
Generating both 2D and 3D structures can provide a more complete representation of the molecular structure, as it captures both the planar arrangement of atoms in the molecule and their spatial arrangement in 3D space. By combining the generation of 2D and 3D structures, we can provide a more comprehensive understanding of the molecular structure, which could be useful for a variety of applications in drug discovery, materials science, and other fields. 
We discuss their effects from three perspectives, (1) conformational space, (2) stereochemistry, and (3) constraint.
\paragraph{Impact of 2D Generation on 3D Generation} 
(1) From the conformational perspective, the 2D structure can provide information about the planar arrangement of atoms in the molecule, which can be used to guide the generation of the 3D structure. By considering the 2D structure, the generation algorithm can explore different conformations and orientations of the molecule in a 3D space, which can help generate a more accurate 3D structure. 
(2) From the stereochemistry perspective, the 2D structure can provide information about the stereochemistry and chirality of the molecule, which can be used to guide the generation of the 3D structure. For example, if the 2D structure indicates that two atoms are connected by a double bond, the generation algorithm can infer the correct geometry for the double bond in the 3D space based on the stereochemistry of the molecule.
(3) From the constraints' perspective, the 2D structure can provide constraints on the geometry of the molecule, which can be used to guide the generation of the 3D structure. For example, if the 2D structure indicates that two atoms are connected by a ring, the generation algorithm can use this information to constrain the geometry of the ring in a 3D space.

\paragraph{Impact of 3D Generation on 2D Generation}
(1) From the conformational perspective, the generation of 3D structures can provide information about the conformational space of the molecule, which can be used to refine the 2D structure. For example, if the 3D structure indicates that two atoms are in close proximity, the generation algorithm can adjust the 2D structure to reflect this.
(2) From the stereochemistry perspective, the 3D structure can provide additional information about the stereochemistry and chirality of the molecule, which can be used to refine the 2D structure. For example, if the 3D structure indicates that two atoms have a specific orientation in 3D space, the generation algorithm can adjust the 2D structure to reflect this.
(3) From the constraint perspective, the 3D structure can provide additional constraints on the geometry of the molecule, which can be used to refine the 2D structure. For example, if the 3D structure indicates that two atoms are connected by a ring, the generation algorithm can adjust the 2D structure to ensure that the ring is planar.

\subsection{Limitation}
Generating molecular structures using graph models presents a challenge in representing the structures in a way that can be processed by the model. One commonly used approach is to represent the structure as a dense adjacency tensor, where each element in the tensor corresponds to the presence or absence of a bond between two atoms. However, generating these dense tensors can be computationally expensive, particularly for larger or more complex molecular structures.

\paragraph{Approach \& Limitation} In our model, which includes a transformer and diffusion model, we generate 2D molecular structures and edge features by creating a dense adjacency tensor of size $n\times n\times b$, where n represents the number of atoms in a molecule and b represents the number of edge types. In the case of molecule scenarios, our model predicts a dense tensor of size $n\times n\times 4$, which includes four bond types (no-bond, single bond, double bond, and triple bond). One advantage of using dense tensors is that they can capture more detailed information about the molecular structure, including the precise location and type of each bond. This level of detail can be particularly crucial in cases where subtle differences in the structure can have a significant impact on the molecule's properties or behavior. However, a significant disadvantage of dense tensors is the computational cost required to generate and process them, which can limit the scalability and efficiency of the model.

\paragraph{Sparse Tensor Solution}
Sparse tensors can be a useful approach to reducing the computational cost of generating molecular structures. Sparse tensors are similar to dense tensors, but they only store the non-zero elements of the tensor, rather than the entire tensor.
To use sparse tensors for the problem of generating molecular structures, one approach is to represent the adjacency matrix as a sparse tensor. Rather than creating a dense tensor of size $n\times n\times b$, we can instead create a sparse tensor that only stores the non-zero elements of the adjacency matrix. 
To make predictions for sparse tensors, we can use specialized algorithms designed for sparse tensors. One common approach is to use sparse matrix multiplication algorithms, which can efficiently perform matrix operations on sparse tensors. These algorithms are designed to take advantage of the sparsity of the tensor to perform the required computations more efficiently.

\paragraph{Multi-resolution Representation Solution}
Multi-resolution representation of molecules is an approach to represent molecular structures at multiple levels of detail, allowing for more efficient processing while still capturing important structural information. Essentially, the idea is to represent the molecule in different ways, with each representation capturing different levels of detail.
One common approach is to use a hierarchical representation, where the molecular structure is represented as a series of nested substructures. For example, the molecule could be represented as a set of atoms, each associated with a set of neighboring atoms. This set of neighboring atoms could then be recursively expanded to include their own neighboring atoms, resulting in a hierarchical representation of the molecule that captures different levels of detail at different scales.
Another approach is to use a multi-scale representation, where the molecular structure is represented at different levels of detail using different feature maps or descriptors. For example, the molecule could be represented as a set of atoms, each associated with a descriptor that captures its physical properties.
By using multi-resolution representations, we can reduce the computational cost of generating molecular structures while still capturing important structural information.

\subsection{Overview}
In this chapter, we introduce our MUDiff model \cite{hua2023mudiff}, a transformer-based framework for learning and generating a complete molecule representation using a novel GNN architecture, MUformer, with group equivariance. A distinct advantage of MUDiff and MUformer is the ability to function independently when either 2D or 3D structural information is missing. Our model remains effective in generating and learning a complete representation of molecules, even when the input data lacks the 2D graph structure or the 3D geometric structure. This design is particularly useful because datasets sometimes have missing 3D coordinates and geometry, resulted from limitations in experimental techniques or the unavailability of suitable computational resources \cite{mobley2017predicting}. For instance, the conformational analysis of molecules, which involves determining the 3D structures that result from the rotation of single bonds, is of critical importance in understanding molecular interactions \cite{leach2011molecular, copeland2013evaluation}. However, obtaining accurate conformational data can be computationally demanding \cite{gaulton2012chembl}. In such cases, MUDiff and MUformer can provide a robust and versatile solution for handling incomplete molecular datasets, ensuring comprehensive molecular representations even when faced with such challenges.
\chapter{Conclusion and Future Work}

In this work, we study the capabilities of Graph Neural Networks (GNNs),  exploring their fundamental designs and applications to address real-world challenges. In particular, introduce a GNN featuring a high-order pooling function, capturing complex interactions between nodes in graph-structured data. designed to capture intricate interactions between nodes in graph-structured data. This pooling function enhances the GNN's performance on both node- and graph-level tasks. Furthermore, we propose a GNN as the backbone model for a molecular graph generative model. The GNN backbone is adept at learning both the invariant and the equivariant aspects of molecules. The molecular graph generative model leverages these features to simultaneously learn and generate atom-bond molecular graphs with atom positions. Our models are rigorously evaluated through experiments and comparisons with existing methods, demonstrating their effectiveness in tackling real-world problems with diverse datasets.

\section{Conclusion and Contribution}
\label{sec:conclusion}
We have presented two novel contributions to the field of graph neural networks: a high-order pooling function for graph classification, and a molecular graph generative model for molecule design. We have shown that both of these contributions can improve the performance and expressiveness of graph neural networks, and address some of the real-world challenges and applications that involve graph-structured data.

Our first contribution is a high-order pooling function for graph neural networks, based on symmetric tensor decomposition. We have designed a new layer, called the CP layer, that can model non-linear high-order multiplicative interactions among node representations, and enhance the representation power of graph neural networks. We have theoretically proved that the CP layer can compute any permutation-invariant multilinear polynomial, and is universally strictly more expressive than sum and mean pooling. We have also empirically demonstrated that the CP layer can improve the accuracy and robustness of graph neural networks on various graph classification tasks, and outperform existing pooling methods.

Our second contribution is a molecular graph generative model for molecule design, based on a graph transformer backbone and a generative diffusion model. We have proposed a new framework, called MUDiff, that can co-generate molecular graphs and structures, and learn both the invariant and the equivariant aspects of molecules. We have leveraged the graph transformer to learn the hidden representations of molecular graphs and structures, and the generative diffusion model to generate realistic and diverse molecules. We have also experimentally shown that MUDiff can generate valid and novel molecules that satisfy various design objectives and constraints, and outperform previous methods on various molecule generation tasks.

We believe that our contributions can advance the state-of-the-art of graph neural networks, and provide new insights and solutions for real-world problems that involve graph-structured data. We hope that our work can inspire more research on the design and application of graph neural networks, and open up new possibilities and directions for future work.

\section{Limitation and Future Work}
\label{sec:future.work}
\subsection{Pooling Function}
We have presented a novel approach for high-order pooling for graph neural networks, based on symmetric tensor decomposition. Our approach, tGNN, can model non-linear high-order multiplicative interactions among node representations, and achieve state-of-the-art performance on various graph classification tasks. However, there are still some limitations and challenges that can be addressed in future work.

One limitation of our approach is that it requires a fixed number of node representations to be aggregated by the CP layer, which may not be suitable for graphs with varying sizes or structures. Therefore, a possible direction for future work is to explore how to adapt the CP layer to handle variable-sized inputs, such as using attention mechanisms, dynamic programming, or graph coarsening techniques.

Another challenge of our approach is that it relies on a predefined set of basis tensors to parameterize the CP layer, which may limit the expressiveness and flexibility of the model. Moreover, the choice of the basis tensors may affect the computational efficiency and stability of the model. Therefore, another direction for future work is to develop more adaptive and efficient methods to learn the basis tensors from the data, such as using neural networks, sparse coding, or tensor factorization methods.

A third challenge of our approach is that it does not explicitly incorporate edge features or edge weights into the aggregation function, which may ignore some important information or relationships in the graph. Therefore, a third direction for future work is to explore how to extend the CP layer to account for edge features or edge weights.

\subsection{Generative Model}
We have presented a novel framework for co-generating molecular structures and geometry, based on a graph transformer backbone and a generative diffusion model. Our framework, MUDiff, can generate realistic and diverse molecules that satisfy various design objectives and constraints. However, there are still some limitations and challenges that need to be addressed.

One limitation of our framework is that it requires a large and high-quality dataset of molecular graphs and structures to train the model. However, such datasets are not always available or easy to obtain, especially for specific domains or applications. Therefore, a possible direction for future work is to explore how to leverage other sources of data, such as synthetic or simulated data, to augment the training data and improve the model performance.

Another challenge of our framework is that it relies on a fixed and predefined set of atom types and bond types, which may limit the diversity and novelty of the generated molecules. Moreover, the model may not be able to handle new or rare atom types or bond types that are not seen in the training data. Therefore, another direction for future work is to develop more flexible and adaptive methods to represent and generate molecular graphs and structures, such as using continuous or latent variables to encode the atom types and bond types, or using graph neural networks to learn the atom types and bond types from the data.

\newpage

\bibliography{references}
\bibliographystyle{acm}

\end{document}